\documentclass{article}

    \usepackage[preprint]{neurips_2025}

\usepackage[utf8]{inputenc}
\usepackage[T1]{fontenc}
\usepackage{hyperref}
\usepackage{url}
\usepackage{booktabs}
\usepackage{amsfonts}
\usepackage{nicefrac}
\usepackage{amssymb}
\usepackage{microtype}
\usepackage{xcolor}
\usepackage{graphicx}
\usepackage{amsmath}
\usepackage{enumitem}
\usepackage{amsthm}
\usepackage{multirow}
\usepackage[font=small,labelfont=bf]{caption}
\usepackage{wrapfig}  %
\usepackage{tcolorbox} %
\tcbuselibrary{skins} %

\newtcolorbox{conclusionbox}[1][]{%
    enhanced,
    colback=gray!3!white,
    colframe=blue!40!black,
    boxsep=2pt,
    left=2pt,
    right=2pt,
    top=2pt,
    bottom=2pt,
    sharp corners,
    halign=flush left,
    #1
}

\usepackage{colortbl}
\usepackage{booktabs}
\usepackage{subfigure}
\usepackage{subcaption}

\definecolor{bucket1}{RGB}{255, 240, 240}
\definecolor{bucket2}{RGB}{255, 245, 230}
\definecolor{bucket3}{RGB}{240, 255, 240}
\definecolor{bucket4}{RGB}{230, 245, 255}
\definecolor{bucket5}{RGB}{245, 240, 255}

\newcommand{\numClusters}{C} %
\newcommand{\clusterIdx}{i} %
\newcommand{\clusterProb}{p_{\clusterIdx}} %
\newcommand{\zipfExp}{\alpha} %
\newcommand{\gtDictWidth}{d} %
\newcommand{\saeDictWidth}{d_\text{sae}} %
\newcommand{\saeClusterCapacity}{D_{\clusterIdx}} %
\newcommand{\allocExp}{\beta} %
\newcommand{\localRedundancyFactor}{R_{\clusterIdx}} %

\newtheorem{theorem}{Theorem}
\newtheorem{lemma}{Lemma}
\newtheorem{definition}{Definition}

\newtheorem{corollary}{Corollary}

\title{Position: Mechanistic Interpretability Should Prioritize Feature Consistency in SAEs}

\author{%
  Xiangchen Song\thanks{Equal contribution. Correspondence to \texttt{xiangchensong@cmu.edu} and \texttt{amuhamed@andrew.cmu.edu}.} \\
  Carnegie Mellon University \\
  \And
  Aashiq Muhamed\footnotemark[1] \\
  Carnegie Mellon University \\
  \And
  Yujia Zheng \\
  Carnegie Mellon University \\
  \And
  Lingjing Kong \\
  Carnegie Mellon University \\
  \And
  Zeyu Tang \\
  Carnegie Mellon University \\
  \And
  Mona T. Diab \\
  Carnegie Mellon University \\
  \And
  Virginia Smith \\
  Carnegie Mellon University \\
  \And
  Kun Zhang \\
  Carnegie Mellon University \& MBZUAI\\
}

\begin{document}

\maketitle

\begin{abstract}
\looseness=-1
Sparse Autoencoders (SAEs) are a prominent tool in mechanistic interpretability (MI) for decomposing neural network activations into interpretable features. However, the aspiration to identify a canonical set of features is challenged by the observed inconsistency of learned SAE features across different training runs, undermining the reliability and efficiency of MI research. \textbf{This position paper argues that mechanistic interpretability should prioritize feature consistency in SAEs}---the reliable convergence to equivalent feature sets across independent runs. We propose using the Pairwise Dictionary Mean Correlation Coefficient (PW-MCC) as a practical metric to operationalize consistency and demonstrate that high levels are achievable ($0.80$ for TopK SAEs on LLM activations) with appropriate architectural choices. 
Our contributions include detailing the benefits of prioritizing consistency; providing theoretical grounding and synthetic validation using a \emph{model organism}, which verifies PW-MCC as a reliable proxy for ground-truth recovery; and extending these findings to real-world LLM data, where high feature consistency strongly correlates with the semantic similarity of learned feature explanations.
We call for a community-wide shift towards systematically measuring feature consistency to foster robust cumulative progress in MI.\footnote{Code is available at \url{https://github.com/xiangchensong/sae-feature-consistency}.}

\end{abstract}

\section{Introduction}

Mechanistic Interpretability (MI) seeks to reverse-engineer neural networks into human-understandable algorithms \citep{olah2020zoom, elhage2021mathematical}, with Sparse Autoencoders (SAEs) emerging as a prominent tool for decomposing model activations into more interpretable, monosemantic features \citep{bricken2023monosemanticity, cunningham2023sparse, gao2025scaling, pach2025sparse}. The aspiration within the MI community is often to identify a canonical set of features---unique, complete, and atomic units of analysis that faithfully represent the model's internal computations \citep{leask2025sparse}. However, a significant challenge highlighted by recent work~\citep{paulo2025sparse, marks2024enhancing, fel2025archetypal, leask2025sparse, meloux2025everything} is the observed inconsistency of features learned by SAEs across different training runs, even when using identical data and model architectures. This instability, potentially arising from phenomena like feature splitting \citep{leask2025sparse, chanin2024absorption} or the amortization gap \citep{o2024compute}, undermines the reliability of derived interpretations, reduces research efficiency, and impacts the trust in findings derived from MI.

\textbf{Mechanistic Interpretability Should Prioritize Feature Consistency in SAEs}. We argue that the reliable convergence to equivalent feature sets across independent SAE training runs should be elevated from a secondary concern to an essential evaluation criterion and an active research priority. We present the Pairwise Dictionary Mean Correlation Coefficient (PW-MCC) as a concrete, working example of how consistency can be operationalized. Furthermore, we demonstrate that high levels of consistency are achievable with appropriate architectural and training choices (e.g., with TopK SAEs), and highlight the benefits of such prioritization for scientific rigor and practical utility in MI. 

Our main contributions are:
\begin{itemize}[leftmargin=*, itemsep=1pt, parsep=1pt, topsep=2pt, partopsep=0pt]
    \item \looseness=-1 We advocate for prioritizing feature consistency in MI for SAEs, detailing its benefits for scientific reproducibility, research efficiency, and trustworthiness of interpretations. We propose using PW-MCC as a practical metric for operationalizing run-to-run feature consistency (Section~\ref{sec:our_position}).
    \item \looseness=-1 We provide theoretical grounding for achieving strong feature consistency by connecting SAEs to established identifiability results in overcomplete sparse dictionary learning. We validate this using a synthetic \emph{model organism}, demonstrating that PW-MCC reliably tracks ground-truth feature recovery (GT-MCC) and that a specific SAE architecture (TopK SAE) achieves high consistency ($\approx 0.97$) under idealized, matched-capacity conditions (Section~\ref{sec:evidence_synthetic}).
    \item We demonstrate empirically on large language model activations that high feature consistency (PW-MCC $\approx 0.80$ for TopK SAEs) is attainable with appropriate architectural choices and training. Our real-world data experiments reflect findings from the synthetic setting (e.g., frequency-dependent consistency) and critically show that high PW-MCC scores correlate strongly with the semantic similarity of feature explanations (Section~\ref{sec:eval_consistency_real_world}).
\end{itemize}
Ultimately, this work calls for a community-wide shift towards valuing and systematically measuring feature consistency for cumulative progress in understanding the inner workings of complex neural models, and we outline several open questions and future research directions to achieve this.

\section{Background and Related Work}
\label{sec:related_work}

\paragraph{Sparse Autoencoders for MI.}
\looseness=-1
MI aims to reverse-engineer neural networks into human-understandable algorithms by identifying and explicating their internal components and computational processes~\citep{olah2020zoom, pach2025sparse, cunningham2023sparse}. A central challenge in MI is polysemanticity, where individual neurons respond to multiple, unrelated concepts, which obscures straightforward interpretation of model internals \citep{elhage2022toy}. SAEs have emerged as a tool to address this challenge by decomposing high-dimensional neural network activations into a sparser, higher-dimensional representation that aims to isolate \emph{monosemantic} features \citep{bricken2023monosemanticity, cunningham2023sparse}.
An SAE comprises an encoder and decoder network. The encoder transforms an input activation vector $\mathbf{x} \in \mathbb{R}^m$ into a sparse latent representation $\mathbf{f} \in \mathbb{R}^{d_{\text{sae}}}$ (where $d_{\text{sae}} > m$ establishes an overcomplete dictionary): $\mathbf{f}(\mathbf{x}) = \sigma(\mathbf{W}_{\text{enc}}\mathbf{x} + \mathbf{b}_{\text{enc}})$. Here, $\mathbf{W}_{\text{enc}}$ and $\mathbf{b}_{\text{enc}}$ are the encoder weights and biases, while $\sigma$ is a non-linear activation function that is sparsity-inducing. The decoder reconstructs the input from $\mathbf{f}(\mathbf{x})$ as $\hat{\mathbf{x}} = \mathbf{W}_{\text{dec}}\mathbf{f}(\mathbf{x}) + \mathbf{b}_{\text{dec}}$, where $\mathbf{W}_{\text{dec}}$ and $\mathbf{b}_{\text{dec}}$ are the decoder weights and biases. SAEs are trained by minimizing a loss function that balances reconstruction fidelity with sparsity: $L(\mathbf{x}) = \|\mathbf{x} - \hat{\mathbf{x}}\|_2^2 + \lambda S(\mathbf{f}(\mathbf{x}))$, where $S(\cdot)$ represents a sparsity-inducing penalty (e.g., L1 norm) and $\lambda$ controls the sparsity trade-off. 
Once trained, SAEs provides a decomposition of the input activation as $\smash{\mathbf{x} \approx \sum_{i=1}^{d_{\text{sae}}} f_i(\mathbf{x}) \mathbf{a}_i}$, where $f_i(\mathbf{x})$ are the sparse feature activations, and $\mathbf{a}_i$ are dictionary elements corresponding to columns of feature dictionary $\mathbf{A}$ (i.e., $\mathbf{W}_{\text{dec}}$). Several SAE variants have been proposed to improve feature quality and sparsity control. These include Standard ReLU \citep{bricken2023monosemanticity}, TopK \citep{gao2024scaling}, BatchTopK \citep{bussmann2024batchtopk}, Gated~\citep{rajamanoharan2024improving}, and JumpReLU~\citep{rajamanoharan2024jumping}. The ultimate aspiration for many researchers employing SAEs is to identify a \emph{canonical} set of features: unique, complete, and atomic units of analysis that faithfully represent the model's internal computations \citep{leask2025sparse}.

\paragraph{The Challenge of Feature Consistency.}
Despite their promise, SAEs trained on identical data and architectures but different random initializations often converge to substantially different feature sets~\citep{marks2024enhancing, karvonen2024saebench}, with overlap sometimes as low as 30\% for Standard SAEs~\citep{paulo2025sparse}. 
This inconsistency manifests through several documented phenomena, including \emph{feature splitting}, where concept representations vary across runs~\citep{leask2025sparse}, and \emph{feature absorption}, where general features are usurped by more specific ones~\citep{chanin2024absorption}.
These empirical instabilities stem from fundamental limitations in overcomplete dictionary learning~\citep{leask2025sparse}. Despite theoretical advances~\citep{sun2024global, donoho2003optimally}, the gap between idealized assumptions and practical implementations undermines guarantees for unique feature recovery. Existing approaches to address these limitations include Mutual Feature Regularization~\citep{marks2024enhancing}, which forces alignment between concurrently trained SAEs but addresses the effect rather than underlying causes, and Archetypal SAEs~\citep{fel2025archetypal}, which impose geometric constraints that may sacrifice representational power for stability. These challenges have fostered skepticism, with some \citep{paulo2025sparse} suggesting that SAE features should be viewed as pragmatically useful decompositions rather than exhaustive, universal sets.
Contrary to this prevailing skepticism, our work shows that high feature consistency is attainable through careful architectural and training choices without explicit alignment mechanisms.

\section{Mechanistic Interpretability should prioritize Feature Consistency in SAEs}
\label{sec:our_position}

A prerequisite for the scientific validity and practical utility of features extracted in SAEs, is their \emph{consistency}---the reliable convergence to equivalent feature sets across independent training runs given the same data and model architecture. Feature consistency should be elevated from a secondary concern to an essential evaluation criterion and an active research priority for the following reasons.

\textbf{Firstly, achieving consistency yields substantial benefits for current MI practices that use SAEs.}
\begin{itemize}[leftmargin=*,itemsep=1pt,topsep=0pt,parsep=0pt]
    \item \textit{Improved Scientific Reproducibility:} Reproducibility is a cornerstone of science. If SAEs produce different feature dictionaries run-to-run \citep{paulo2025sparse, fel2025archetypal}, feature explanations and discovered circuits become difficult to replicate. Consistency ensures findings are robust, not initialization artifacts, and helps foster cumulative progress.
    \item \textit{Improved Research Efficiency and Resource Allocation:} Significant effort is invested in interpreting SAE features, whether manually or through automated methods~\citep{bricken2023monosemanticity}. If features are not consistent across training runs, this entire interpretation process: identifying, labeling, and understanding features must be repeated for each new SAE training instance. Consistent features, however, can be reliably matched across instances, allowing interpretations to be reused and incrementally refined, thereby saving substantial researcher time and computational resources.
     \item \textit{Increased Trust in Explanations from SAEs on Private Data:} When SAEs are trained on private data and their dictionaries and feature explanations are shared, feature consistency increases credibility of the explanations. Although the training process (e.g., random initialization) can create features that are artifacts of that specific run, a feature consistently emerging across multiple runs is more likely a stable abstraction learned from the data distribution than an ephemeral artifact. This measurable robustness to training variability lends greater confidence that the interpretations reflect genuine learned patterns, which is even more important when direct data validation is impossible.
\end{itemize}

\textbf{Secondly, many current SAE techniques already implicitly assume feature consistency, even if this assumption is not explicitly verified.} The long-term MI ambition of identifying \emph{canonical units of analysis}---features that are complete, atomic, and unique \citep{leask2025sparse}---requires run-to-run consistency as a prerequisite before addressing complexities like atomicity or completeness. Another example of implicit reliance is \textit{feature stitching} \citep{leask2025sparse}, which compares or combines features across different models or SAEs; this process relies on identifiable and stable underlying features. Similarly, downstream applications like model steering~\citep{chalnev2024improving}, unlearning~\citep{muhamed2025saes}, bias removal ~\citep{marks2024sparse}), and feature ablation assume that the targeted features are well-defined, consistent entities. If these base features are unstable or non-identifiable, the reliability of such applications is severely compromised.

\paragraph{Our Proposal: Defining and Measuring Feature Consistency.}
While prior work has highlighted the challenge of feature inconsistency in SAEs \citep{paulo2025sparse, marks2024enhancing, o2024compute}, often leading to pessimistic conclusions about achieving stable feature sets, our findings demonstrate that high levels of feature consistency \textit{are} attainable, with appropriate architectural choices. This motivates a renewed focus on consistency as a key dimension for evaluation. Our central position is that the MI community \textbf{should prioritize feature consistency in SAEs}.
 This requires not only acknowledging its importance but also adopting rigorous methods for its quantification. Conceptually, feature consistency implies that two feature dictionaries, $\smash{\mathbf{A}}$ and $\smash{\mathbf{A}'}$ (both $\smash{\in \mathbb{R}^{m \times d_{\text{sae}}}}$ with $\smash{d_{\text{sae}}}$ features), learned from independent training runs using same dataset, should capture the same underlying concepts. We formalize this ideal with an empirically tractable notion, \textbf{Strong Feature Consistency}, where the dictionaries are considered equivalent if their feature vectors align up to a permutation and individual non-zero scaling factors. That is, for each feature vector $\smash{\mathbf{a}_i}$ in $\smash{\mathbf{A}}$, there should ideally exist a corresponding feature vector $\smash{\mathbf{a}_{\sigma(i)}'}$ in $\smash{\mathbf{A}'}$ (where $\sigma$ is a permutation) such that $\smash{\mathbf{a}_i  = \lambda_i \mathbf{a}_{\sigma(i)}'}$ for some scaling factor $\smash{\lambda_i \neq 0}$. More general notions of consistency are detailed in Appendix~\ref{app:formal_consistency_definitions}.

To make this actionable, throughout this paper we adopt a commonly used evaluation metric from the independent component analysis literature \citep{hyvarinen2000independent}: the Mean Correlation Coefficient (MCC). This metric directly evaluates permutation and scaling equivalence, making it a robust measure of Strong Feature Consistency for dictionary-based features in SAEs. Cosine similarity addresses arbitrary positive scaling, while the use of the absolute value in MCC accounts for feature sign. We present this as a concrete, working example of how consistency can be operationalized, although alternative metrics may be more appropriate for other feature types or notions of equivalence.

We define a general \textbf{Mean Correlation Coefficient (MCC)} between any two feature dictionaries $\smash{\mathbf{A} \in \mathbb{R}^{m \times d_A}}$ and $\smash{\mathbf{B} \in \mathbb{R}^{m \times d_B}}$ with columns $\smash{\mathbf{a}_i}$ and $\smash{\mathbf{b}_j}$ respectively. Let $\smash{n = \min(d_A, d_B)}$ and $\smash{\mathcal{M}_{n}(\mathbf{A}, \mathbf{B})}$ be the set of all possible one-to-one matchings of size $\smash{n}$ between the features of $\smash{\mathbf{A}}$ and $\smash{\mathbf{B}}$. The MCC is defined as:
\begin{equation*}
\text{MCC}(\mathbf{A}, \mathbf{B}) = \frac{1}{n} \max_{M \in \mathcal{M}_{n}(\mathbf{A}, \mathbf{B})} \sum_{(i,j) \in M} \frac{|\langle \mathbf{a}_{i}, \mathbf{b}_{j} \rangle|}{\|\mathbf{a}_{i}\|_2 \|\mathbf{b}_{j}\|_2}.
\end{equation*}

The optimal matching $M^*$ that achieves this maximum is typically found using the Hungarian algorithm.
From this general definition, we derive two specific metrics for our evaluations:

1. \textbf{Pairwise Dictionary Mean Correlation Coefficient (PW-MCC)}: When comparing two dictionaries $\smash{\mathbf{A}}$ and $\smash{\mathbf{A}'}$ of learned features, both of size $\smash{d_{\text{sae}}}$, we use $\text{PW-MCC}(\mathbf{A}, \mathbf{A}') = \text{MCC}(\mathbf{A}, \mathbf{A}')$ where $n = d_{\text{sae}}$. A PW-MCC approaching unity signifies robust convergence to highly similar feature dictionaries across independent training runs.

2. \textbf{Ground-Truth MCC (GT-MCC)}: In controlled synthetic environments, where a ground-truth dictionary $\smash{\mathbf{A}_{\text{gt}} \in \mathbb{R}^{m \times d_{\text{gt}}}}$ is known, we use $\text{GT-MCC}({\mathbf{A}}, \mathbf{A}_{\text{gt}}) = \text{MCC}({\mathbf{A}}, \mathbf{A}_{\text{gt}})$ to evaluate the recovery quality of a learned dictionary $\smash{{\mathbf{A}} \in \mathbb{R}^{m \times d_{\text{sae}}}}$, where $n = \min(d_{\text{sae}}, d_{\text{gt}})$. GT-MCC can be used for validating PW-MCC as a proxy for consistency.

Prioritizing consistency, and employing well-defined metrics such as PW-MCC to quantify it, offers several advantages: (i) it provides an objective measure of run-to-run stability for this key notion of feature equivalence; (ii) it facilitates equitable comparisons across methods and settings; and (iii) it incentivizes the development of techniques that yield more reliable features. In the following sections, we provide evidence from theory, synthetic experiments, and real-world applications to support our position and illustrate both the attainability and the challenges of achieving high feature consistency.

\section{Evidence from Theoretical Analysis and Synthetic Experiments}
\label{sec:evidence_synthetic}

\subsection{Theoretical Foundations for Feature Consistency}
\label{sec:theoretical_foundations}

SAEs learn to represent input data $\mathbf{X} \in \mathbb{R}^{m \times n}$ through a dictionary $\mathbf{A} \in \mathbb{R}^{m \times d_{\text{sae}}}$ and corresponding sparse activations $\mathbf{F} \in \mathbb{R}^{d_{\text{sae}} \times n}$, such that $\mathbf{X} \approx \mathbf{A}\mathbf{F}$. Previous work often dismisses non-invertible dictionaries as non-consistent~\citep{o2024compute, joshi2025identifiable}, particularly in the overcomplete regime of dictionary learning where $d_{\text{sae}} > m$. However, this overlooks the natural sparsity present in real signals. Drawing inspiration from sparse dictionary learning literature, we show how sparsity enables feature consistency guarantees even in overcomplete settings.
We build our analysis on the \textbf{spark condition}~\citep{hillar2015when, donoho2003optimally}, which precisely characterizes when unique sparse representations exist:

\begin{definition}[Spark condition]\label{def:spark_condition}
A dictionary $\mathbf{A}\in\mathbb{R}^{m\times d_{\text{sae}}}$ satisfies the spark condition at sparsity level $k$ if for any two $k$-sparse vectors $\mathbf{f},\mathbf{f}'\in\mathbb{R}^{d_{\text{sae}}}$, the equality $\mathbf{A}\mathbf{f} = \mathbf{A}\mathbf{f}'$ implies that $\mathbf{f}=\mathbf{f}'$.
\end{definition}
This condition ensures that distinct $k$-sparse vectors produce distinct outputs when transformed by the dictionary $\mathbf{A}$. Equivalently, 
it provides \emph{injectivity of the linear map $\mathbf{A}$ on the set of $k$-sparse vectors} $\Sigma_k := \{\mathbf{f} \in \mathbb{R}^{d_{\text{sae}}} : \|\mathbf{f}\|_0 \leq k\}$. This is precisely the algebraic property needed for uniqueness of sparse representations.
We leverage the following result from \citep{hillar2015when}:

\begin{theorem}[Adapted from \citep{hillar2015when}]\label{thm:hillar} 
Fix sparsity level $k$. There exists a witness set of $\smash{n=k\binom{d_{\text{sae}}}{k}^{2}}$ $k$-sparse vectors $\smash{\mathbf{f}_1,\dots,\mathbf{f}_n \in \Sigma_k}$ such that for \emph{any} pair of dictionaries $\smash{\mathbf{A},\mathbf{A}' \in \mathbb{R}^{m\times d_{\text{sae}}}}$ satisfying the spark condition, the factorizations
$
\smash{
\mathbf{X}=[\mathbf{A}\mathbf{f}_1,\dots,\mathbf{A}\mathbf{f}_n]
\quad\text{and}\quad
\mathbf{X}=[\mathbf{A}'\mathbf{f}'_1,\dots,\mathbf{A}'\mathbf{f}'_n]
}
$
with $k$-sparse codes must coincide up to a permutation and scaling of columns:
$\smash{\mathbf{A}' = \mathbf{A}\mathbf{P}\mathbf{D}}$ and $\smash{\mathbf{F}' = \mathbf{D}^{-1}\mathbf{P}^{\top}\mathbf{F}}$ for some permutation matrix $\mathbf{P}$ and diagonal invertible $\mathbf{D}$.
\end{theorem}

\paragraph{Implications for TopK SAE Feature Consistency.}
TopK SAEs achieve feature consistency by satisfying the conditions required for unique sparse factorization. Consider a TopK SAE with encoder $E$ and decoder $\mathbf{A}$ that enforces exactly $k$-sparse activations via $\mathbf{f}\;\mapsto\;\operatorname{TopK}_k(\mathbf{f})$. The training objective simultaneously encourages three key properties: \textbf{(1) Exact $k$-sparsity} by construction of the TopK constraint, which zeros all but the $k$ largest coordinates; \textbf{(2) Zero reconstruction error} by minimizing $\|\mathbf{X}-\mathbf{A}\mathbf{F}\|_F$ on data containing the witness set from Theorem~\ref{thm:hillar}; and \textbf{(3) The spark condition} through what we term the \emph{round-trip property} $E(\mathbf{A}\mathbf{f}) = \mathbf{f}$. As we prove in Appendix~\ref{sec:roundtrip-spark-proof}, this round-trip property directly implies the spark condition. 
When these three conditions hold with data coverage meeting the requirements of Theorem~\ref{thm:hillar}, any two TopK SAEs trained on the same data must learn dictionaries that are identical up to permutation and scaling, establishing \textbf{strong feature consistency}  we introduced in Section~\ref{sec:our_position} even in overcomplete regimes. This explains why TopK SAEs can achieve consistent features: their training objective directly optimizes for the mathematical prerequisites required by the identifiability theorem.

\begin{conclusionbox}
\small
\textbf{Takeaway:} SAEs with $k$-sparsity and minimal reconstruction error satisfy strong feature consistency when the learned dictionary meets the spark condition.
\end{conclusionbox}

\subsection{Synthetic Verification}
To empirically validate our theoretical analysis, we conduct synthetic experiments comparing two representative SAE variants: TopK SAE and Standard SAE. We show that models designed according to our theoretical criteria achieve consistent feature representations.

\begin{figure}[tbp]
    \centering
    \begin{minipage}[t]{0.3234\textwidth}
        \centering
        \includegraphics[width=\textwidth]{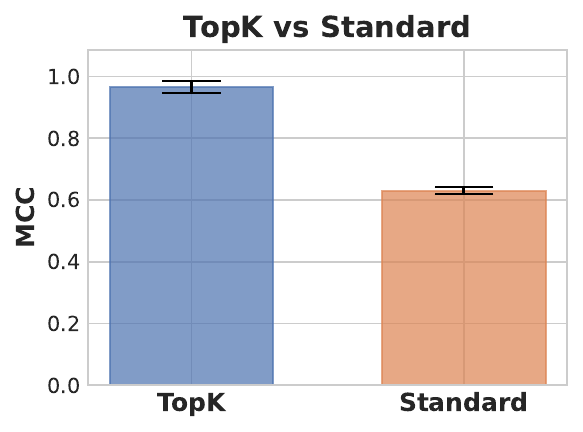}
        \caption{\small TopK SAE is significantly better than Standard SAE (0.97 vs 0.63) in terms of GT-MCC.}
        \label{fig:top_k_vs_standard}
    \end{minipage}
    \hfill
    \begin{minipage}[t]{0.66\textwidth}
        \centering
        \begin{minipage}[b]{0.49\textwidth}
            \includegraphics[width=\textwidth]{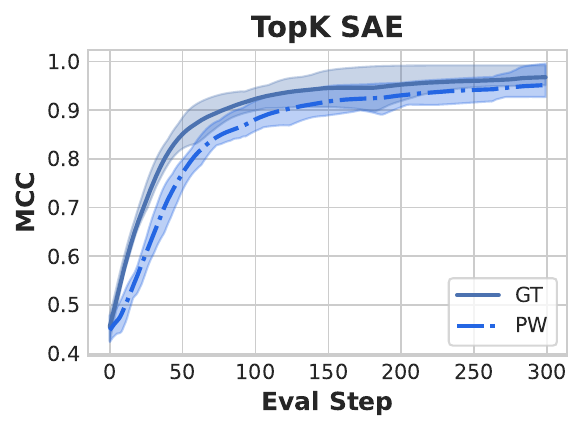}
        \end{minipage}%
        \hfill
        \begin{minipage}[b]{0.49\textwidth}
            \includegraphics[width=\textwidth]{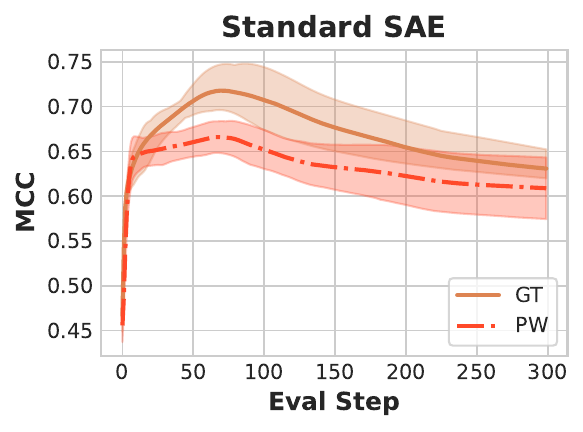}
        \end{minipage}
     \caption{\small GT-MCC and PW-MCC for TopK and Standard SAE. PW-MCC follows the same trend as GT-MCC, both converging to comparable values. Shaded region represents max-min range across seeds.}
        \label{fig:synthetic_sae_curve}
    \end{minipage}
\end{figure}

Following conventions in dictionary learning literature, we generate synthetic data by first sampling a ground-truth feature dictionary $\smash{\mathbf{A}_{\text{gt}} \in \mathbb{R}^{m \times d_{\text{gt}}}}$ from a standard normal distribution. We can represent all data points as $
 \smash{   \mathbf{X} = \mathbf{A}_{\text{gt}}\mathbf{F}_{\text{gt}}}$,
where $\smash{\mathbf{F}_{\text{gt}} \in \mathbb{R}^{d_{\text{gt}} \times n}}$ contains the activations for all $n$ data points. For each individual data sample $\mathbf{x}$, we enforce the $k$-sparse condition by randomly selecting $k$ features and setting their values to independent Gaussian samples: $
    \mathbf{x} = \mathbf{A}_{\text{gt}}\,\mathbf{f}_{\text{gt}}(\mathbf{x}),
$
where $\mathbf{f}_{\text{gt}}(\mathbf{x}) \in \mathbb{R}^{d_{\text{gt}}}$ represents a single column of $\mathbf{F}_{\text{gt}}$ corresponding to data point $\mathbf{x}$ and contains at most $k$ non-zero entries. 
In this synthetic setting ($\smash{m=8, d_{\text{gt}} = 16, k=3, n=5e4}$) we can evaluate the estimated feature dictionary $\mathbf{A}$ against the ground-truth $\mathbf{A}_{\text{gt}}$ using GT-MCC. We also conduct additional experiments by training multiple SAEs (5 seeds) with identical data and model architecture but different weight initializations, comparing the PW-MCC curves with the GT-MCC curves.
Figure~\ref{fig:top_k_vs_standard} presents the final MCC evaluation results, showing that TopK SAE achieves significantly better consistency than Standard SAE, which confirms our analysis in Theorem~\ref{thm:hillar}. More importantly, Figure~\ref{fig:synthetic_sae_curve} demonstrates that the empirical PW-MCC values follow the same trend as GT-MCC, achieving comparable final values, suggesting that PW-MCC serves as an effective alternative to GT-MCC when ground truth dictionaries are unavailable. We refer to this setting as the \textbf{matched regime}, where the empirical dictionary size \(d_{\text{SAE}}\) matches \(d_{\text{gt}}\). In all experiments for TopK SAE, the empirical sparsity value $k$ used in during training matches the ground truth sparsity. See Appendix~\ref{app:misspecifying_k} for the extended analysis when $k$ is misspecified.

\begin{conclusionbox}
\small
\textbf{Takeaway:} We observe that Pairwise MCC converges to GT-MCC and strong feature consistency is achieved with TopK SAE in synthetic matched settings.
\end{conclusionbox}

\subsection{A Synthetic Model Organism for Analyzing Feature Consistency}
\label{sec:model_organism_consistency_analysis}

\looseness=-1
While TopK SAEs can achieve high feature consistency under idealized matched-capacity scenarios, real-world data introduces substantial complexities that degrade this ideal. To show how these complexities affect feature consistency, we develop and analyze a synthetic model organism, progressively introducing realistic data characteristics. This allows us to observe how metrics like PW-MCC respond to these challenges providing insights into their continued diagnostic utility.

\textbf{Consistency and Global Capacity Regimes.}
Fundamental challenges to feature consistency arise even before considering heterogeneous ground-truth distributions, stemming from the relationship between the SAE's dictionary width $d_{\text{sae}}$ and ground-truth dictionary width $d_{\text{gt}}$. Using the linear generative model ($\mathbf{X} = \mathbf{A}_{\text{gt}}\mathbf{F}_{\text{gt}} $) earlier, where ground-truth $\mathbf{f}(\mathbf{x})$ are $k$-sparse ($k=8$ in our experiments) and when all ground-truth features are uniformly sampled, we observe distinct behaviors. 

In a \textbf{globally redundant regime} ($d_{\text{sae}} > d_{\text{gt}}$), where the SAE has more dictionary features than the ground truth (e.g., $d_{\text{sae}}=160, d_{\text{gt}}=80,k=8,n=5e4$), it can achieve high alignment with the ground-truth dictionary (GT-MCC $0.95$). This suggests learned features accurately represent underlying concepts (\autoref{fig:mcc_combined}). However, run-to-run consistency is often significantly lower (PW-MCC $0.77$). This discrepancy arises from \textbf{selection ambiguity}: with excess capacity, multiple learned dictionary vectors can be comparably good matches for a single ground-truth feature, leading to different, yet individually valid, feature sets being learned across runs. The lower PW-MCC here appropriately reflects this reduced stability in feature selection.  

Conversely, in a \textbf{globally compressive regime} ($d_{\text{sae}} < d_{\text{gt}}$), where the SAE has insufficient capacity (e.g., $d_{\text{sae}}=80, d_{\text{gt}}=800, k=8,n=5e4$), both GT-MCC ($0.75$) and PW-MCC ($0.60$) are diminished due to the inability to represent all true features (\autoref{fig:mcc_combined}).  The parallel decline of both metrics indicates their shared sensitivity to fundamental capacity limitations.
These global capacity mismatches show that consistency in practice might be harder to achieve, and PW-MCC provides a direct measure of this practical stability.

\textbf{Zipfian Feature Frequencies and Non-Uniform Capacity Allocation.}
A primary characteristic of natural language data is the Zipfian (power-law) distribution of underlying feature frequencies (\autoref{fig:synthetic_token_distribution_main_text_pos_v5})—a few features are common, many are rare. We study the effect of this heterogeneity in the globally compressive regime where SAEs operate in, given the vast number of true concepts versus typical dictionary sizes \citep{bricken2023towards_monosemanticity_scaling}.
To model this, we partition our synthetic ground truth features into $\numClusters$ clusters uniformly, each containing $\gtDictWidth$ features, and impose an arbitrary ranking on these clusters. Data points are generated by first sampling a cluster $\clusterIdx$ with probability $\clusterProb$ (following a Zipfian distribution with exponent $\zipfExp$), and then sampling $k$ true features from that cluster.

Post-training analysis of TopK SAEs trained on this data show that SAEs do not allocate their dictionary capacity uniformly across these clusters. Instead, the effective capacity $\saeClusterCapacity$ (number of learned SAE features, matched via Hungarian algorithm, corresponding to ground-truth cluster $\clusterIdx$) is well-approximated by a power law: $\smash{ \saeClusterCapacity = \saeDictWidth \cdot {\clusterProb^{\allocExp}}/{\sum_{j} {p_{j}}^{\allocExp}}}$. Our experiments empirically find $\allocExp \approx 1.4$. Thus, in a globally compressive setting (e.g., $\numClusters=10$, $\gtDictWidth=80$ per cluster, total $d_{\text{gt}} = 800$; $d_{\text{sae}}=80, k=8$), more frequent clusters (higher $\clusterProb$) receive a proportionally larger share of the SAE's limited dictionary representation and, as a result, exhibit higher GT-MCC scores, indicating better feature recovery for more common concepts (\autoref{fig:cluster_metrics_beta1_1_main_text_pos_v5}). This differential ground-truth recovery suggests that run-to-run consistency would similarly depend on the frequency of features, a pattern that feature-level PW-MCC analysis would capture. For additional details see \autoref{app:supplementary}.

\begin{figure}[tbp]
    \centering
    \begin{minipage}[t]{0.48\textwidth}
        \centering
        \begin{minipage}[t]{0.48\textwidth}
            \includegraphics[width=\textwidth]{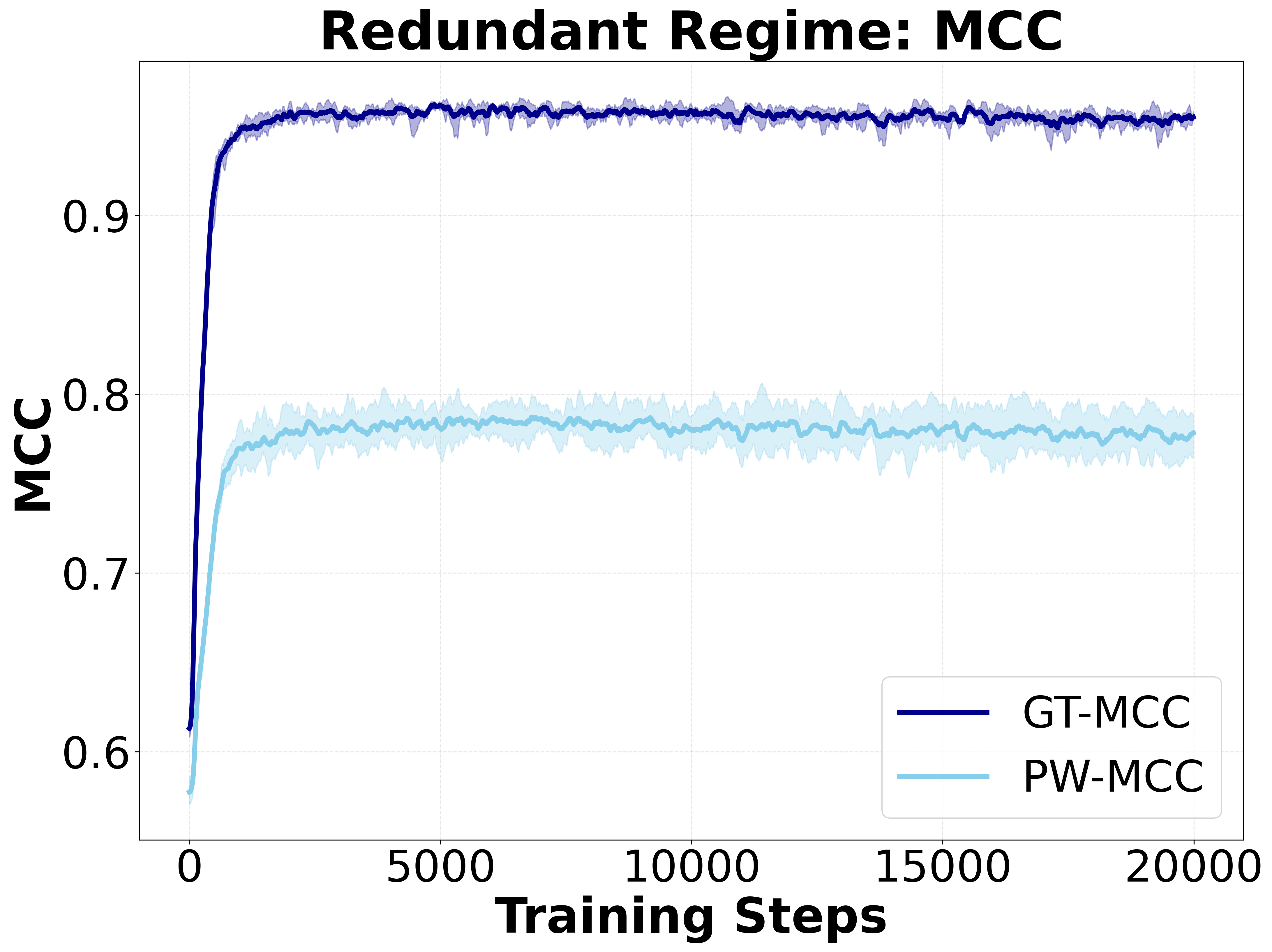}
        \end{minipage}%
        \hfill
        \begin{minipage}[t]{0.48\textwidth}
            \includegraphics[width=\textwidth]{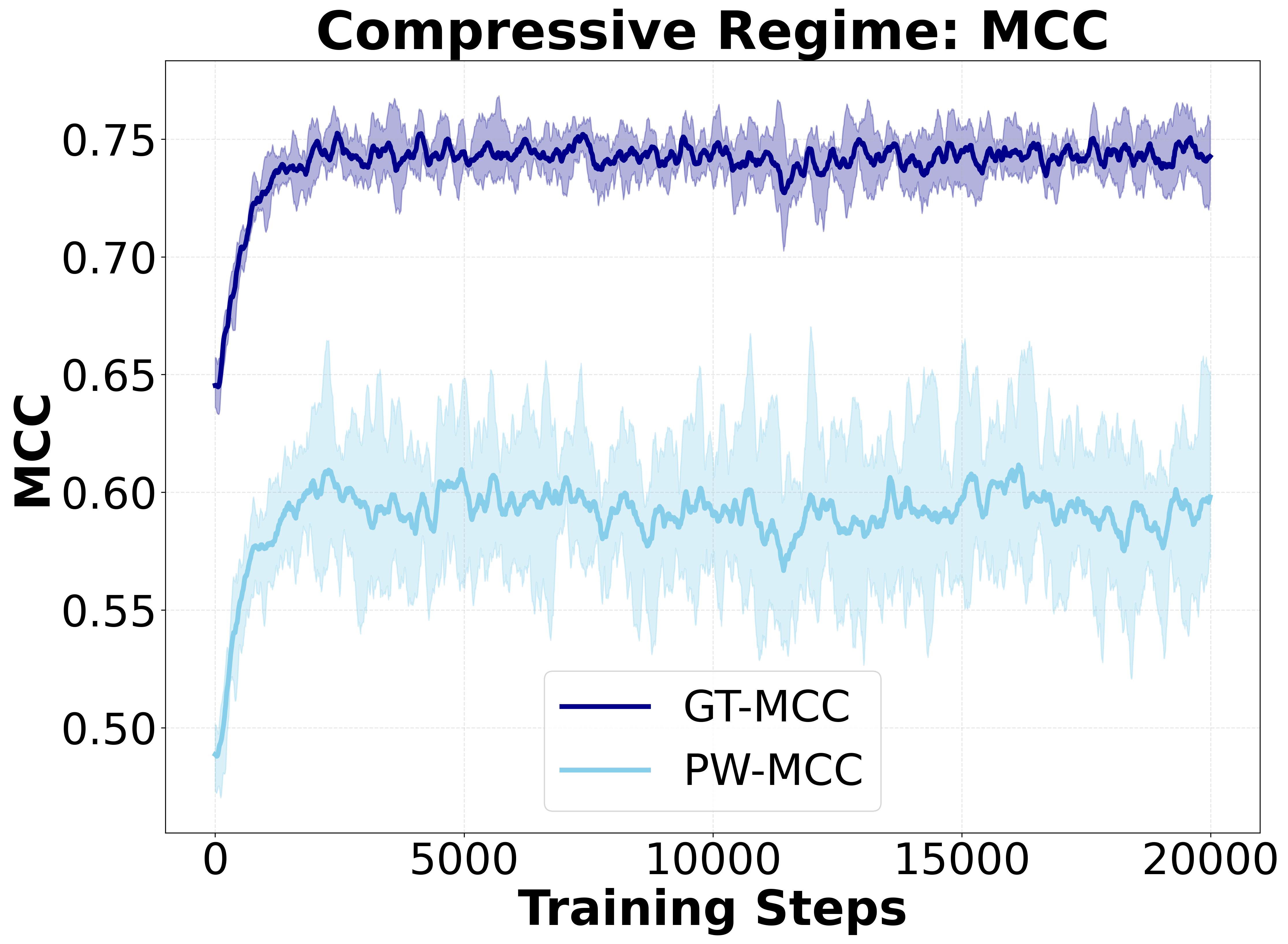}
        \end{minipage}
        \caption{\small Left: Redundant regime with high GT-MCC but lower PW-MCC due to selection ambiguity. Right: Compressive regime with lower GT-MCC and PW-MCC. Max-min range across 5 seeds is shaded.}
        \label{fig:mcc_combined}
    \end{minipage}
    \hfill
    \begin{minipage}[t]{0.48\textwidth}
        \centering
        \begin{minipage}[t]{0.48\textwidth}
            \includegraphics[width=\textwidth]{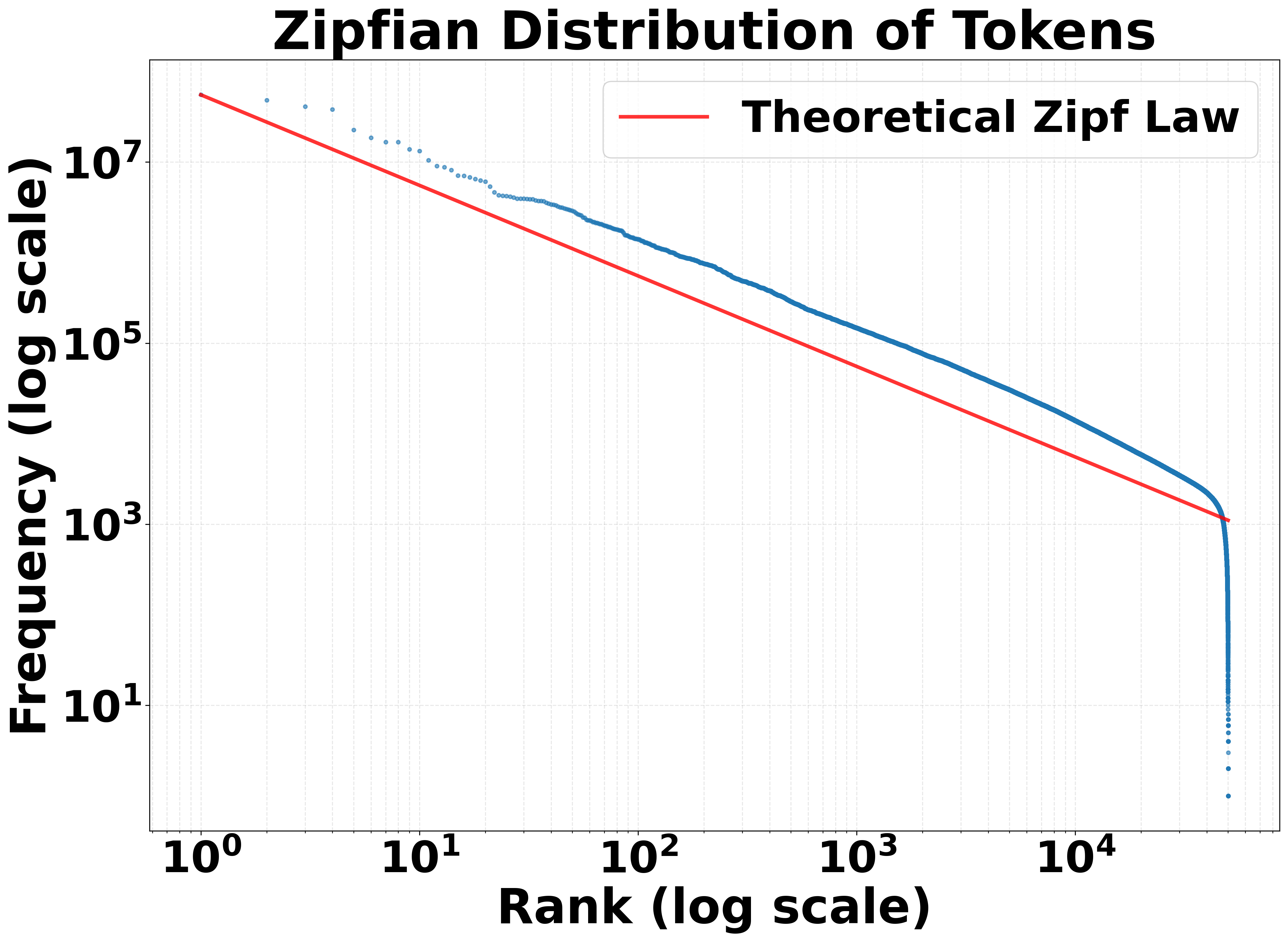}
        \end{minipage}%
        \hfill
        \begin{minipage}[t]{0.48\textwidth}
            \includegraphics[width=\textwidth]{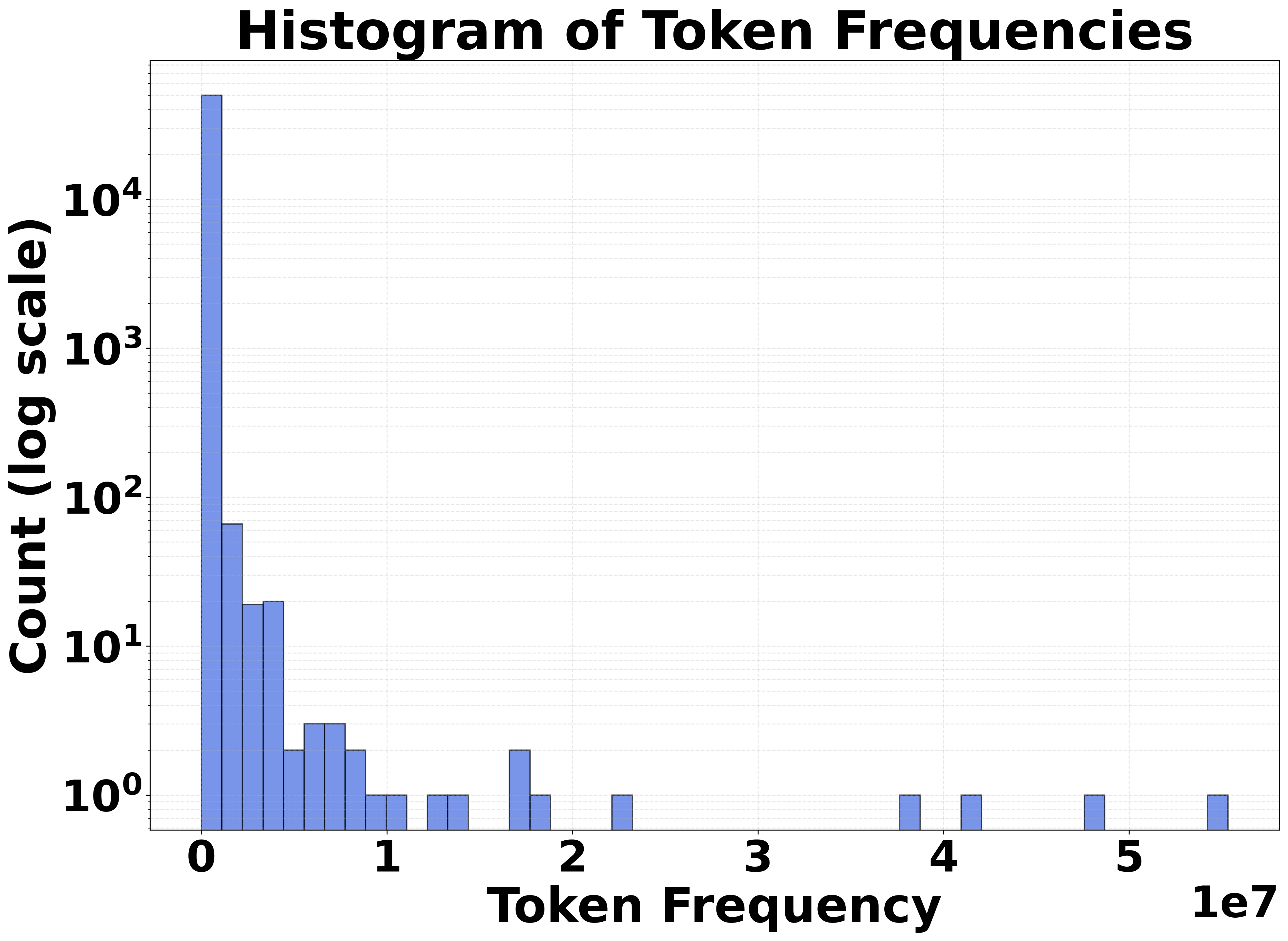}
        \end{minipage}
        \caption{\small Token frequency in 1M tokens from Pile, showing the Zipfian distribution in real data, with a long and sparse tail.}
        \label{fig:synthetic_token_distribution_main_text_pos_v5}
    \end{minipage}
\end{figure}

\begin{figure}[tbp]
    \centering
    \begin{minipage}[t]{0.49\textwidth}
        \centering
    \includegraphics[height=3cm]{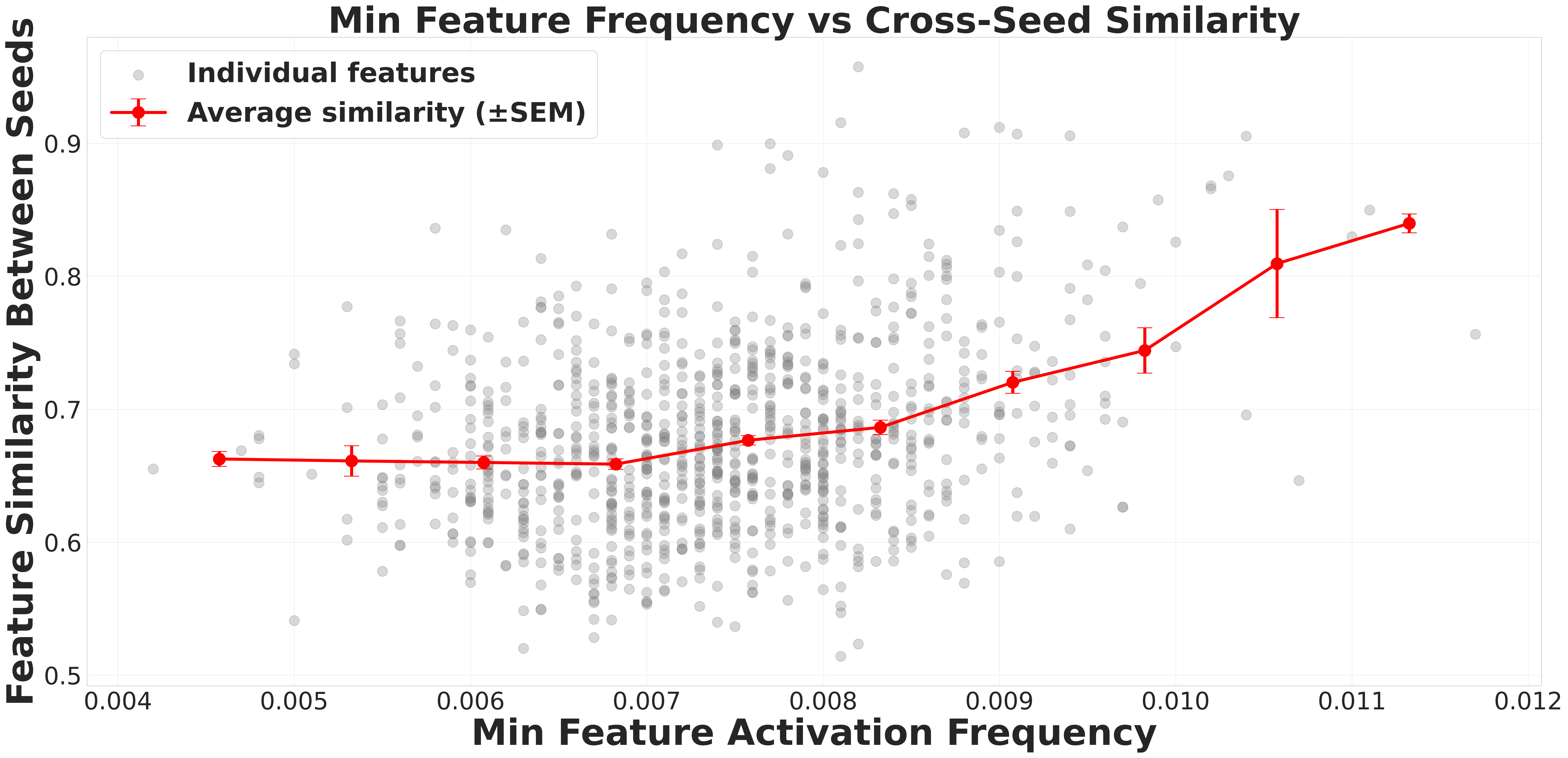}
    \caption{\small Min activation frequency between matched feature pairs vs. pairwise similarity. Data from two-phase Zipfian model ($d_{\text{gt}}=5000$, $d_{\text{sae}}=1000$). Feature-level similarity captures the influence of local consistency regimes across the frequency spectrum.}
    \label{fig:two_phase_1000_main_text_pos_v5}
        
    \end{minipage}
    \hfill
    \begin{minipage}[t]{0.49\textwidth}
        \centering
        \includegraphics[height=3cm]{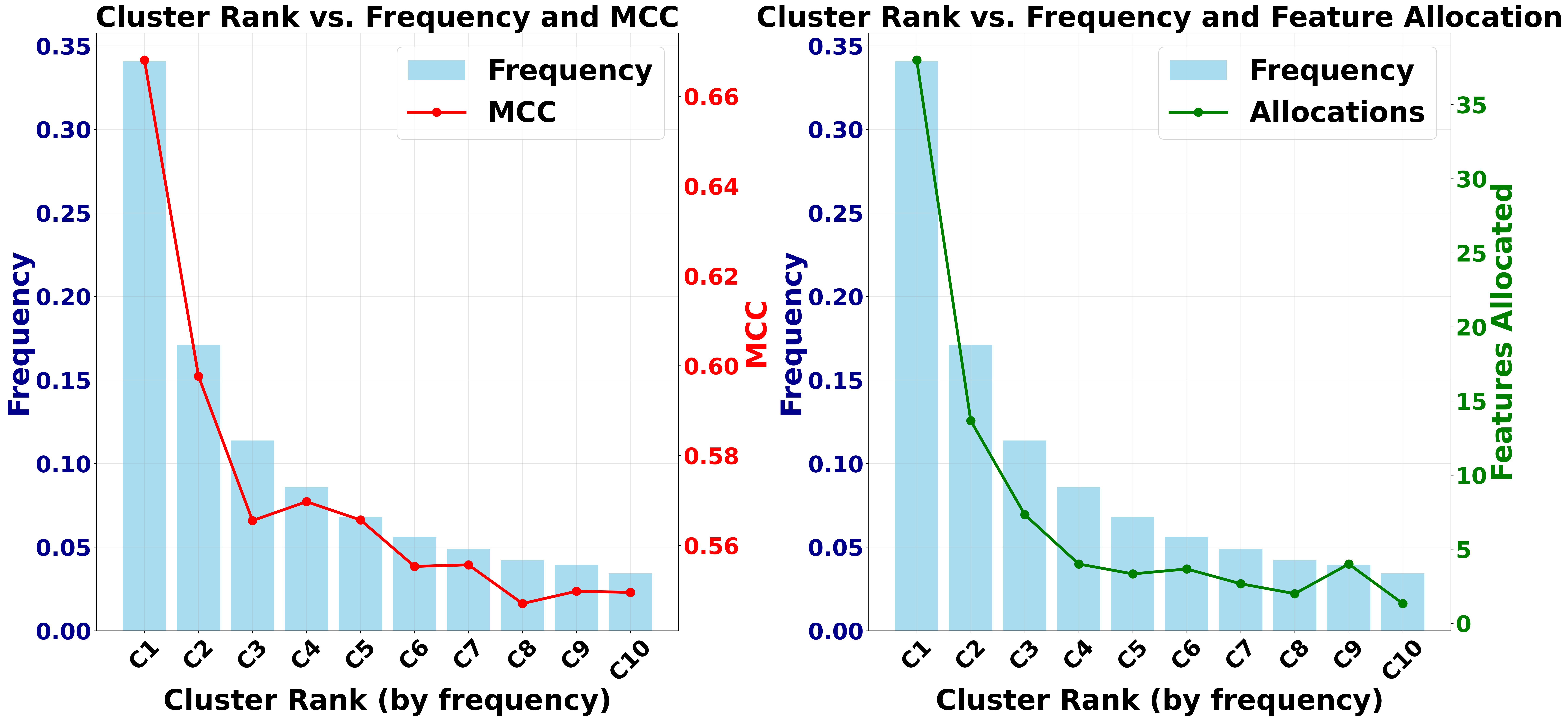}
        \caption{\small Globally Compressive Zipfian Model ($\zipfExp=1$, 10 clusters, $d_{gt}=800, d_{SAE}=80$) shows cluster frequency-dependent GT-MCC (red line, left y-axis) and SAE feature allocation (green line, right y-axis).}
        \label{fig:cluster_metrics_beta1_1_main_text_pos_v5}
    \end{minipage}
\end{figure}

\textbf{Emergence of Local Identifiability Regimes and Frequency-Dependent Consistency.}
This non-uniform capacity allocation driven by Zipfian frequencies means that different ground-truth clusters experience varied effective representational capacity within the same SAE. We define a \textbf{local redundancy factor} for each ground-truth cluster $\clusterIdx$ (containing $\gtDictWidth$ true features) as $\localRedundancyFactor := \saeClusterCapacity / \gtDictWidth$. This factor characterizes three distinct \textbf{local identifiability regimes}:
 \textit{Locally Redundant} ($\localRedundancyFactor > 1$) where the SAE has allocated more features to this cluster than true underlying features, risking selection ambiguity for this cluster's concepts; \textit{Locally Matched} ($\localRedundancyFactor \approx 1$) where the allocated capacity approximately matches the cluster's complexity; and \textit{Locally Compressive} ($\localRedundancyFactor < 1$) where insufficient capacity allocated for the cluster prevents full recovery.
These local regimes coexist within a single, even globally compressive SAE. Frequent clusters may become locally redundant or matched, while rare clusters inevitably remain locally compressive. This coexistence translates to frequency-dependent consistency (PW-MCC) of individual learned features.

\textbf{Probing the Full Spectrum of Consistency: A Two-Phase Zipfian Model.}
To more robustly investigate how these local regimes affect the consistency of individual features across a wide dynamic range, especially for very rare concepts, we further enhance our model organism. As real-world feature distributions exhibit an extremely long and sparse tail, we employed a two-phase feature cluster frequency distribution for 50000 clusters with $d_\text{gt}=400000$: a Mandelbrot-Zipf function ($g(r; s_1, q) = (r+q)^{-s_1}$, with $s_1=1.05, q=5.0$) for common concepts (rank $r < 40,000$), transitioning to a steeper power law ($g(r; s_2) \propto r^{-s_2}$, with $s_2=30.0$) for the long tail (see Appendix~\ref{app:two_phase}). Training a TopK SAE with a relatively ample dictionary width ($d_{\text{sae}}=1000$) on this data reveals a clear positive correlation between the minimum activation frequency of matched feature pairs and their inter-run representational similarity as shown in \autoref{fig:two_phase_1000_main_text_pos_v5}.

This correlation emerges naturally from the interplay of local regimes:
Frequent features are more likely to be in locally redundant or matched regimes. They receive sufficient allocated capacity, leading to stable learning and high GT-MCC. As true feature frequency decreases, the corresponding learned features are more likely to transition into locally compressive regimes. Here, insufficient allocated capacity relative to the conceptual complexity leads to lower inter-run similarity scores. Features in the extreme tail become so deeply locally compressive that they may be learned inconsistently across runs or not at all, resulting in minimal inter-run similarity.
This spectrum of varying stability is effectively quantified by analyzing PW-MCC at the individual feature level.

\begin{conclusionbox}
\small
\textbf{Takeaway:} Our synthetic model organism validates PW-MCC as a robust diagnostic, capturing how global capacity, Zipfian skew, and local identifiability regimes affect feature consistency.
\end{conclusionbox}

\section{Evidence from Applications and Large-Scale Validation}
\label{sec:eval_consistency_real_world}
\subsection{Evaluating Consistency in Real-World Applications}
\label{subsec:eval_consistency_real_world_discussion}
Prevailing practices in SAE evaluation prioritize metrics like reconstruction error (L2 loss), Fraction of Variance Explained, and sparsity (L0/L1) \citep{bricken2023monosemanticity, cunningham2023sparse}. While reconstruction fidelity is one aspect of SAE quality, incorporating feature consistency, quantified by PW-MCC, into the evaluation process offers benefits for both practical model development and the interpretation of learned features.

\textbf{PW-MCC enables more decisive SAE model comparisons and hyperparameter selection where conventional metrics like reconstruction loss prove ambiguous.} This advantage is particularly pronounced when controlling feature sparsity. When tuning the L1 coefficient, lower penalties improve reconstruction loss while potentially degrading feature quality. Similarly, in TopK SAEs where reconstruction loss improves monotonically with $k$, reconstruction loss alone fails to identify optimal sparsity levels for feature quality.
PW-MCC addresses this limitation by identifying the underlying trade-off. Insufficient sparsity from excessively large $k$ or low L1 penalties causes feature selection ambiguity, while excessive sparsity from small $k$ or high L1 penalties leads to inconsistent concept representation. Both extremes degrade GT-MCC and PW-MCC, enabling clear identification of optimal $k$. We demonstrate how PW-MCC guides this selection in practice (Appendix~\ref{app:misspecifying_k}).

\textbf{PW-MCC acts as a justifiable proxy for ground-truth alignment in unsupervised settings}. Its utility stems from observations in our synthetic experiments where PW-MCC strongly correlated with GT-MCC, tracked its progression during training, and served as a practical lower bound. Consequently, low PW-MCC across independent training runs suggests that an SAE is unlikely to converge to a well-defined, \emph{true} feature set. This ability to signal robust feature learning, even without ground truth, underpins its use as a key evaluation metric in the real-world experiments presented next.

\subsection{Training SAEs on LLM Activations: Empirical Consistency Results}
\label{sec:training_saes_on_natural_language}
\begin{figure}[tbp]
    \centering
    \begin{minipage}{0.49\textwidth}
        \centering
        \includegraphics[width=\textwidth]{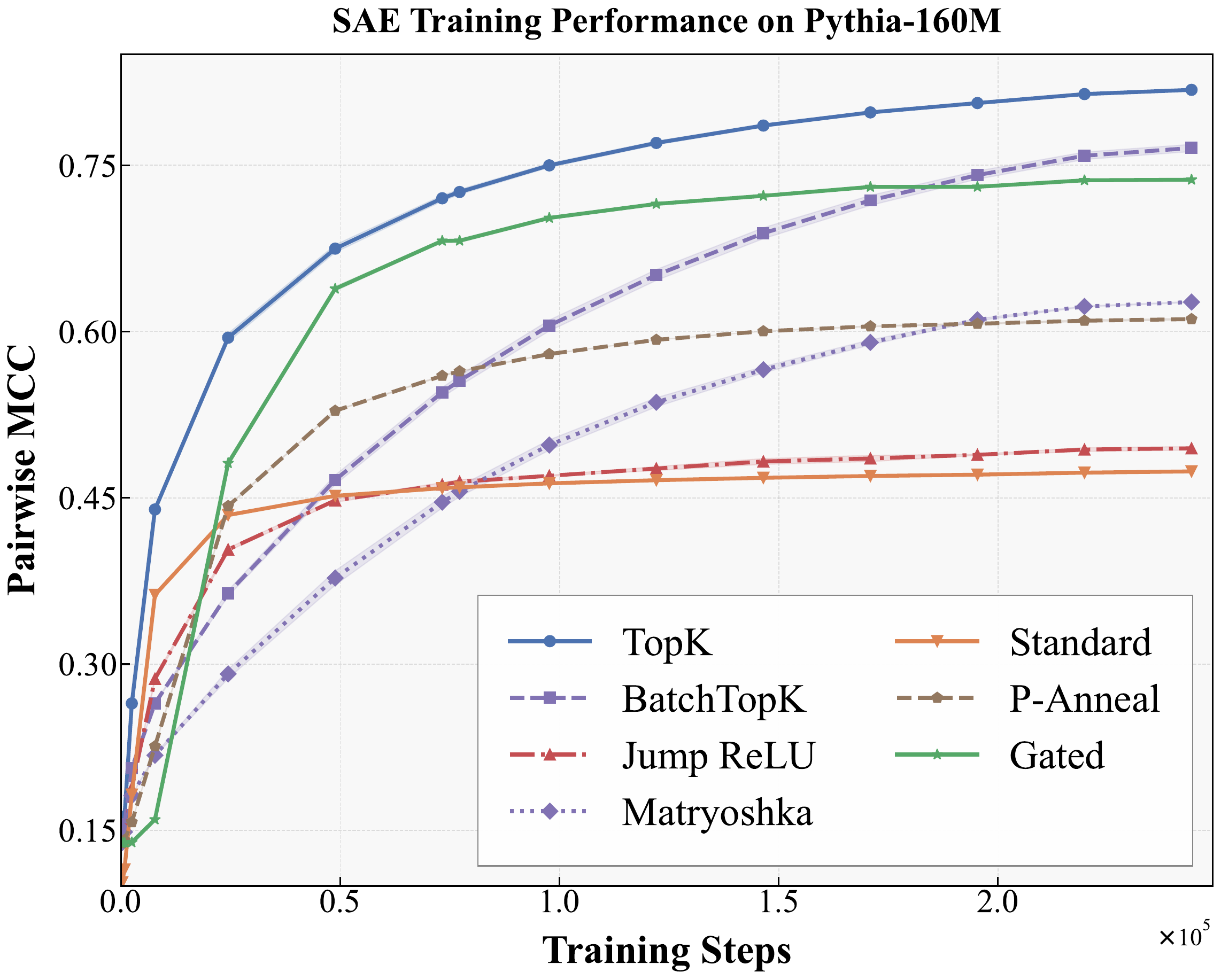}
        \caption{\small PW-MCC vs. train steps for BatchTopK, Gated, P-Anneal, JumpReLU, Standard, TopK, and Matryoshka BatchTopK SAEs on Pythia-160M activations. Higher PW-MCC indicates greater feature consistency.}
        \label{fig:pythia_160M}
    \end{minipage}
    \hfill
    \begin{minipage}{0.49\textwidth}
        \centering
        \includegraphics[width=\textwidth]{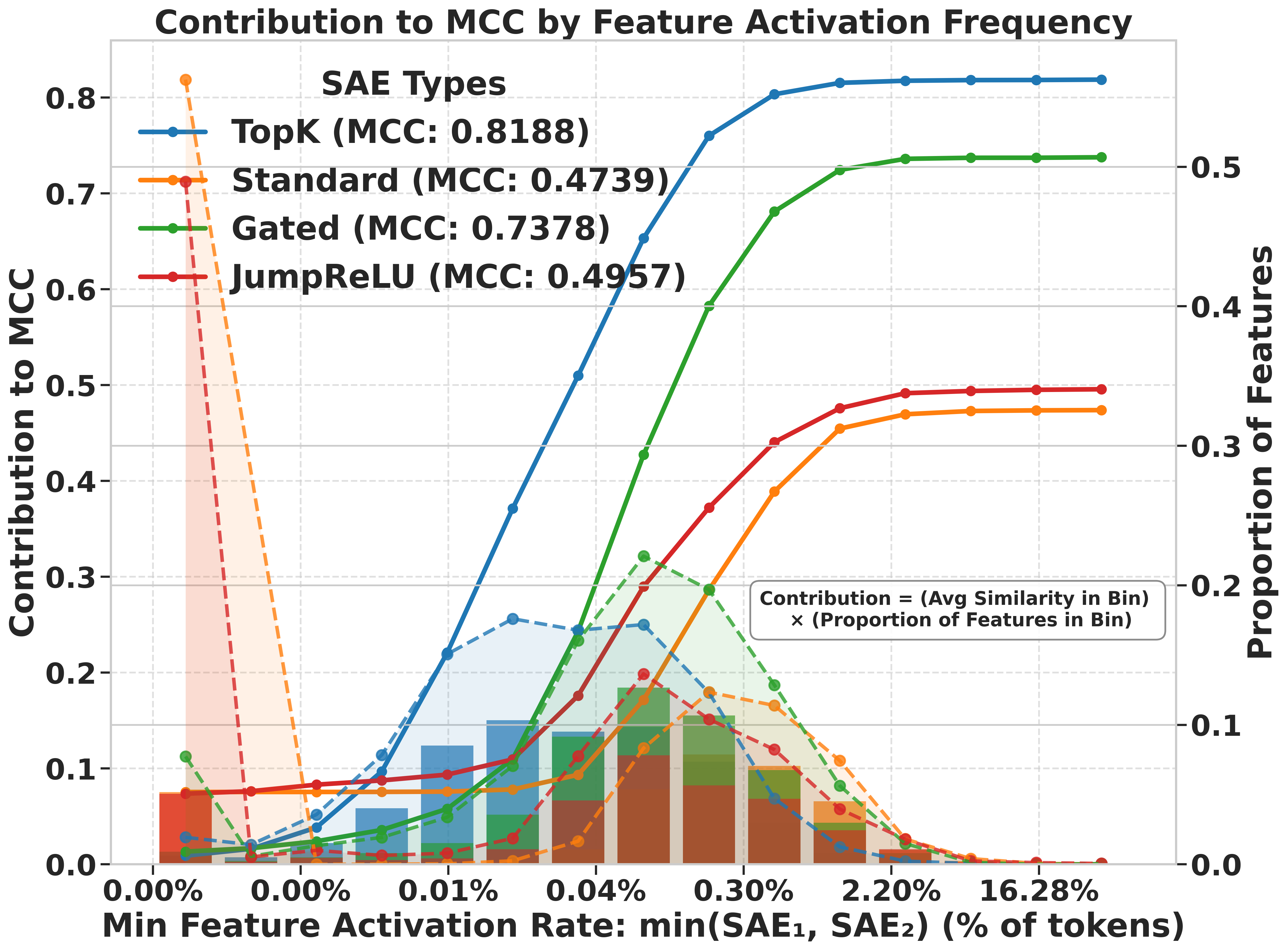}
        \caption{\small  PW-MCC contribution by feature activation frequency for TopK, Standard, Gated, and JumpReLU SAEs. Bars (left axis) show each bin's contribution; solid lines show cumulative contribution. Dashed lines (right axis) show feature distribution across bins.}
        \label{fig:mcc_contributions}
    \end{minipage}
\end{figure}

\begin{table}[b]
\small
    \begin{minipage}{0.48\textwidth}
    \centering
    \caption{\small TopK SAE (Pythia-160M): Higher activation frequency features show stronger mean cosine similarity (and lower variance) between matched pairs from independently trained SAEs, indicating greater consistency for more prevalent features.}
    \small
\begin{tabular}{@{}rrr@{}}
    \toprule
    \textbf{Act freq/1M tokens} & \textbf{Features} & \textbf{Similarity} \\ %
    \midrule
    0.1--2.4 & 127 & 0.514 $\pm$ 0.280 \\
    2.4--54.1 & 2,542 & 0.742 $\pm$ 0.295 \\
    54.1--1.2k & 10,013 & 0.837 $\pm$ 0.209 \\
    1.2k--26.7k & 3,548 & 0.888 $\pm$ 0.157 \\
    26.7k--592.2k & 33 & 0.964 $\pm$ 0.087 \\
    \bottomrule
    \end{tabular}%
    \label{tab:sae_analysis_combined}
    \end{minipage}
    \hfill
    \begin{minipage}{0.49\textwidth}
    \centering
    \caption{\small TopK SAE (Pythia-160M): Semantic similarity scores (GPT Score, 1-10 scale) for automatically generated explanations of matched feature pairs correlate strongly with the features' dictionary vector cosine similarity. GPT-score is averaged over 20 pairs.}
    \small
    \begin{tabular}{@{}crr@{}}
    \toprule
    \textbf{Similarity Range} & \textbf{Feat pairs} & \textbf{GPT Score} \\
    \midrule
    0.0654--0.1128 & 34 & 1.71 \\
    0.1128--0.1947 & 311 & 2.19 \\
    0.1947--0.3359 & 975 & 3.27 \\
    0.3359--0.5795 & 1,423 & 4.12 \\
    0.5795--0.9999 & 13,640 & 8.28 \\
    \bottomrule
    \end{tabular}%
    \label{tab:gpt-scores}
    \end{minipage}
\end{table}

\paragraph{Experimental Setup.}
We train SAEs on Pythia-160M model~\citep{biderman2023pythia}, with width $2^{14}$ on \(500\) million tokens from \texttt{monology/pile-uncopyrighted}~\citep{gao2020pile}, using residual stream activations from layer 8. For each SAE, we performed a hyperparameter sweep, selecting the configuration that yielded the highest final PW-MCC across three independent training runs. Further details on the training setup and hyperparameters as well as results for Gemma-2-2B are provided in \autoref{app:real-world}.

\paragraph{Overall Consistency and Training Dynamics.}
\autoref{fig:pythia_160M} shows the evolution of PW-MCC during training for several SAE architectures. We observe the steady increase in PW-MCC over training steps, indicating that as SAEs learn, their feature dictionaries converge and PW-MCC captures the emergence of consistent representations. Among the evaluated architectures, TopK and BatchTopK SAEs achieved the highest PW-MCC scores. The PW-MCC for some architectures, like BatchTopK, has not fully saturated by $2.5 \times 10^5$ training steps, indicating that longer training might yield even higher consistency. The curves also reveal interesting dynamics; for instance, some methods (e.g., BatchTopK) may start with lower consistency but exhibit faster improvement, eventually surpassing others (e.g., Gated SAE), suggesting that different SAEs impose varied structural assumptions~\citep{hindupur2025projecting}.

\paragraph{Consistency Across the Feature Frequency Spectrum.}
The insights from our synthetic model organism, particularly the relationship between feature frequency and consistency, are reflected in these real-data experiments. \autoref{tab:sae_analysis_combined} quantifies this: features with higher activation frequencies exhibit markedly stronger inter-run similarity. For example, the rarest features show an average similarity of $0.514$, while the most frequent features achieve a much higher $0.964$. This trend is broadly observed across architectures (see Appendix~\ref{sec:pythia_results}). This confirms that frequently occurring concepts are generally learned more stably, and PW-MCC reveals this spectrum of consistency. \\
\autoref{fig:mcc_contributions} further dissects this relationship by showing the contribution of different feature frequency bins to the overall PW-MCC. For the TopK SAE, we observe a relatively symmetric contribution from features across a wide range of activation frequencies, with few dead features. In contrast, the Standard SAE exhibits a larger proportion of features in the lowest frequency bins which contribute minimally to the cumulative PW-MCC, effectively pulling down its overall consistency. Gated SAEs perform well, approaching TopK SAEs but with a slightly larger tail of less active, less consistent features. PW-MCC thus enables a nuanced understanding of how different SAEs utilize their dictionary and achieve consistency across the frequency spectrum.

\paragraph{Correlation with Semantic Similarity.}
\looseness=-1
We find that feature consistency is related to learning semantically related concepts. \autoref{tab:gpt-scores} shows that when we matched features between two independent SAE runs, binned these pairs by their cosine similarity, used an automated interpretation pipeline (details in Appendix~\ref{app:qualitative}) to generate explanations for each feature in a pair, and used an LLM to rate the semantic similarity of these two explanations (GPT Score), the results show a strong positive correlation. Feature pairs with high dictionary vector similarity receive high semantic similarity scores from the LLM, while pairs with low vector similarity are judged as semantically divergent. This validates that PW-MCC and feature-level similarity are indicative of genuine stability in the learned semantic representations, making them valuable for ensuring the reliability of interpretability findings.

\begin{conclusionbox}
\small
\textbf{Takeaway:} High PW-MCC is achievable on LLM activations with TopK SAEs. PW-MCC tracks feature consistency, reveals frequency-dependent consistency, and correlates with feature semantic similarity.
\end{conclusionbox}

\section{Alternative Views}
\label{sec:alternative_views}

While we advocate for prioritizing feature consistency in SAEs, we acknowledge and address potential counterarguments from the community.

\textbf{Some argue that SAE feature consistency is fundamentally unachievable, viewing features merely as a useful, pragmatic decomposition}\citep{paulo2025sparse}. This view is bolstered by findings that multiple, incompatible mechanistic explanations can coexist for the same model behavior \citep{meloux2025everything, makelov2023subspace}, questioning the existence of a single, canonical feature set.
While perfect universality for every feature on real data is indeed challenging, our work demonstrates that \textit{high levels of consistency are attainable} with appropriate methods and evaluation (e.g., TopK SAEs achieving PW-MCC $\approx 0.80$; Sections~\ref{sec:training_saes_on_natural_language}). A pragmatic decomposition gains significant scientific utility when its components are demonstrably stable. The focus, therefore, should be on understanding, maximizing, and characterizing this stability.

\looseness=-1
\textbf{Another perspective is that \emph{sufficiently good} interpretability can be achieved without demanding perfect feature consistency, and an excessive focus on stability might stifle exploratory research}\citep{lipton2018mythos, freiesleben2024scientific}.
Initial exploration using single-run features certainly has value. However, for claims aspiring to scientific robustness—such as those underpinning causality~\citep{marks2025sparse, geiger2023causal}, safety verification~\citep{abdaljalil2025safe,harle2024scar}, or canonical understanding~\citep{o2024steering, olah2022mechanistic}---sufficiently good stability must be quantitatively defined and verified. The current degree of instability in many applications is often underestimated \citep{paulo2025sparse}. We advocate for establishing measurable baselines for consistency to add rigor for cumulative progress.

\textbf{It is also suggested that the pursuit of low-level feature stability might divert from the arguably more important goal of understanding higher-level conceptual abstractions or circuits}~\citep{olah2020zoom}.
We contend, however, that reliable higher-level understanding requires robust lower-level foundations. If the fundamental feature vocabulary is ill-defined or shifts between runs, any circuits or mechanisms built upon them become inherently suspect and difficult to validate \citep{olah2018building}. Stable features are important, dependable anchors for trustworthy compositional analyses and the mapping of learned circuits.

\textbf{Finally, there's a view that the field will organically converge on more consistent SAE methods without explicit mandates for consistency benchmarking.}
While scientific fields do self-correct over time, feature instability remains a significant and often under-reported issue \citep{fel2025archetypal}, even in widely-used methods. An active, concerted push for prioritizing consistency, supported by standardized metrics (like PW-MCC) and benchmarks, can substantially accelerate progress, guiding the community towards more scientifically sound practices.

\section{Conclusion and Call to Action}
\label{sec:call_to_action}

\looseness=-1
Prioritizing feature consistency in SAEs is important for advancing MI towards a more robust engineering science. This requires a shift in how we develop and evaluate feature extraction methods. \textbf{We call upon the community to routinely report quantitative consistency scores} (e.g., PW-MCC), ideally contextualized by feature frequency, alongside standard metrics, enabling meaningful comparisons. \textbf{Furthermore, we propose developing standardized benchmarks for consistency}, such as challenging synthetic model organisms with known ground-truth features and data heterogeneity. \textbf{Finally, focused research is needed to deeply understand the determinants of consistency}, including the interplay between SAE architecture, optimization, data characteristics, and evaluation metrics.

Our work highlights several fertile avenues for future research: (a) designing SAEs for robust consistency across diverse LLM activation statistics (e.g., early vs. late layers) and developing adaptive sparsity mechanisms responsive to local data properties; (b) improving consistency for less frequent yet potentially critical features, and exploring techniques to target specific parts of the feature spectrum; (c) defining broader notions of feature equivalence beyond strong feature consistency (e.g., functional, subspace alignment) and corresponding consistency metrics; (d) understanding how the SAE encoder's amortization gap \citep{o2024compute} influences dictionary stability and whether encoder improvements or explicit consistency regularizers can enhance it; (e) establishing stronger theoretical guarantees for the consistency of features from modern SAEs under realistic data assumptions.
Addressing these challenges and embracing a research culture that values and quantifies feature consistency will be pivotal in building a more reliable and cumulative science of MI.

\bibliographystyle{plain} %
\bibliography{references}   %

\newpage
\appendix

\section{Additional Related Work}
\subsection{Sparse Autoencoders for Mechanistic Interpretability}

This section provides further context on the specific SAE architectures evaluated in our work, complementing the broader discussion of SAEs in Section \ref{sec:related_work}.

SAEs aim to learn overcomplete dictionaries that can decompose high-dimensional neural network activations into sparser, potentially more interpretable feature representations. The core principle involves training an autoencoder to reconstruct an input activation $\mathbf{x}$ while simultaneously encouraging the latent representation $f(\mathbf{x})$ to be sparse. Various SAE architectures implement different mechanisms to achieve this sparsity objective. The key architectures employed in our experiments are described below.

\paragraph{Standard SAE.}
The architecture we refer to as \emph{Standard SAE} is an L1-penalized ReLU SAE that incorporates several contemporary training practices aimed at improving stability and reducing the incidence of inactive (dead) features \citep{bricken2023monosemanticity, marks2024dictionary_learning}. A distinguishing characteristic of this variant is the application of the L1 penalty to feature activations $f(\mathbf{x})$ after they have been scaled by the L2 norm of their corresponding decoder dictionary vectors: $\lambda_1 \sum_j |f_j(\mathbf{x})| \cdot \|\mathbf{a}_j\|_2$, where $\mathbf{a}_j$ is the $j$-th column of the decoder matrix. Unlike some earlier L1 SAEs that explicitly constrain decoder column norms to unity during optimization, this approach omits such a constraint, integrating the decoder norm directly into the sparsity term. We use Adam optimization and gradient clipping.

\paragraph{TopK SAE.}
TopK SAEs~\citep{gao2024scaling} enforce sparsity structurally, rather than through a continuous penalty term. For each input, only the $k$ features with the highest pre-activation values (typically after a ReLU non-linearity) are selected to be active, while all other feature activations are set to zero. The integer $k$ directly determines the L0 norm of the feature activation vector. This design obviates the need for tuning an L1 coefficient but introduces $k$ as a crucial hyperparameter. We do not incorporate additional auxiliary loss terms designed to prevent feature death in this work. 

\paragraph{BatchTopK SAE.}
The BatchTopK SAE architecture \citep{bussmann2024batchtopk} adapts the TopK principle by enforcing the $k$-sparsity constraint on average across a batch of inputs, rather than strictly on a per-sample basis. This is achieved by learning a global activation threshold that is dynamically adjusted during training (using an Exponential Moving Average of feature pre-activations) to ensure that, averaged over a batch, approximately $k$ features are active per input sample. This allows for greater variability in per-sample sparsity while maintaining a target average L0 norm. 

\paragraph{Gated SAE.}
Gated SAEs \citep{rajamanoharan2024improving} are designed to decouple the decision of whether a feature activates from the magnitude of that activation. They employ two distinct pathways for processing the input: a \emph{gate} pathway, which produces values (near 0 or 1 via an L1 or similar sparsity penalty on the gate outputs) that determine if each feature should be active, and a \emph{magnitude} pathway, which computes the strength of each feature if it is gated on. The final feature activation is then the element-wise product of the outputs from these two pathways. The rationale behind this design is to allow features to activate with strong magnitudes when relevant, without these magnitudes being directly suppressed by the primary sparsity-inducing penalty, as that penalty is instead applied to the gating mechanism.

\paragraph{P-Anneal SAE.}
This SAE variant~\citep{karvonen2024measuring} modifies L1-penalized SAEs by employing a dynamic sparsity penalty based on an $L_p^p$-norm. In this approach, the exponent $p$ in the sparsity term $\lambda_s \|f(\mathbf{x})\|_{p_s}^{p_s}$ is annealed during the training process. Training typically commences with $p_s=1$ (equivalent to L1 minimization, which offers a convex optimization landscape) and $p_s$ is progressively decreased towards a target value $p_{end} < 1$ (e.g., $p_{end}=0.2$ in the original work). This annealing schedule aims to first guide the optimization towards a good region using the L1 penalty, and then gradually shift towards a non-convex objective that more closely approximates L0 sparsity, potentially yielding sparser solutions. To ensure the effective strength of the sparsity penalty remains relatively consistent as $p_s$ changes, the scaling coefficient $\lambda_s$ is also adaptively adjusted during training based on statistics derived from recent batches of feature activations.

\paragraph{JumpReLU SAE.}
JumpReLU SAEs \citep{rajamanoharan2024jumping} employ a JumpReLU  activation function which uses per-feature learnable thresholds, $\theta_j$. For an input language model activation $x \in \mathbb{R}^n$, the encoder computes pre-activations $\pi_j(x) = (W_{\text{enc}} x + b_{\text{enc}})_j$ for each feature $j$. The feature activation $f_j(x)$ is then given by $f_j(x) = \text{JumpReLU}_{\theta_j}(\pi_j(x)) = \pi_j(x) H(\pi_j(x) - \theta_j)$, where $H$ is the Heaviside step function and $\theta_j > 0$ is the learned threshold for feature $j$. Sparsity in the feature representation $f(x)$  is encouraged by an $L_0$ penalty on the feature activations: for instance, using a loss term like $\lambda (\|f(x)\|_0 / L^{\text{target}}_0 - 1)^2$ to drive the average number of active features towards a predefined target $L^{\text{target}}_0$. The non-differentiable nature of both the JumpReLU (with respect to $\theta_j$) and the L0 penalty is handled during training using Straight-Through Estimators. This architecture allows the model to learn distinct activation sensitivities for different features, as each $\theta_j$ can be optimized independently.

\paragraph{Matryoshka BatchTopK SAE.}
Matryoshka SAEs~\citep{bussmann2025learning} introduce a hierarchical structure to the learned dictionary. In this paradigm, multiple, nested dictionaries of progressively increasing sizes are trained simultaneously within a single model. Features are ordered or grouped, and the training objective is designed to encourage more general or broadly important features to be learned by the smaller, inner dictionaries (analogous to the inner dolls in a Matryoshka set). More numerous or specialized features are then captured by the larger, outer dictionaries. This hierarchical approach aims to learn features at multiple levels of granularity and can offer computational efficiencies at inference time if a smaller, inner dictionary provides sufficient representational power for a given task. The Matryoshka BatchTopK variant evaluated in our study combines this hierarchical dictionary organization with the BatchTopK mechanism for selecting active features.

Each architecture comes with its own set of hyperparameters, computational considerations, and characteristic effects on the learned feature space.

\subsection{Extended Discussion on Dictionary Learning Identifiability}
The feature consistency challenges observed in SAEs can be understood through the theoretical lens of dictionary learning identifiability. Dictionary learning identifiability addresses a fundamental question: under what conditions can we guarantee that a learning algorithm will recover the true underlying dictionary (or an equivalent version up to permutation and scaling) from observed data? This question directly parallels our inquiry into when SAEs can consistently learn the same features across different initializations.
Dictionary learning can be formalized as the problem of finding a dictionary matrix $\mathbf{A} \in \mathbb{R}^{m \times d}$ and a sparse coefficient matrix $\mathbf{F} \in \mathbb{R}^{d \times n}$ such that $\mathbf{X} \approx \mathbf{A}\mathbf{F}$, where $\mathbf{X} \in \mathbb{R}^{m \times n}$ represents observed data. In the overcomplete setting ($d > m$), which is most relevant to SAEs, the problem becomes particularly challenging because infinitely many solutions can potentially fit the data equally well.

Several lines of theoretical work establish conditions under which overcomplete dictionary recovery is possible. For example, as discussed in our paper, the Spark condition introduced by \citep{donoho2003optimally} states that when spark($\mathbf{A}$) is sufficiently large relative to the sparsity level, unique recovery becomes possible. Specifically, a unique sparse representation is guaranteed when spark($\mathbf{A}$) $> 2k$, where $k$ is the sparsity level.
Building on this foundation, \citep{sun2024global} recently established more comprehensive identifiability results for overcomplete dictionary learning. Their work introduces conditions under which global identifiability holds, showing that the identifiability of dictionaries depends on both the structure of the dictionary itself and the generative mechanism for coefficient vectors.
A key insight from \citep{sun2024global} is that traditional dictionary identifiability frameworks rely on verifying two conditions: (1) coefficients are sufficiently diverse to span the full space of possibilities, and (2) dictionaries satisfy appropriate structural conditions such as the Spark condition. When both conditions hold, the dictionary can be uniquely determined up to permutation and scaling—exactly the type of consistency we seek in SAE features.
The Restricted Isometry Property (RIP) provides another important set of conditions. A dictionary matrix $\mathbf{A}$ satisfies RIP of order $k$ with constant $\delta_k$ if $(1-\delta_k)|\mathbf{x}|_2^2 \leq |\mathbf{A}\mathbf{x}|_2^2 \leq (1+\delta_k)|\mathbf{x}|_2^2$ for all $k$-sparse vectors $\mathbf{x}$. When RIP holds with sufficiently small $\delta_k$, consistent recovery becomes possible even in overcomplete regimes \citep{arora2014new}. In our discussion, we opted to use the Spark condition due to its clearer connections to the training objectives of SAEs and the simplicity for implementing the condition in the learning algorithm.

We also draw inspiration from the independent component analysis (ICA) literature for defining our evaluation metric MCC, as dictionary learning has deep connections to ICA, particularly in its overcomplete form. ICA assumes that observed signals are linear mixtures of statistically independent source signals and aims to recover both the mixing matrix and the source signals \citep{hyvarinen2000independent}, an objective shared with dictionary learning for finding the dictionary and sparse coefficient matrix~\citep{wang2024identifiability}.

\section{Formal Definitions of Feature Consistency}
\label{app:formal_consistency_definitions}

This appendix provides more formal definitions for the concepts of $\mathcal{T}$-Feature Consistency and Strong Feature Consistency, briefly introduced in Section~\ref{sec:our_position}. These definitions help to precisely articulate what it means for two sets of learned features to be considered \emph{equivalent.}

\begin{definition}[$\mathcal{T}$-Feature Consistency]
\label{def:general_feature_consistency_appendix}
Let $\mathbf{A}$ and $\mathbf{A}'$ be two dictionaries (matrices whose columns are feature vectors), each containing $d$ features, learned from the same dataset $\mathcal{D}$ using the same algorithm and hyperparameters but with different random initializations. These dictionaries are said to be \textbf{$\mathcal{T}$-consistent} ($\mathbf{A} \sim_{\mathcal{T}} \mathbf{A}'$) if there exists a permutation $\sigma \in S_d$ (the set of all permutations of $\{1, \dots, d\}$) and a specified transformation $\mathcal{T}$ such that for all feature indices $i \in \{1, \dots, d\}$:
\begin{equation*}
    \mathbf{a}_i = \mathcal{T}(\mathbf{a}_{\sigma(i)}'),
\end{equation*}
where $\mathbf{a}_j^{(k)}$ denotes the $j$-th feature vector (column) of dictionary $\mathbf{A}^{(k)}$.
\end{definition}

The transformation $\mathcal{T}$ can take various forms. For instance, in some contexts, $\mathcal{T}$ might represent a more complex function if features are not considered atomic or have internal structure. However, for dictionary learning in sparse autoencoders, where features are typically represented by individual vectors (dictionary atoms), a more specific and common notion of equivalence is based on permutation and scaling.

\begin{definition}[Strong Feature Consistency]
\label{def:strong_feature_consistency_appendix}
The dictionaries $\mathbf{A}$ and $\mathbf{A}'$ from Definition~\ref{def:general_feature_consistency_appendix} are said to exhibit \textbf{Strong Feature Consistency} if the transformation $\mathcal{T}$ corresponds to an individual, per-feature scaling. That is, there exists a permutation $\sigma \in S_d$ and a set of non-zero scaling factors $\lambda_1, \dots, \lambda_d \in \mathbb{R} \setminus \{0\}$ such that for all $i \in \{1, \dots, d\}$:
\begin{equation*}
    \mathbf{a}_i = \lambda_i \mathbf{a}_{\sigma(i)}'.
\end{equation*}
\end{definition}

This definition implies that each feature learned in one run has a one-to-one correspondent in the other run that points in the same (or exactly opposite, if $\lambda_i < 0$) direction, differing only in magnitude. If feature vectors are constrained to have unit $\ell_2$-norm (either by explicit normalization during training or as a post-processing step before comparison), the scaling factors $\lambda_i$ would effectively become $\pm 1$. The Mean Correlation Coefficient (MCC) metrics used in this paper—namely PW-MCC and GT-MCC---are designed to measure this Strong Feature Consistency. The use of cosine similarity $|\langle \mathbf{u}, \mathbf{v} \rangle| / (\|\mathbf{u}\|_2 \|\mathbf{v}\|_2)$ inherently accounts for differences in magnitude (norm), and the absolute value handles the sign ambiguity (features pointing in opposite directions are still considered perfectly correlated in direction).

We prioritize Strong Feature Consistency---alignment up to permutation and scaling---as a foundational and empirically tractable starting point. This notion directly connects to identifiability results in dictionary learning and allows for straightforward quantification using metrics like MCC. While other, broader notions of consistency, such as functional equivalence (where features achieve similar outcomes despite different dictionary vectors) or subspace alignment, are undoubtedly important and represent valuable avenues for future research, establishing robust vector-level consistency is a critical first step.

\section{How the Round-Trip Condition Guarantees Spark in SAEs}
\label{sec:roundtrip-spark-proof}
This section proves that the round-trip property directly implies the spark condition for TopK SAEs. The argument is purely algebraic.

\subsection{Setting and Notation}
Throughout this section, we fix a sparsity level $k \geq 1$ and denote by $
\Sigma_k := \{\mathbf{f} \in \mathbb{R}^d : \|\mathbf{f}\|_0 \leq k\}
$
the set of $k$-sparse vectors.

\paragraph{SAE Components.} Let $\mathbf{A} \in \mathbb{R}^{m \times d}$ be the decoder (dictionary) learned by a TopK SAE, and let $\mathbf{E}: \mathbb{R}^m \to \Sigma_k$ be its deterministic TopK encoder. The encoder $\mathbf{E}$ selects the $k$ largest magnitude inner products $|\langle \mathbf{a}_j, \mathbf{x} \rangle|$ and returns their signed values, with ties broken lexicographically to ensure $\mathbf{E}$ is a deterministic function.

\paragraph{Round-Trip Property.} We assume that the encoder-decoder pair $(\mathbf{E}, \mathbf{A})$ satisfies the round-trip property if:
\begin{equation}
\label{eq:roundtrip}
\forall \mathbf{f} \in \Sigma_k, \quad \mathbf{E}(\mathbf{A}\mathbf{f}) = \mathbf{f}.
\end{equation}

\paragraph{$k$-Injectivity and Spark.} A dictionary $\mathbf{A}$ is \emph{$k$-injective} if $\forall \mathbf{f}, \mathbf{f}' \in \Sigma_k$, $\mathbf{A}\mathbf{f} = \mathbf{A}\mathbf{f}'$ implies $\mathbf{f} = \mathbf{f}'$. This is equivalent to the spark condition $\operatorname{spark}(\mathbf{A}) > 2k$, where $\operatorname{spark}(\mathbf{A})$ is the size of the smallest linearly dependent column set of $\mathbf{A}$~\citep{donoho2003optimally}:
\[
\text{$\mathbf{A}$ is $k$-injective} \quad \Longleftrightarrow \quad \operatorname{spark}(\mathbf{A}) > 2k.
\]

\subsection{Key Decomposition Lemma}
The following lemma provides the crucial technical tool for our main result:

\begin{lemma}[Two-Vector Decomposition]
\label{lem:two-vector}
Let $\mathbf{h} \in \mathbb{R}^d \setminus \{\mathbf{0}\}$ with $\|\mathbf{h}\|_0 \leq 2k$. There exist distinct vectors $\mathbf{f}, \mathbf{f}' \in \Sigma_k$ with disjoint supports such that $\mathbf{h} = \mathbf{f} - \mathbf{f}'$. Consequently, if $\mathbf{A}\mathbf{h} = \mathbf{0}$, then $\mathbf{A}\mathbf{f} = \mathbf{A}\mathbf{f}'$.
\end{lemma}

\begin{proof}
Let $S = \operatorname{supp}(\mathbf{h})$, so $|S| \leq 2k$. We can partition $S$ into two disjoint sets $S_1, S_2$ such that $|S_1|, |S_2| \leq k$. This is always possible since $|S| \leq 2k$.

Define vectors $\mathbf{f}, \mathbf{f}' \in \mathbb{R}^d$ by:
\[
\mathbf{f}_j := \begin{cases}
\mathbf{h}_j & \text{if } j \in S_1 \\
0 & \text{if } j \notin S_1
\end{cases}, \quad
\mathbf{f}'_j := \begin{cases}
-\mathbf{h}_j & \text{if } j \in S_2 \\
0 & \text{if } j \notin S_2
\end{cases}.
\]

By construction:
\begin{enumerate}
    \item $\mathbf{f}, \mathbf{f}' \in \Sigma_k$ since $\operatorname{supp}(\mathbf{f}) \subseteq S_1$ and $\operatorname{supp}(\mathbf{f}') \subseteq S_2$ with $|S_1|, |S_2| \leq k$
    \item $\operatorname{supp}(\mathbf{f}) \cap \operatorname{supp}(\mathbf{f}') = S_1 \cap S_2 = \emptyset$ (disjoint supports)
    \item $\mathbf{f} \neq \mathbf{f}'$ since $\mathbf{h} \neq \mathbf{0}$ implies at least one of $S_1, S_2$ is non-empty
    \item $\mathbf{h} = \mathbf{f} - \mathbf{f}'$ by direct verification on each coordinate
\end{enumerate}

If $\mathbf{A}\mathbf{h} = \mathbf{0}$, then $\mathbf{A}(\mathbf{f} - \mathbf{f}') = \mathbf{0}$, which immediately gives $\mathbf{A}\mathbf{f} = \mathbf{A}\mathbf{f}'$.
\end{proof}

\subsection{Main Result}

\begin{theorem}[Round-Trip Implies $k$-Injectivity]
\label{thm:roundtrip-spark}
If the round-trip property \eqref{eq:roundtrip} holds, then $\operatorname{spark}(\mathbf{A}) > 2k$. Equivalently, $\mathbf{A}$ is $k$-injective.
\end{theorem}

\begin{proof}
We proceed by contradiction. Assume $\operatorname{spark}(\mathbf{A}) \leq 2k$. Then there exists a non-zero vector $\mathbf{h} \in \mathbb{R}^d$ with $\|\mathbf{h}\|_0 \leq 2k$ such that $\mathbf{A}\mathbf{h} = \mathbf{0}$.

By Lemma~\ref{lem:two-vector}, we can decompose $\mathbf{h} = \mathbf{f} - \mathbf{f}'$ where $\mathbf{f}, \mathbf{f}' \in \Sigma_k$ are distinct vectors with disjoint supports, and $\mathbf{A}\mathbf{f} = \mathbf{A}\mathbf{f}'$ (since $\mathbf{A}\mathbf{h} = \mathbf{0}$).

Let $\mathbf{x} := \mathbf{A}\mathbf{f} = \mathbf{A}\mathbf{f}'$ denote the common image. Since $\mathbf{E}$ is a deterministic function, $\mathbf{E}(\mathbf{x})$ is uniquely determined. However, applying the round-trip property \eqref{eq:roundtrip} to both representations:
\[
\mathbf{E}(\mathbf{x}) = \mathbf{E}(\mathbf{A}\mathbf{f}) = \mathbf{f} \quad \text{and} \quad \mathbf{E}(\mathbf{x}) = \mathbf{E}(\mathbf{A}\mathbf{f}') = \mathbf{f}'.
\]

This implies $\mathbf{f} = \mathbf{f}'$, contradicting the fact that $\mathbf{f}$ and $\mathbf{f}'$ are distinct.

Therefore, our assumption $\operatorname{spark}(\mathbf{A}) \leq 2k$ must be false, which means $\operatorname{spark}(\mathbf{A}) > 2k$.
\end{proof}

\subsection{Application to TopK SAE}

\begin{corollary}[Spark Condition for TopK SAEs]
\label{cor:gt-to-learned-spark}
Let $X \subset \mathbb{R}^m$ be a training dataset and suppose a TopK SAE with learned dictionary $\mathbf{A} \in \mathbb{R}^{m \times d}$ and deterministic Top-$k$ encoder $\mathbf{E}: \mathbb{R}^m \to \Sigma_k$ achieves:
\begin{enumerate}
    \item \textbf{Zero reconstruction error}: $\mathbf{A}\mathbf{E}(\mathbf{x}) = \mathbf{x}$ for all $\mathbf{x} \in X$
    \item \textbf{Reachability}: For every $k$-sparse vector $\mathbf{f} \in \Sigma_k$, there exists $\mathbf{x} \in X$ such that $\mathbf{E}(\mathbf{x}) = \mathbf{f}$
\end{enumerate}
Then the learned dictionary $\mathbf{A}$ is $k$-injective: $\operatorname{spark}(\mathbf{A}) > 2k$.
\end{corollary}

\begin{proof}
Let $\mathbf{f} \in \Sigma_k$ be arbitrary. By reachability, there exists $\mathbf{x} \in X$ such that $\mathbf{E}(\mathbf{x}) = \mathbf{f}$.

Zero reconstruction error gives:
\begin{align}
\mathbf{A}\mathbf{f} = \mathbf{A}\mathbf{E}(\mathbf{x}) = \mathbf{x}.
\end{align}

Applying the encoder to both sides:
\begin{align}
\mathbf{E}(\mathbf{A}\mathbf{f}) = \mathbf{E}(\mathbf{x}) = \mathbf{f}.
\end{align}

Since this holds for arbitrary $\mathbf{f} \in \Sigma_k$, the round-trip property is satisfied. Theorem~\ref{thm:roundtrip-spark} then immediately implies $\operatorname{spark}(\mathbf{A}) > 2k$.
\end{proof}

\subsection{Implications for Feature Consistency}

\paragraph{Theoretical Guarantee.} Corollary~\ref{cor:gt-to-learned-spark} establishes that when TopK SAEs achieve zero reconstruction error and reachability on their training data, the learned dictionary satisfies the spark condition. This provides a theoretical foundation for feature consistency in TopK SAEs based purely on operational training outcomes.

\paragraph{Practical Interpretation.} The two conditions serve complementary but distinct roles in ensuring the spark guarantee. Zero reconstruction error prevents encoder collapse by forcing the encoder to distinguish among all training inputs---if multiple inputs collapsed to the same sparse code, some could not be perfectly reconstructed by the decoder. Reachability ensures comprehensive coverage by guaranteeing that every possible $k$-sparse code appears in the training dataset $X$. This coverage requirement enables the theoretical result to apply universally across all codes in $\Sigma_k$. In practice, exact reachability cannot be verified on finite datasets, so this condition is approximated through diverse training data that provides broad coverage across the feature space.

\paragraph{Connection to Identifiability.} 
Our result shows that learning a decoding dictionary satisfying the spark condition does not require access to the ground truth, which significantly broadens the scope of traditional identifiability discussions. When the ground truth $\mathbf{A}_{\text{gt}}$ is assumed to satisfy the spark condition, our result naturally recovers standard identifiability guarantees. More importantly, even in the absence of any ground truth information, our result ensures strong consistency across all learned dictionaries. This is especially valuable for practitioners who cannot reliably make assumptions about the data generation process.

\section{Supplementary Analysis of SAEs trained on Synthetic Data}
\label{app:supplementary}

\subsection{Detailed Analysis of Learning Regimes}

\label{app:learning_regimes}

This section provides additional analysis of the redundant and compressive learning regimes introduced in the main text. In all experiments, we maintain a constant TopK sparsity parameter $k=8, n=5e4$.

\subsubsection{Redundant Regime Analysis}
\label{app:redundant_regime}

In the redundant regime, we set the ground truth dimension $d_{gt}=80$ and the SAE dictionary size $d_{sae}=160$, creating a setting where the SAE has twice the capacity needed to represent all ground truth features ($d_{sae} > d_{gt}$). 

\begin{figure}[tb]
    \centering
    \begin{minipage}[t]{0.475\textwidth}
        \centering
        \includegraphics[width=\textwidth]{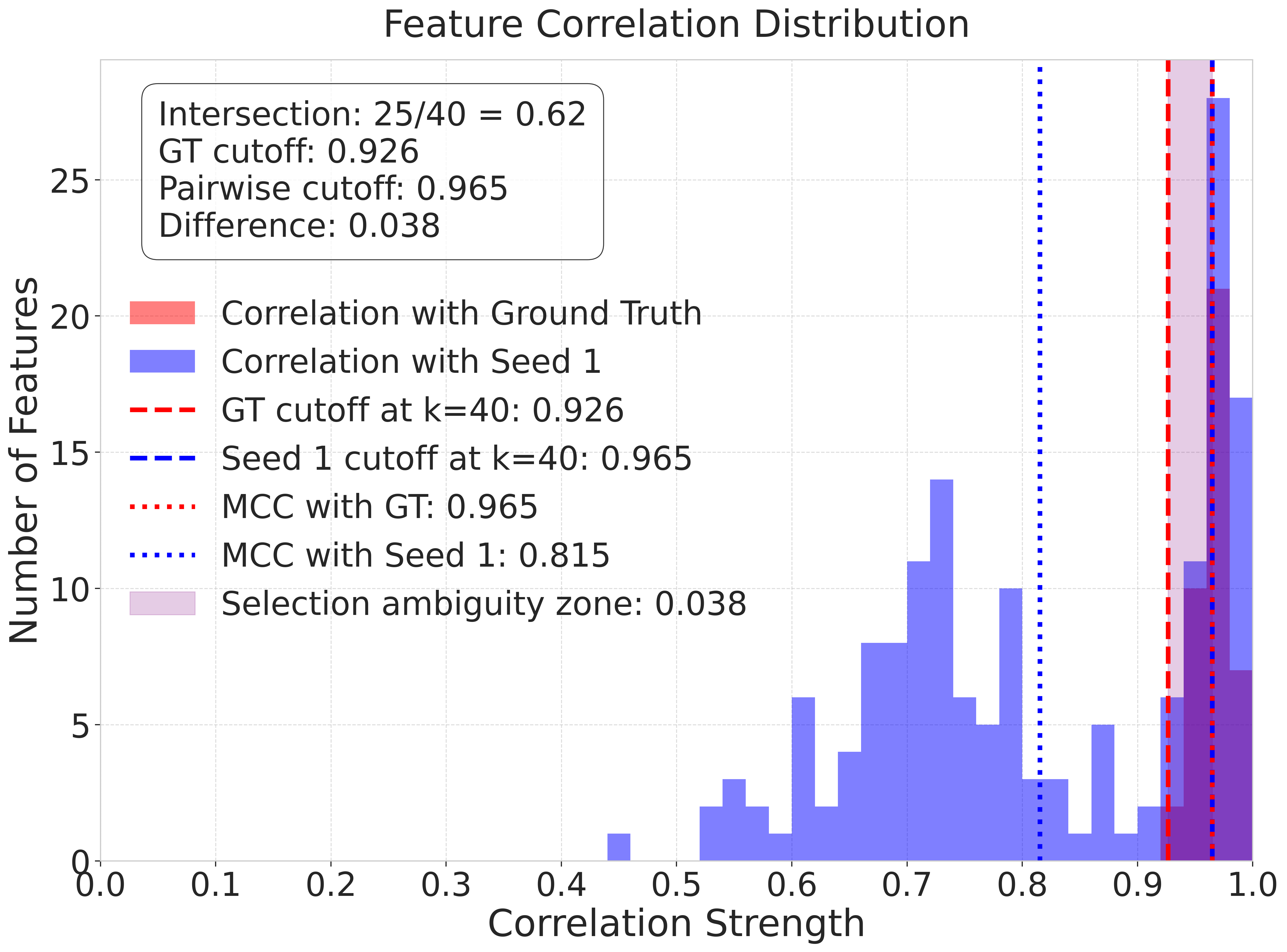}
        \caption{\small Feature Correlation Distribution ($d_{gt}=40, d_{sae}=160, k=8$). Compares similarities of Run 0 features to ground truth (red) and Run 1 features (blue). The substantial overlap in high-similarity regions (purple) demonstrates ambiguity where multiple SAE features are good matches to both ground truth and features learned in other runs, creating selection ambiguity despite high feature quality.}
        \label{fig:sparsedecomp} 
    \end{minipage}%
    \hfill 
    \begin{minipage}[t]{0.475\textwidth}
        \centering
        \includegraphics[width=\textwidth]{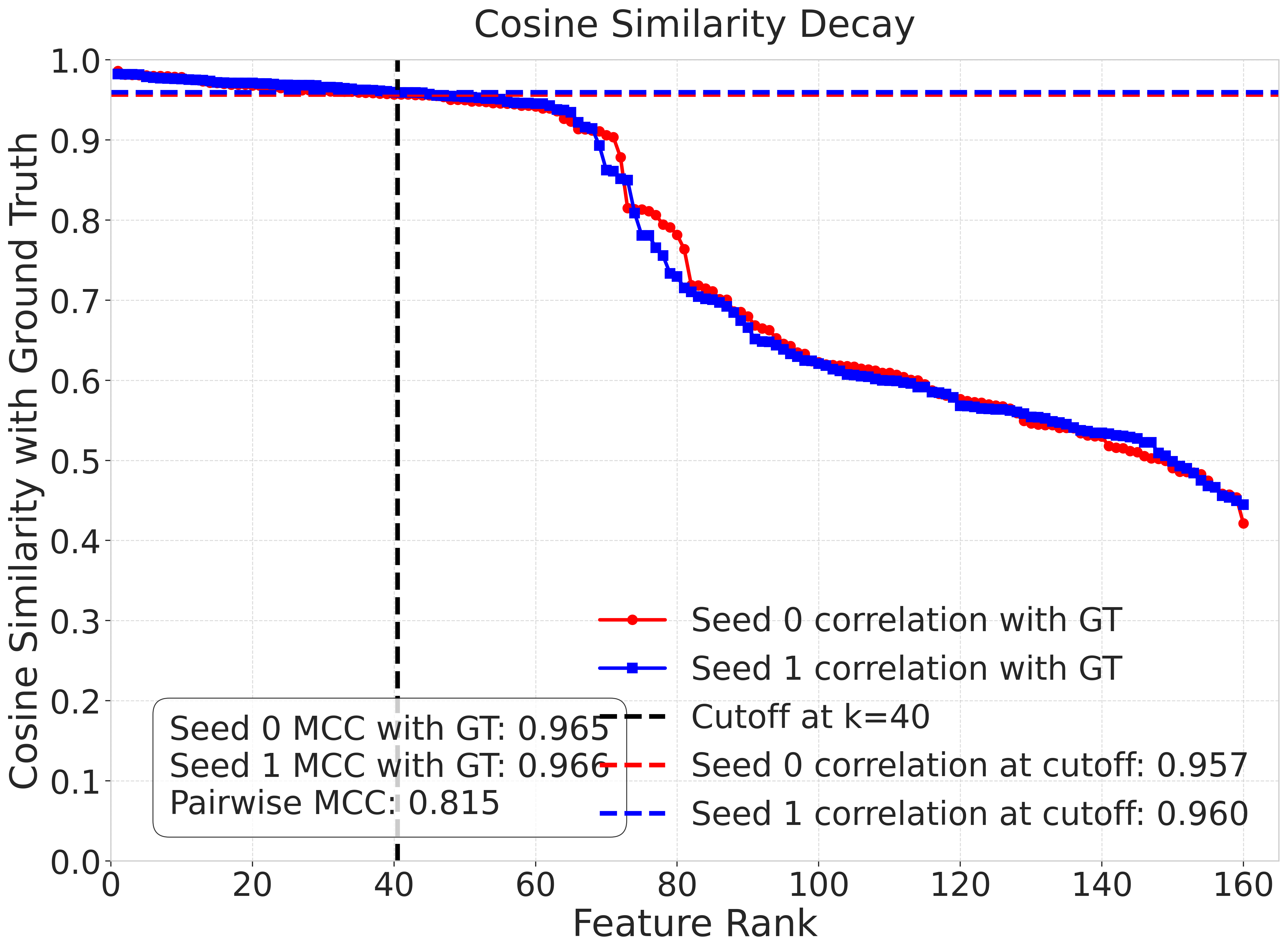}
        \caption{\small Cosine Similarity Decay with Ground Truth ($d_{gt}=40, d_{sae}=160, k=8$). Features are ranked by ground truth similarity for Run 0 (red) and Run 1 (blue). Similarity decays very slowly, remaining high well past rank $d_{gt}=40$, indicating that the SAE learns multiple good representations for each ground truth feature, creating a selection challenge when comparing across runs.}
        \label{fig:rankdecomp} 
    \end{minipage}
\end{figure}

Figures \ref{fig:sparsedecomp} and \ref{fig:rankdecomp} illustrate a key characteristic of the redundant regime: the SAE learns multiple good representations for each ground truth feature. Figure \ref{fig:sparsedecomp} shows substantial overlap between features with high similarity to ground truth and features with high similarity across runs. Figure \ref{fig:rankdecomp} demonstrates that cosine similarity to ground truth decays very slowly, remaining high well beyond the ground truth dimension.

This redundancy creates a fundamental selection ambiguity problem when comparing features across different SAE initializations. The large pool of near-equally good candidates for the top-$d_{gt}$ matches makes the optimal feature matching determined by the Hungarian algorithm highly sensitive to small variations between runs. This dictionary instability persists despite high average MCC scores with the ground truth. In essence, while the SAE learns good representations of the underlying features (as evidenced by high GT-MCC), it lacks a consistent way to select among multiple valid alternatives, leading to different features being selected across runs and consequently comparatively lower pairwise consistency.

As shown in Figure \ref{fig:redundant_gt_mcc}, the Mean GT-MCC (Maximum Correlation Coefficient) reaches high values, indicating strong recovery of ground truth features. However, Figure \ref{fig:redundant_pairwise_mcc} shows that the PW-MCC across different SAE initializations is lower, reflecting the challenge of consistent feature selection despite good ground truth recovery.

\begin{figure}[htbp]
    \centering
    \begin{minipage}{0.48\textwidth}
        \centering
        \includegraphics[width=\textwidth]{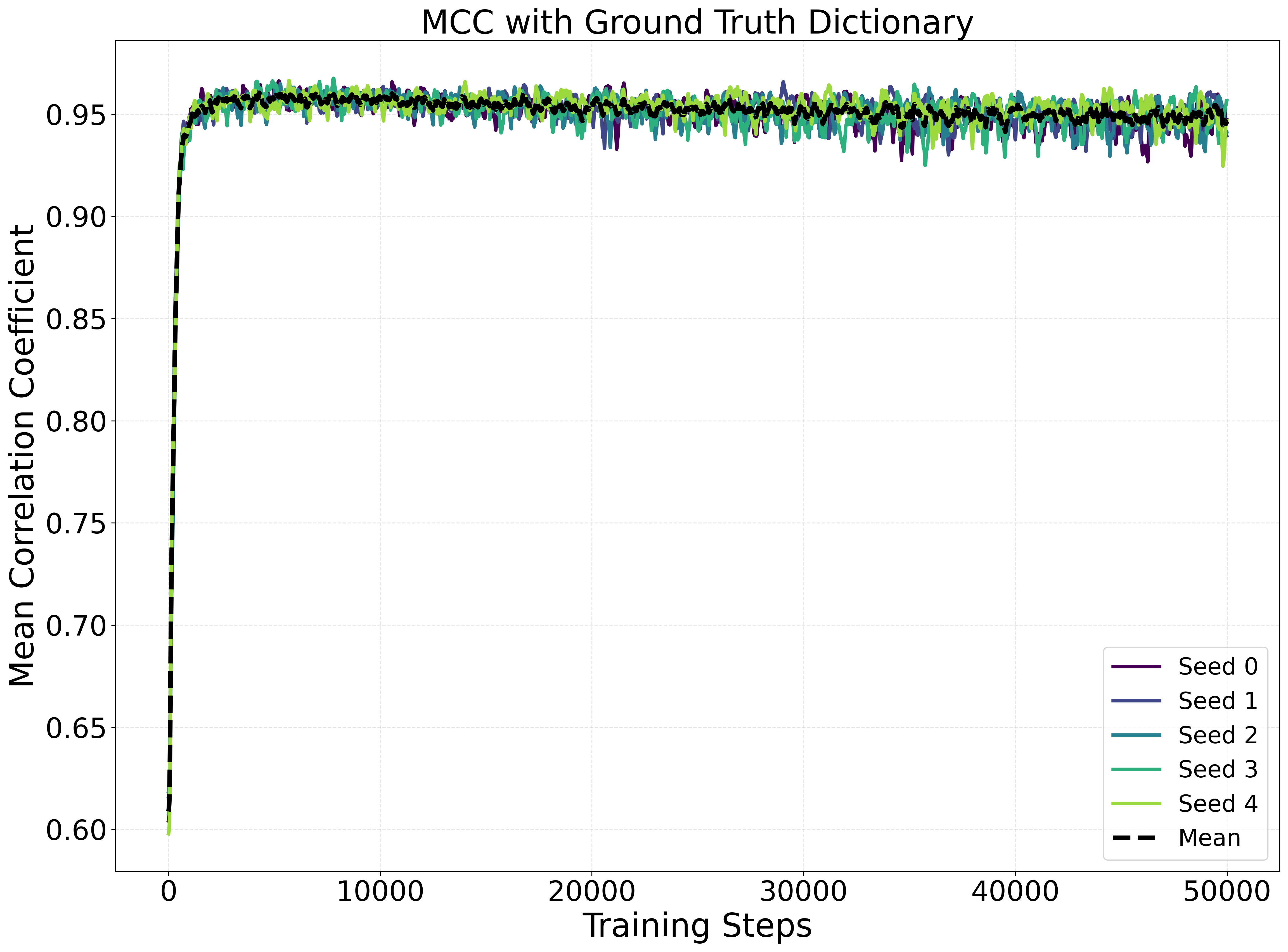}
        \caption{Redundant Regime: TopK SAE Mean GT-MCC (across 5 seeds) vs. Training Steps ($d_{gt}=80, d_{sae}=160, k=8$). The GT-MCC reaches high values, indicating strong recovery of ground truth features.}
        \label{fig:redundant_gt_mcc}
    \end{minipage}
    \hfill
    \begin{minipage}{0.48\textwidth}
        \centering
        \includegraphics[width=\textwidth]{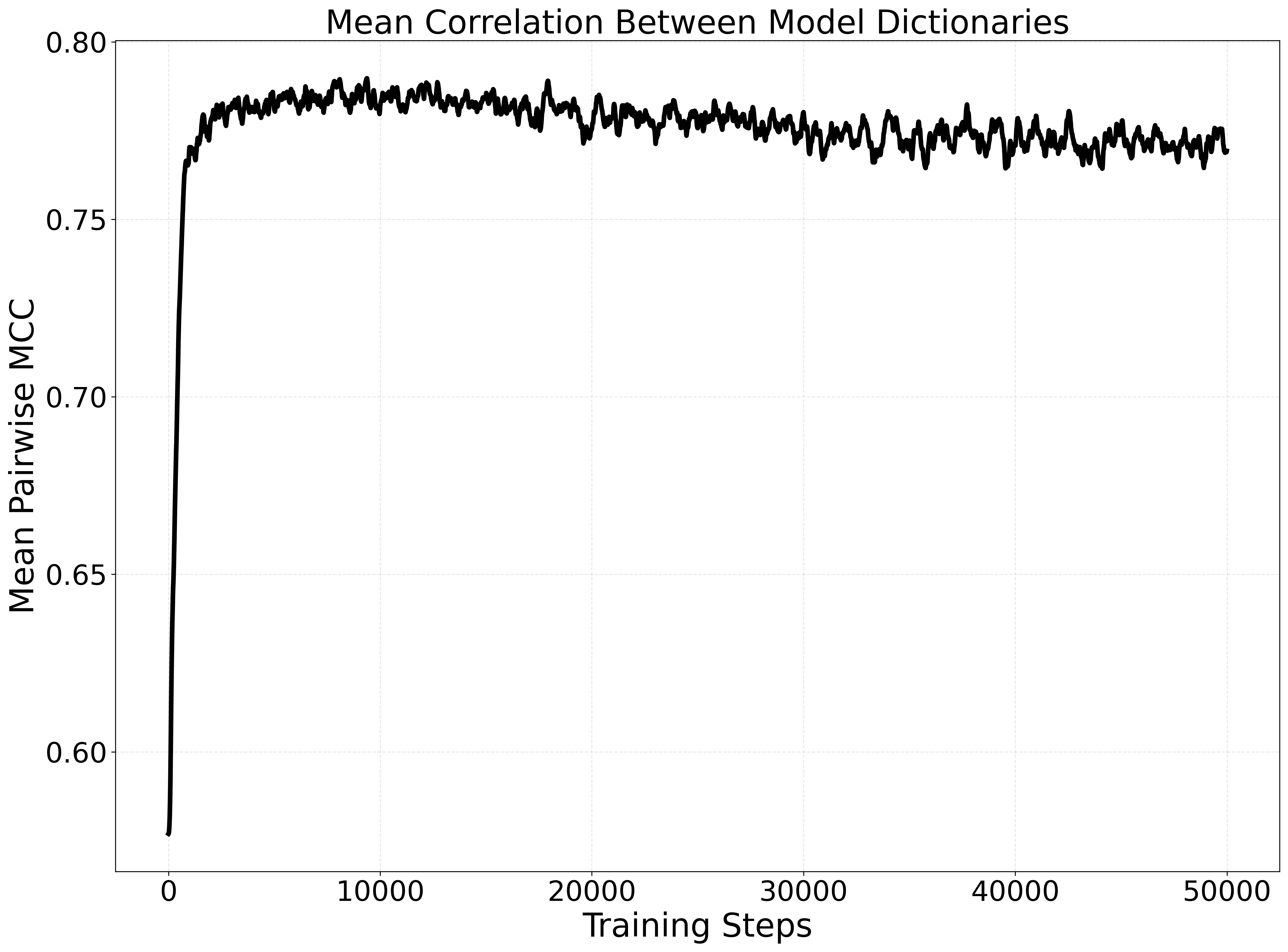}
        \caption{Redundant Regime: TopK SAE Mean PW-MCC (across 5 seeds) vs. Training Steps ($d_{gt}=80, d_{sae}=160, k=8$). The PW-MCC reaches lower values than GT-MCC, reflecting the challenge of feature consistency across different SAE initializations due to selection ambiguity.}
        \label{fig:redundant_pairwise_mcc}
    \end{minipage}
\end{figure}

\paragraph{Intersection Ratio} The Intersection Ratio measures the consistency of feature selection across different training runs by quantifying how often the same learned features that match well to ground truth also match well between different runs. For a pair of runs (run$_1$, run$_2$), we first find $M_{1 \rightarrow GT}$, the optimal matching between dictionaries $A_1$ and $A_{gt}$, and define $I_{1 \rightarrow GT} = \{i \mid \exists j, (i, j) \in M_{1 \rightarrow GT}\}$ as the set of feature indices from Run 1 that successfully match to ground truth features. Next, we find $M_{1 \rightarrow 2}$, the optimal matching between dictionaries $A_1$ and $A_2$, and let $I'_{1 \rightarrow 2}$ be the set of the top $d_{gt}$ feature indices from Run 1 that participate in the highest-scoring similarity pairs $(i,k) \in M_{1 \rightarrow 2}$ (if $d_{sae} < d_{gt}$, we use all $d_{sae}$ indices). The intersection ratio is then computed as $R_{1,2} = \frac{|I_{1 \rightarrow GT} \cap I'_{1 \rightarrow 2}|}{\min(d_{gt}, d_{sae})}$. We report $\mathbb{E}_{i \neq j}[R_{i,j}]$, the expected intersection ratio estimated by averaging over multiple distinct pairs of runs, where higher values indicate reduced selection ambiguity and more consistent feature discovery across training runs.

\begin{figure}[htbp]
\centering
\includegraphics[width=0.48\textwidth]{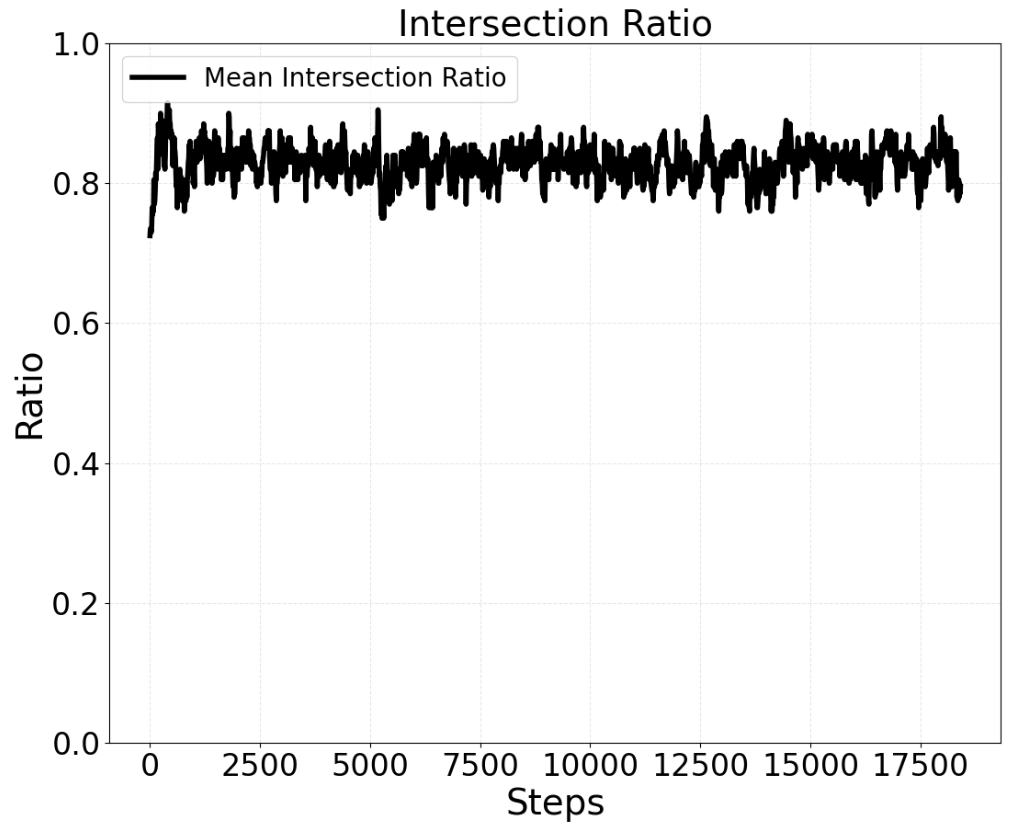}
\caption{Redundant Regime: TopK SAE Mean Intersection Ratio (across 5 seeds) vs. Training Steps ($d_{gt}=80, d_{sae}=160, k=8$). The Intersection Ratio measures the consistency of feature selection indices across different SAE initializations, with higher values indicating more stable feature recovery.}
\label{fig:redundant_intersection_ratio}
\end{figure}

We find that the Intersection Ratio increases over training steps, indicating that SAEs converge toward more consistent feature selection, though perfect consistency remains challenging due to the fundamental ambiguity introduced by excess capacity.

\subsubsection{Compressive Regime Analysis}
\label{app:compressive_regime}

In the compressive regime, we set the ground truth dimension $d_{gt}=800$ and the SAE dictionary size $d_{sae}=80$, creating a setting where the SAE has only one-tenth of the capacity needed to represent all ground truth features ($d_{sae} < d_{gt}$). This capacity limitation forces the SAE to prioritize which features to learn.

\begin{figure}[htbp]
    \centering
    \begin{minipage}{0.48\textwidth}
        \centering
        \includegraphics[width=\textwidth]{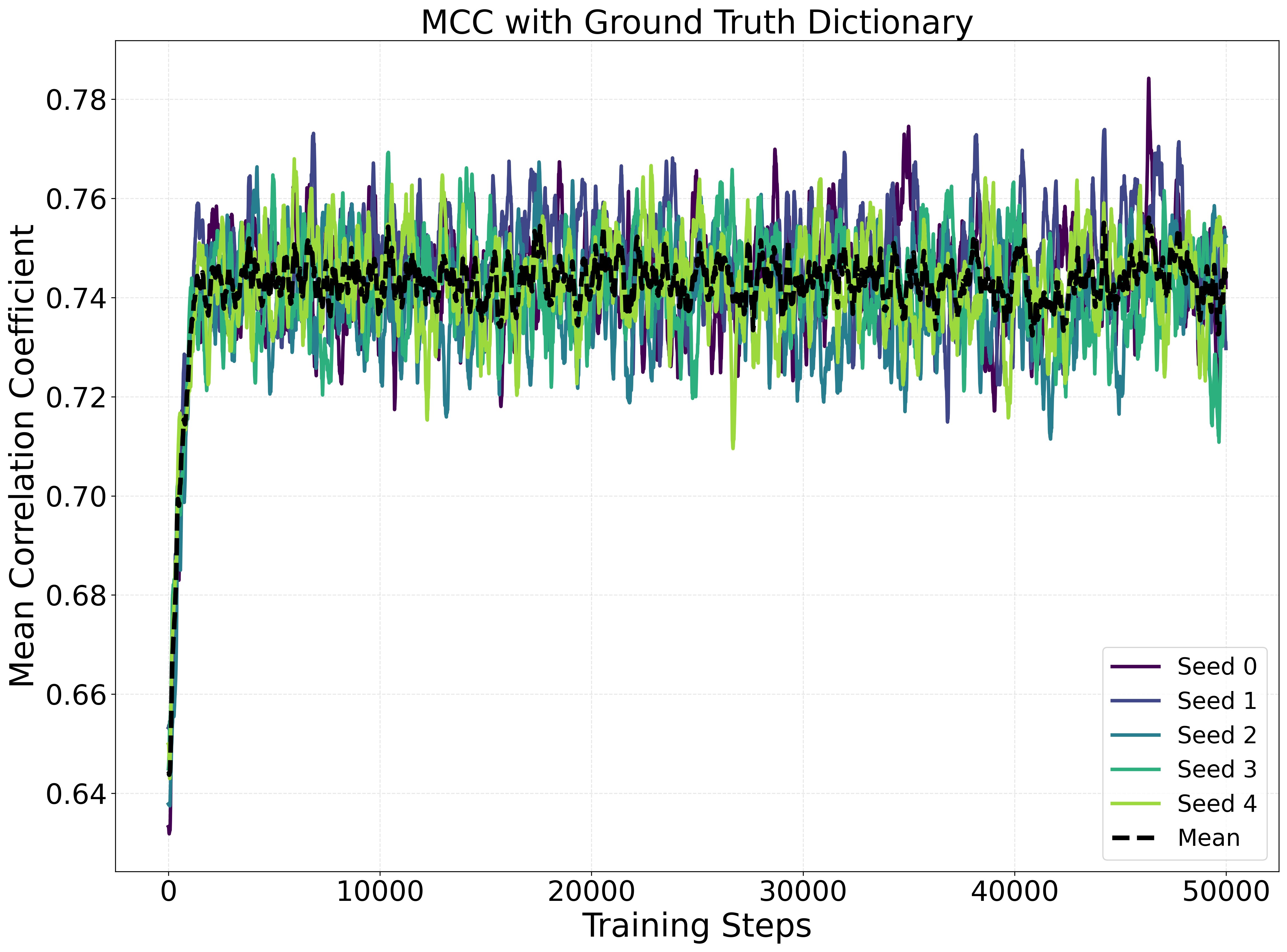}
        \caption{Compressive Regime: TopK SAE Mean GT-MCC (across 5 seeds) vs. Training Steps ($d_{gt}=800, d_{sae}=80, k=8$). The GT-MCC reaches lower values compared to the redundant regime, reflecting the fundamental capacity limitation that prevents complete recovery of all ground truth features.}
        \label{fig:compressive_gt_mcc}
    \end{minipage}
    \hfill
    \begin{minipage}{0.48\textwidth}
        \centering
        \includegraphics[width=\textwidth]{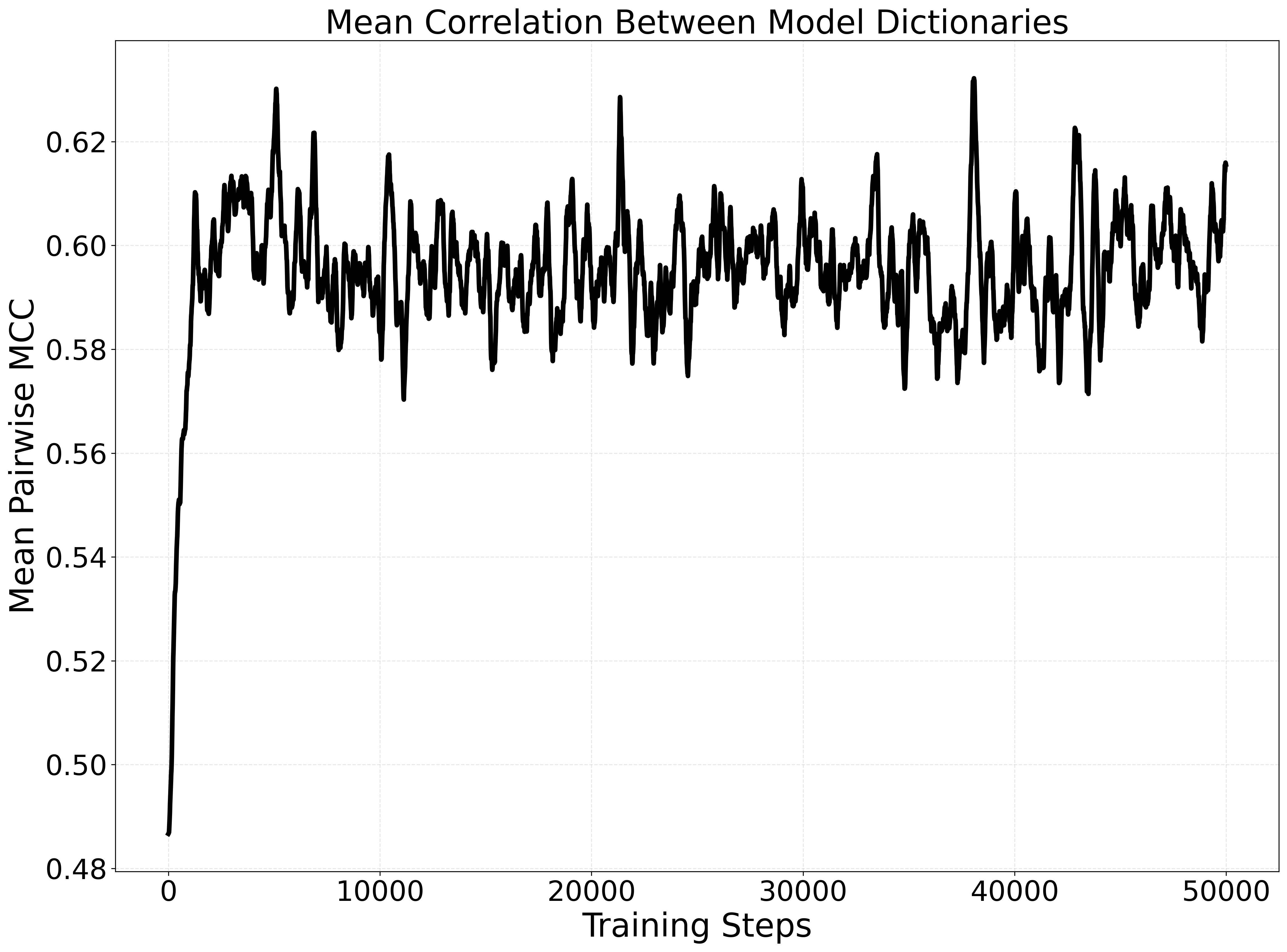}
        \caption{Compressive Regime: TopK SAE Mean PW-MCC (across 5 seeds) vs. Training Steps ($d_{gt}=800, d_{sae}=80, k=8$). The PW-MCC values are also lower compared to the redundant regime, indicating lower overall recovery quality.}
        \label{fig:compressive_pairwise_mcc}
    \end{minipage}
\end{figure}

Figures \ref{fig:compressive_gt_mcc} and \ref{fig:compressive_pairwise_mcc} illustrate the key characteristics of the compressive regime. Unlike the redundant regime, where feature selection ambiguity was the primary challenge, the compressive regime faces a fundamental capacity limitation.

Figure \ref{fig:compressive_gt_mcc} shows that both the PW-MCC and the Mean GT-MCC reach significantly lower values compared to the redundant regime (Figure \ref{fig:redundant_gt_mcc}), reflecting the inability to recover all ground truth features with limited capacity.

\subsection{Uniform Partitioning Experiments}
\label{app:uniform_partitioning}

We analyze how partitioning ground truth features into uniform clusters affects SAE learning dynamics. In these experiments, we maintain a constant total number of ground truth features ($d_{gt}=800$) while varying the number of clusters they are organized into. The dimension of each cluster is $d_{gt}/\text{num\_clusters}$, resulting in fewer features per cluster as the number of clusters increases. The complete hyperparameter settings for these experiments are presented in Table \ref{tab:uniform_params}.

\begin{table}[htb]
\small
    \centering
        \caption{Hyperparameters for Uniform Clustering Experiments}
    \label{tab:uniform_params}
    \begin{tabular}{lr}
    \toprule
    \textbf{Parameter} & \textbf{Value} \\
    \midrule
    TopK sparsity parameter ($k$) & 8 \\
    Activation dimension & 20 \\
    Dictionary size ($d_{sae}$) & 80 \\
    Training examples & 100,000 \\
    Training steps & 20,000 \\
    Learning rate & 0.04 \\
    Learning rate decay factor & 0.1 \\
    Learning rate decay steps & [20,000] \\
    Warmup steps & 1,000 \\
    Minimum learning rate & 1e-05 \\
    L1 coefficient & 0.1 \\
    Batch size & 4,096 \\
    Number of seeds& 3 \\
    Cluster distribution& uniform \\
    Ground truth dimension ($d_{gt}$) & 800 \\
    Number of clusters & varies (1, 10, 50, 100) \\
    Cluster dimensions & $d_{gt}/\text{num\_clusters}$ \\
    \bottomrule
    \end{tabular}
\end{table}

\begin{table}[htb]
    \small
    \centering

    \caption{Effect of uniform partitioning of ground truth vectors. As the number of clusters increases while keeping the total number of ground truth vectors constant, mean PW-MCC shows a weak but consistent increase (0.621 to 0.665), while ground truth MCC remains stable around 0.74.}
    \label{tab:cluster_mcc_effect}
    \begin{tabular}{crrr}
    \toprule
    \textbf{Clusters} & \textbf{Mean GT-MCC} & \textbf{Std GT-MCC} & \textbf{Mean PW-MCC} \\
    \midrule
    1 & 0.742 & 0.006 & 0.621 \\
    10 & 0.747 & 0.002 & 0.634 \\
    50 & 0.740 & 0.002 & 0.651 \\
    100 & 0.746 & 0.007 & 0.665 \\
    \bottomrule
    \end{tabular}

\end{table}

Table \ref{tab:cluster_mcc_effect} presents the results of our uniform partitioning experiments. The key finding is that as we impose additional structure by increasing the number of clusters from 1 to 100 while keeping the total number of ground truth vectors constant, the mean PW-MCC between SAEs increases consistently from 0.621 to 0.665. This suggests that clustered organization of features promotes more consistent feature learning across different SAE initializations. Interestingly, the mean ground truth MCC remains stable around 0.74 across all cluster configurations, indicating that the overall recovery quality of ground truth features is not significantly affected by the clustering structure.

\subsection{Feature Recovery Across Zipf Distributions}
\label{app:zipf_distributions}

This section provides additional experimental results showing how feature recoverability in SAEs varies across different Zipf distributions. We analyzed distributions with exponents $\alpha \in \{1.0, 1.1, 1.5, 2.0\}$ to understand how the skewness of ground truth feature cluster probability affects SAE learning dynamics and feature reproducibility. For all experiments in this section, we set the ground truth dimension $d_{gt}=800$, SAE dictionary size $d_{sae}=80$, and TopK sparsity parameter $k=8$, placing us in the compressive regime where the SAE must learn a compressed representation of the underlying features. The ground truth features are organized into 10 clusters, with 80 ground truth features per cluster, where the probability of each cluster appearing in the data follows a Zipf distribution with varying exponents $\alpha$. Table \ref{tab:zipf_params} provides the complete hyperparameter settings used in these experiments.

\begin{table}[htb]
    \small
    \centering
     \caption{Hyperparameters for Zipf Distribution Experiments}
    \label{tab:zipf_params}
    \begin{tabular}{lr}
    \toprule
    \textbf{Parameter} & \textbf{Value} \\
    \midrule
    TopK sparsity parameter ($k$) & 8 \\
    Activation dimension & 20 \\
    Dictionary size ($d_{sae}$) & 80 \\
    Training examples & 100,000 \\
    Training steps & 20,000 \\
    Learning rate & 0.04 \\
    Learning rate decay factor & 0.1 \\
    Learning rate decay steps & [20,000] \\
    Warmup steps & 1,000 \\
    Minimum learning rate & 1e-05 \\
    L1 coefficient & 0.1 \\
    Batch size & 4,096 \\
    Number of models trained (seeds)& 3 \\
    Number of clusters & 10 \\
    Cluster dimensions & 80 per cluster \\
    Distribution & zipf \\
    Zipf skew ($\alpha$) & varies (1.0, 1.1, 1.5, 2.0) \\
    \bottomrule
    \end{tabular}
\end{table}

\subsubsection{Zipf Distribution with $\alpha=1.0$}
\label{app:zipf_alpha1}

\begin{figure}[htbp]
    \centering
    \begin{minipage}{0.48\textwidth}
        \centering
        \includegraphics[width=\textwidth]{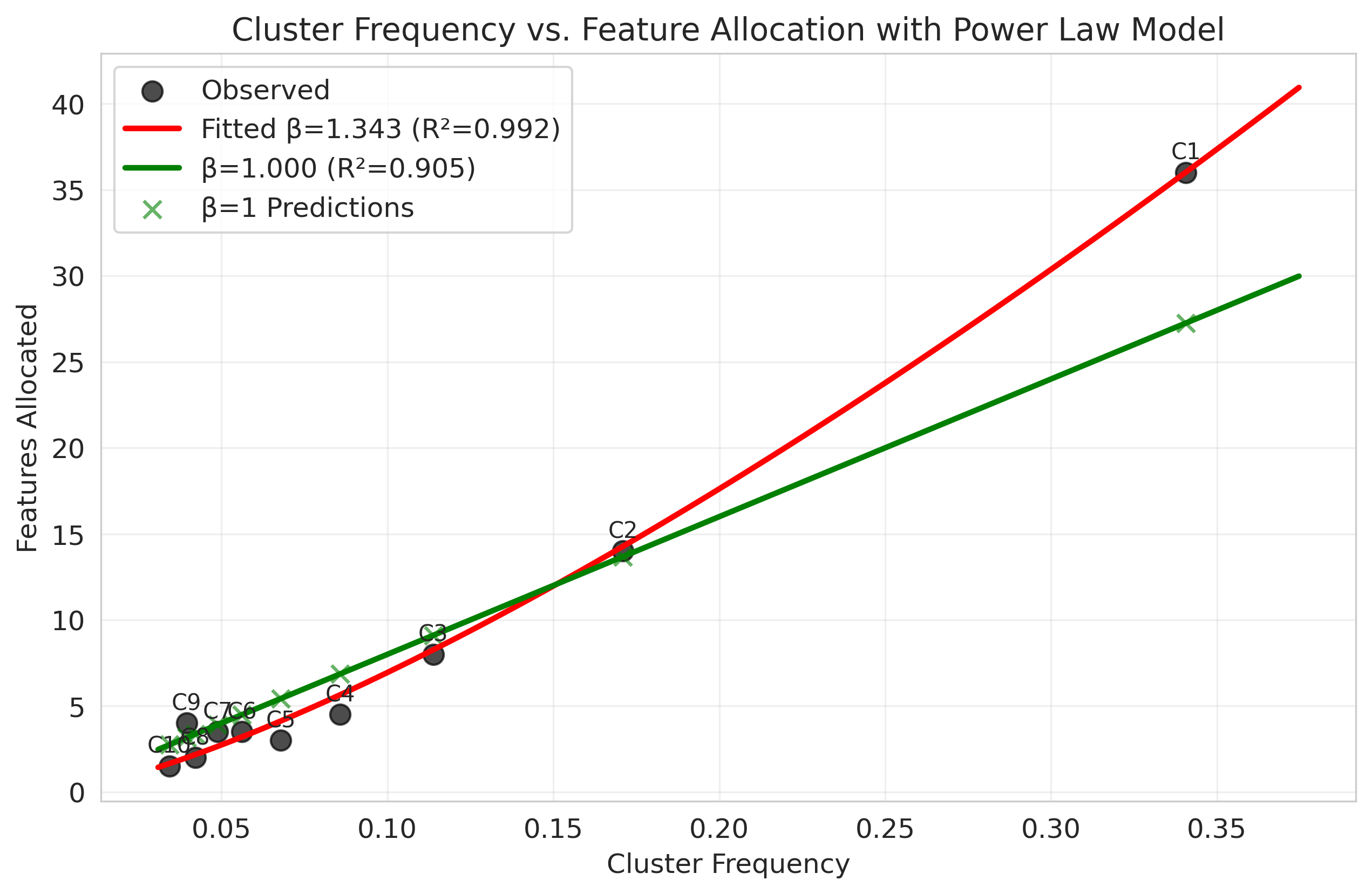}
    \end{minipage}
    \hfill
    \begin{minipage}{0.48\textwidth}
        \centering
        \includegraphics[width=\textwidth]{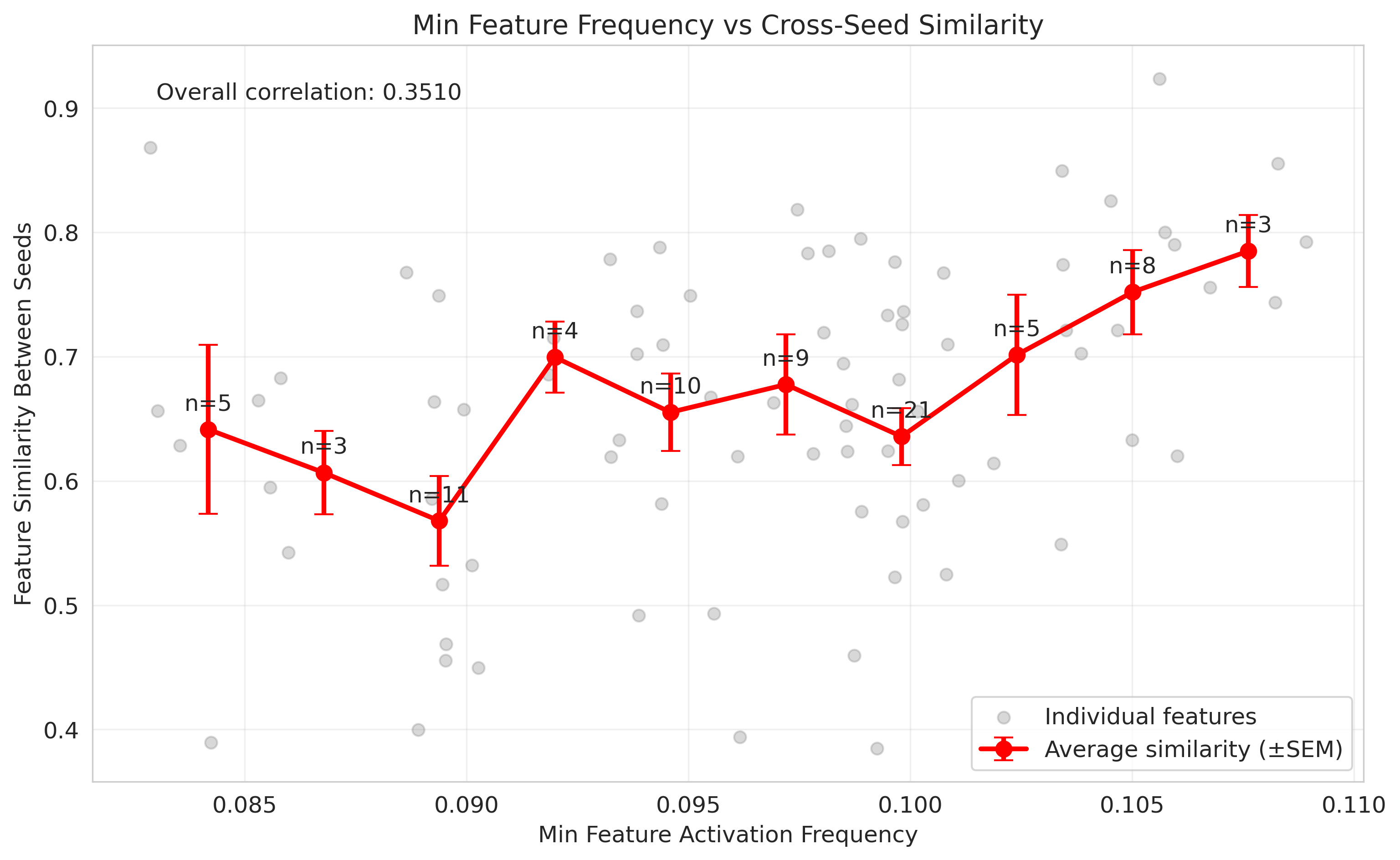}
    \end{minipage}
    \caption{Left: Capacity allocation model for Zipf distribution with $\alpha=1.0$, showing how SAE features are allocated to clusters based on cluster probability. The red curve shows the fitted power law model, following $D_i \propto p_i^{\beta}$ where $\beta \approx 1.343$. Right: Feature similarity between independently trained SAEs as a function of minimum feature activation frequency, with bucketed averages (red) showing a positive trend between activation frequency and feature reproducibility.}
    \label{fig:capacity_similarity_alpha1}
\end{figure}

\begin{figure}[htbp]
    \centering
    \includegraphics[width=0.9\textwidth]{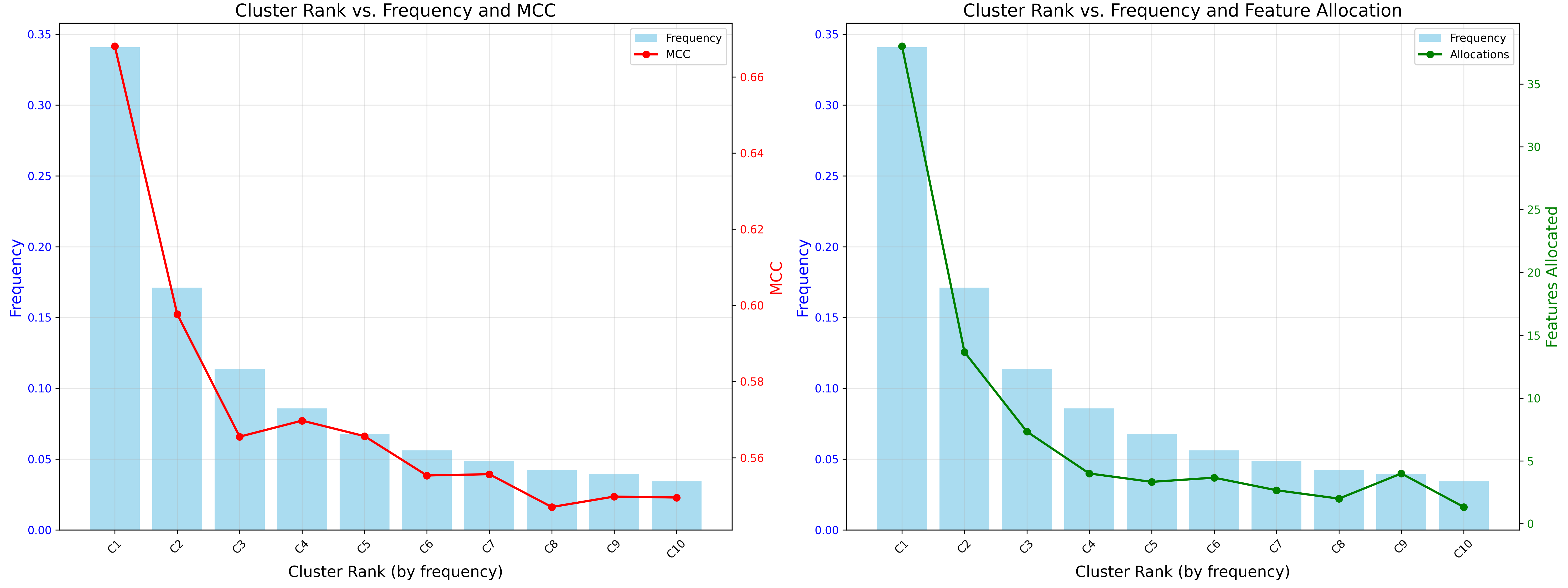}
    \caption{Cluster metrics for Zipf distribution with $\alpha=1.0$. Left: Cluster rank vs. probability (blue bars) and MCC scores (red line), showing how feature recovery quality varies with cluster probability. The MCC scores demonstrate a positive correlation with cluster probability, with lower-ranked (more probable) clusters achieving better feature recovery. Right: Cluster rank vs. probability (blue bars) and feature allocation (green line), demonstrating how the model allocates dictionary features based on cluster probability, with more frequent clusters receiving proportionally more features.}
    \label{fig:cluster_metrics_alpha1}
\end{figure}

\begin{figure}[htbp]
    \centering
    \begin{minipage}{0.48\textwidth}
        \centering
        \includegraphics[width=\textwidth]{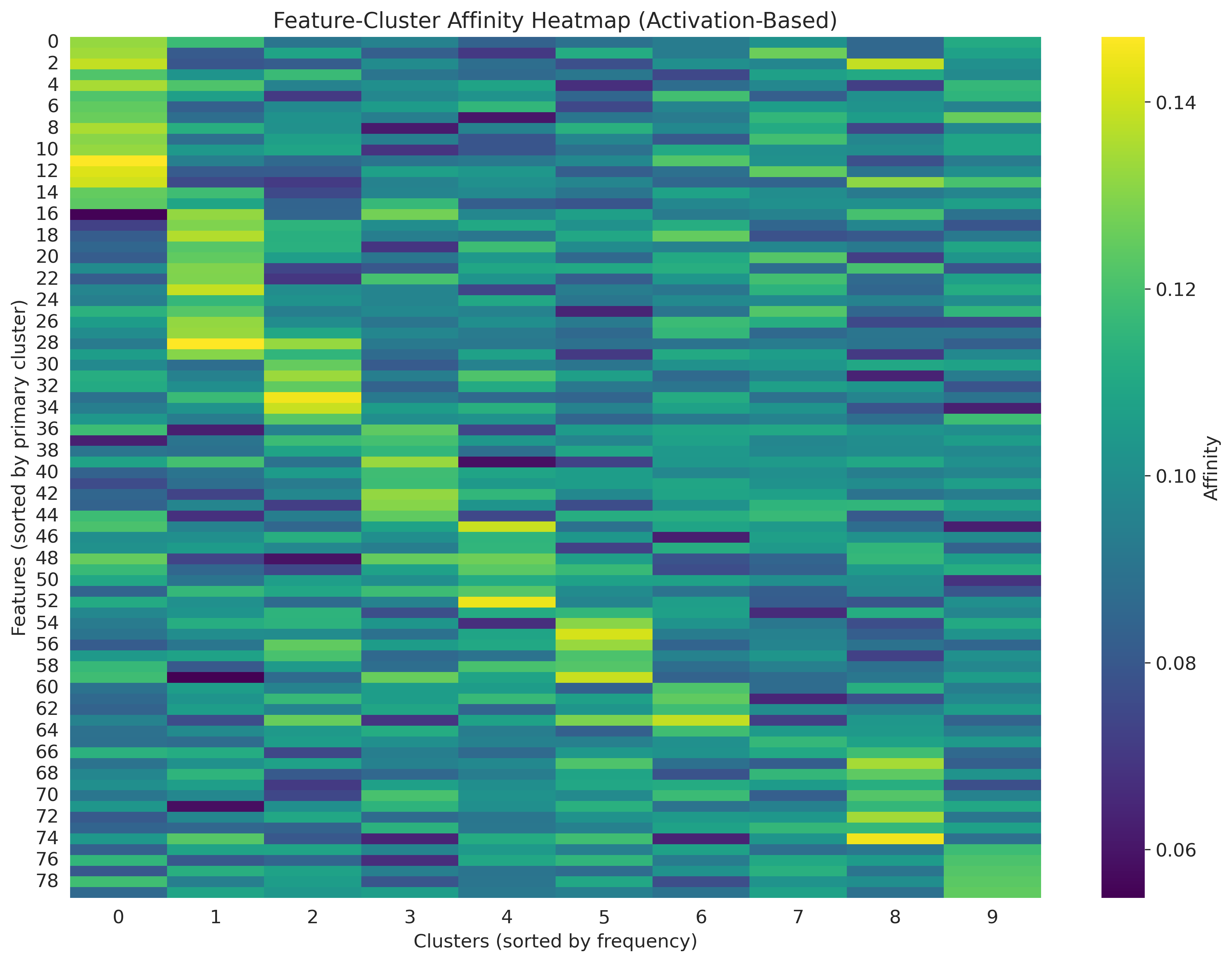}
    \end{minipage}
    \hfill
    \begin{minipage}{0.48\textwidth}
        \centering
        \includegraphics[width=\textwidth]{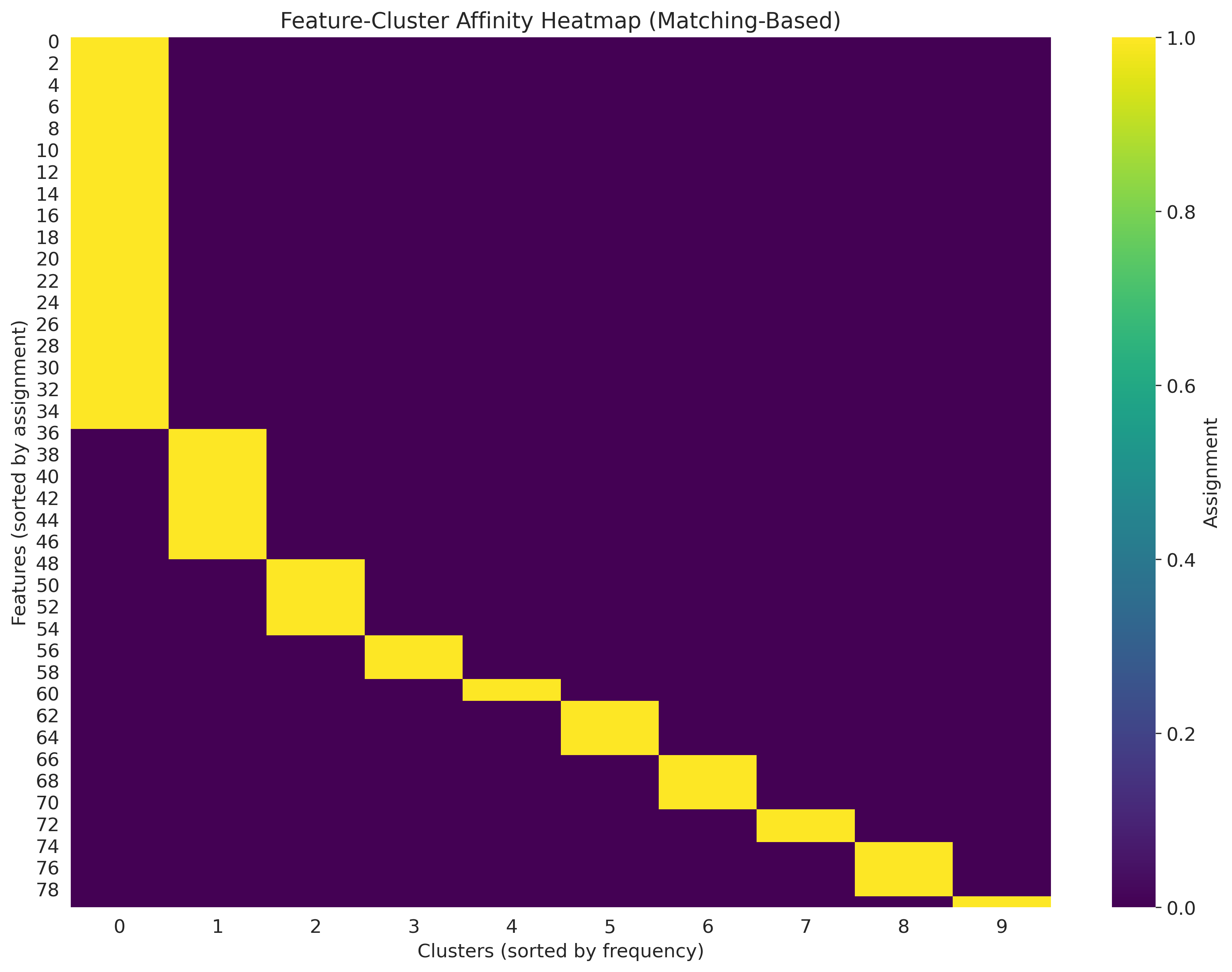}
    \end{minipage}
    \caption{Feature-cluster relationships for Zipf distribution with $\alpha=1.0$. Left: Activation-based affinity heatmap showing how features (y-axis, sorted by primary cluster) are activated by different clusters (x-axis, sorted by probability). Brighter colors indicate stronger activation, showing that more frequent cluster features activate more features. Right: Matching-based affinity heatmap showing global assignment of features to clusters using Hungarian matching, with features on y-axis and clusters on x-axis}
    \label{fig:affinity_alpha1}
\end{figure}

For $\alpha=1.0$, we observe a moderate skew in cluster probabilities with a corresponding power-law allocation of dictionary features. As shown in Figure \ref{fig:capacity_similarity_alpha1} (left), the SAE capacity allocation via Hungarian matching follows $D_i \propto p_i^{\beta}$ with $\beta \approx 1.343$, where $D_i$ is the number of SAE features allocated to cluster $i$ and $p_i$ is the cluster's probability in the data distribution. This superlinear relationship indicates that more probable clusters receive disproportionately more dictionary features. Figure \ref{fig:capacity_similarity_alpha1} (right) demonstrates that features with higher activation frequencies also show greater reproducibility across different SAE initializations, indicating that frequently activated features are more robustly learned.

The cluster metrics in Figure \ref{fig:cluster_metrics_alpha1} further support this relationship, showing that more probable clusters achieve better feature recovery quality as measured by MCC scores. The feature-cluster activation-based affinity map in Figure \ref{fig:affinity_alpha1} reveal that while features tend to specialize for specific clusters, there is some activation overlap, particularly among the most probable clusters. The activation-based affinity (left) displays more diffuse relationships between features and clusters and exhibits less frequency skew compared to the discrete one-to-one assignments established through Hungarian matching (right), which more strongly favors high-probability clusters.

\subsubsection{Zipf Distribution with $\alpha=1.1$}
\label{app:zipf_alpha1_1}

\begin{figure}[htbp]
    \centering
    \begin{minipage}{0.48\textwidth}
        \centering
        \includegraphics[width=\textwidth]{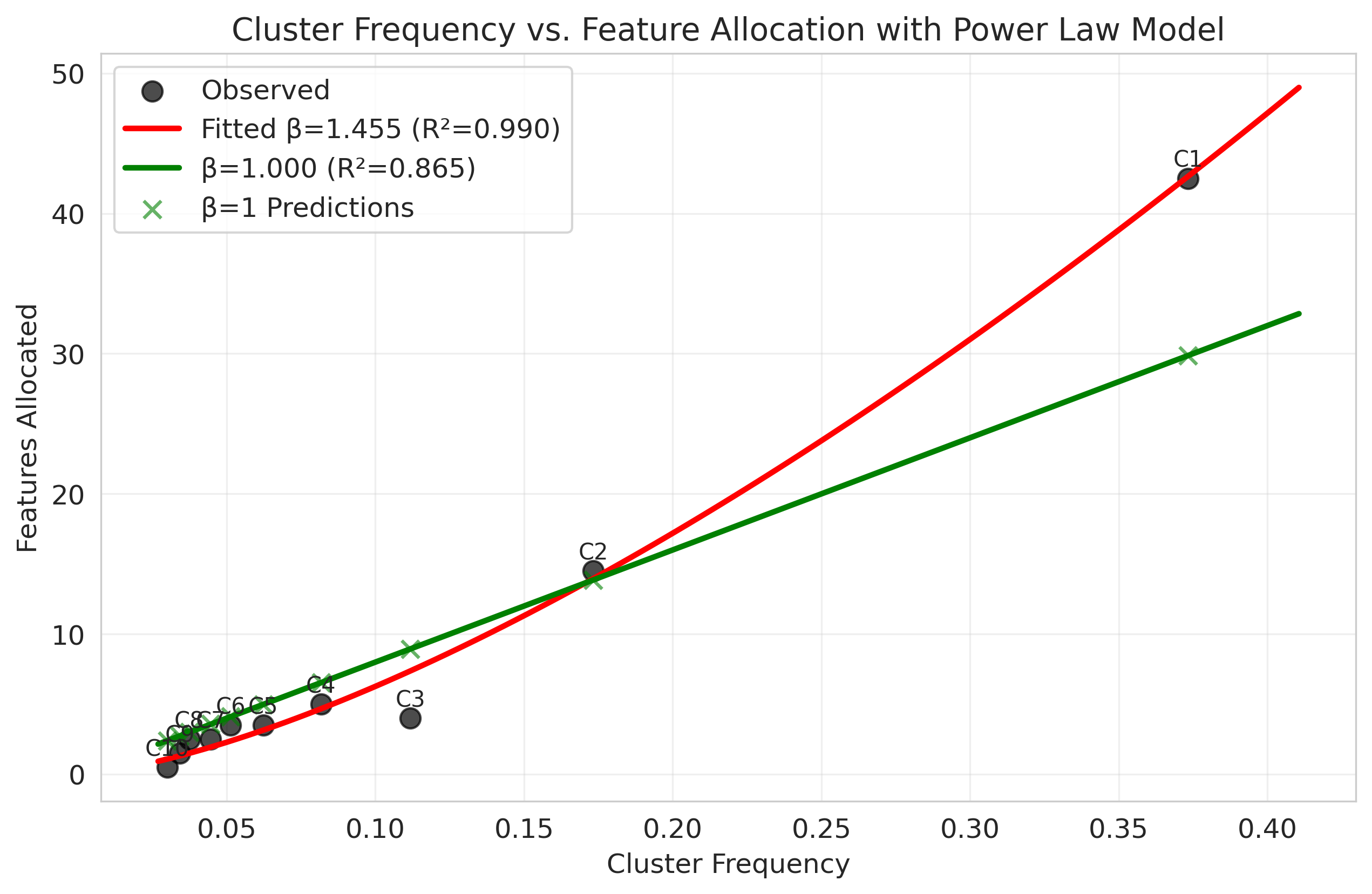}
    \end{minipage}
    \hfill
    \begin{minipage}{0.48\textwidth}
        \centering
        \includegraphics[width=\textwidth]{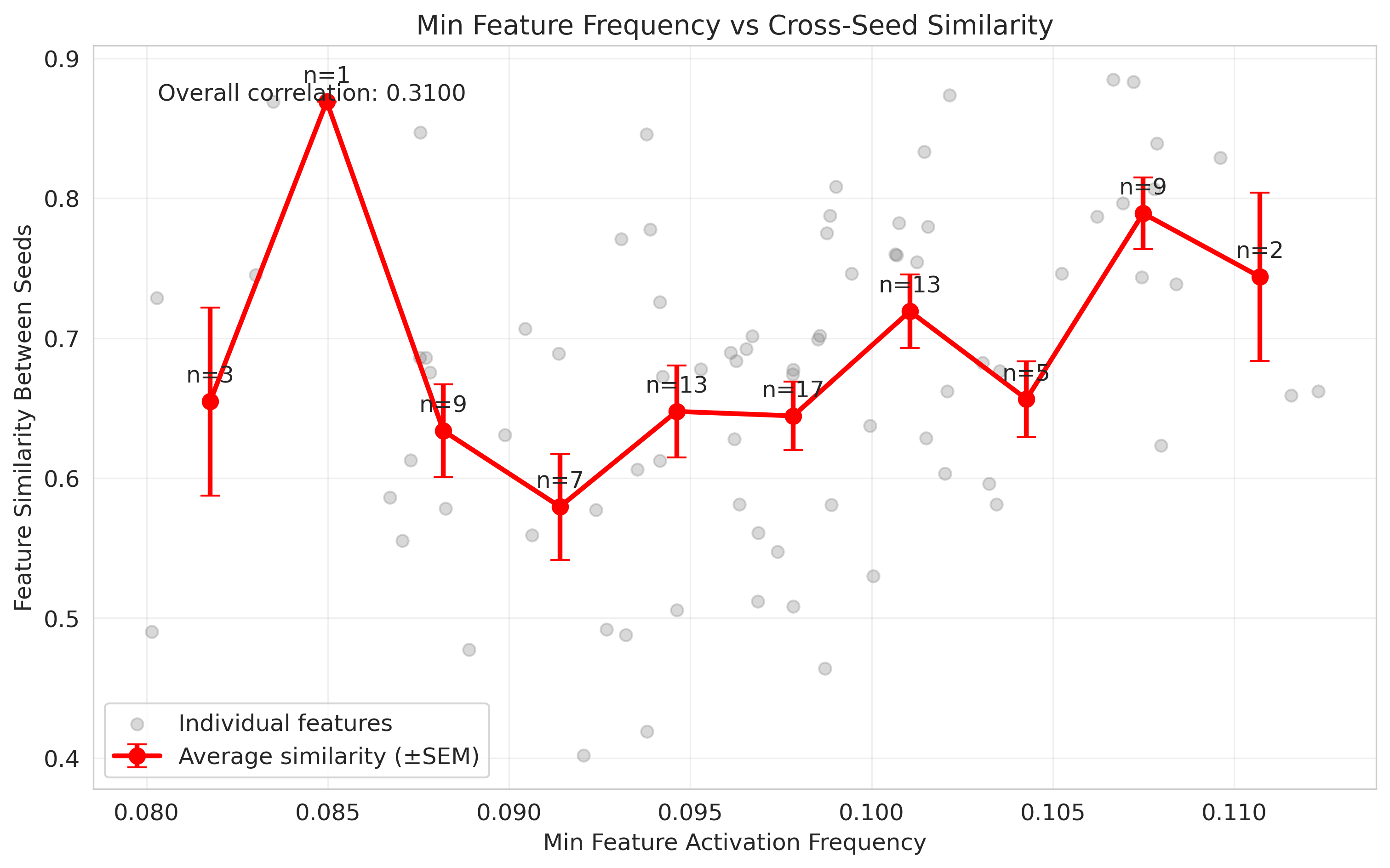}
    \end{minipage}
    \caption{Left: Capacity allocation model for Zipf distribution with $\alpha=1.1$, showing how SAE features are allocated to clusters based on cluster probability. Red curve shows fitted power law model with $D_i \propto p_i^{\beta}$ where $\beta \approx 1.455$. Right: Feature similarity between independently trained SAEs as a function of minimum feature activation frequency, with bucketed averages (red) showing a positive trend between activation frequency and feature reproducibility.}
    \label{fig:capacity_similarity_alpha1_1}
\end{figure}

\begin{figure}[htbp]
    \centering
    \includegraphics[width=0.9\textwidth]{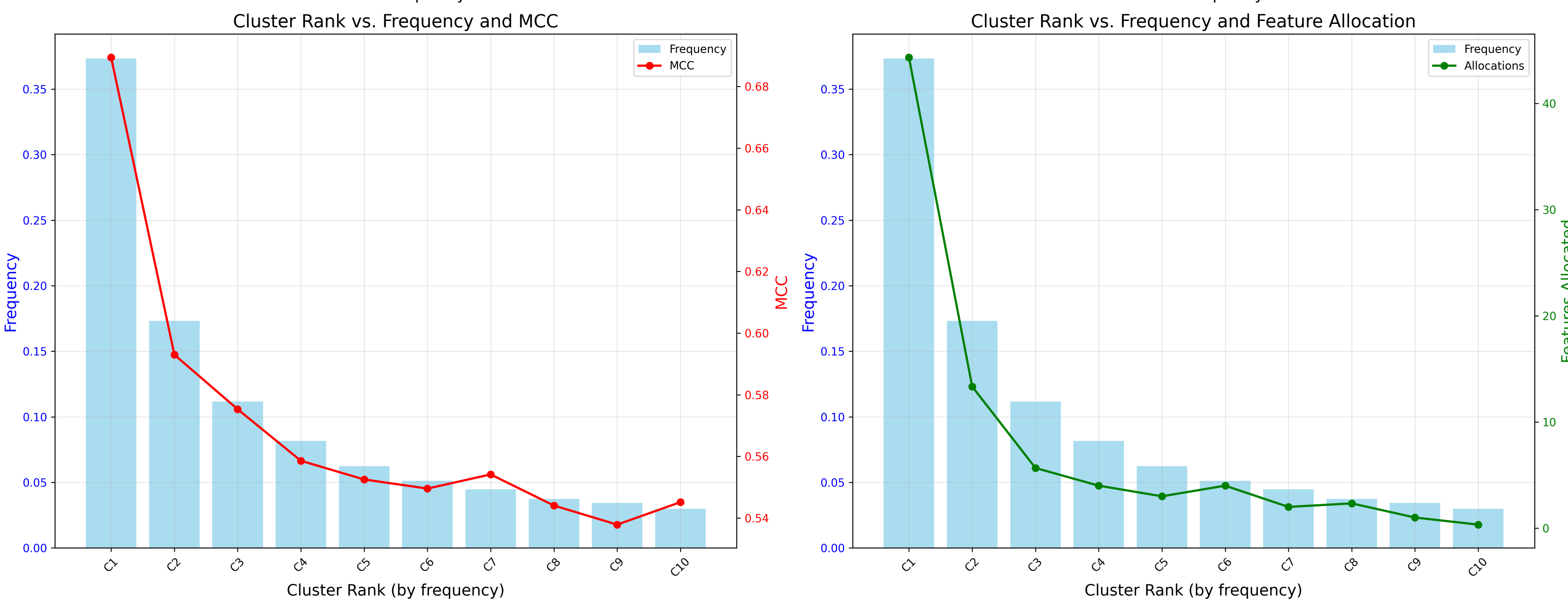}
    \caption{Cluster metrics for Zipf distribution with $\alpha=1.1$. Left: Cluster rank vs. probability (blue bars) and MCC scores (red line), showing a steeper decline in feature recovery quality for less probable clusters compared to $\alpha=1.0$. Right: Cluster rank vs. probability (blue bars) and feature allocation (green line), demonstrating more skewed allocation of dictionary features toward high-probability clusters.}
    \label{fig:cluster_metrics_alpha1_1}
\end{figure}

\begin{figure}[htbp]
    \centering
    \begin{minipage}{0.48\textwidth}
        \centering
        \includegraphics[width=\textwidth]{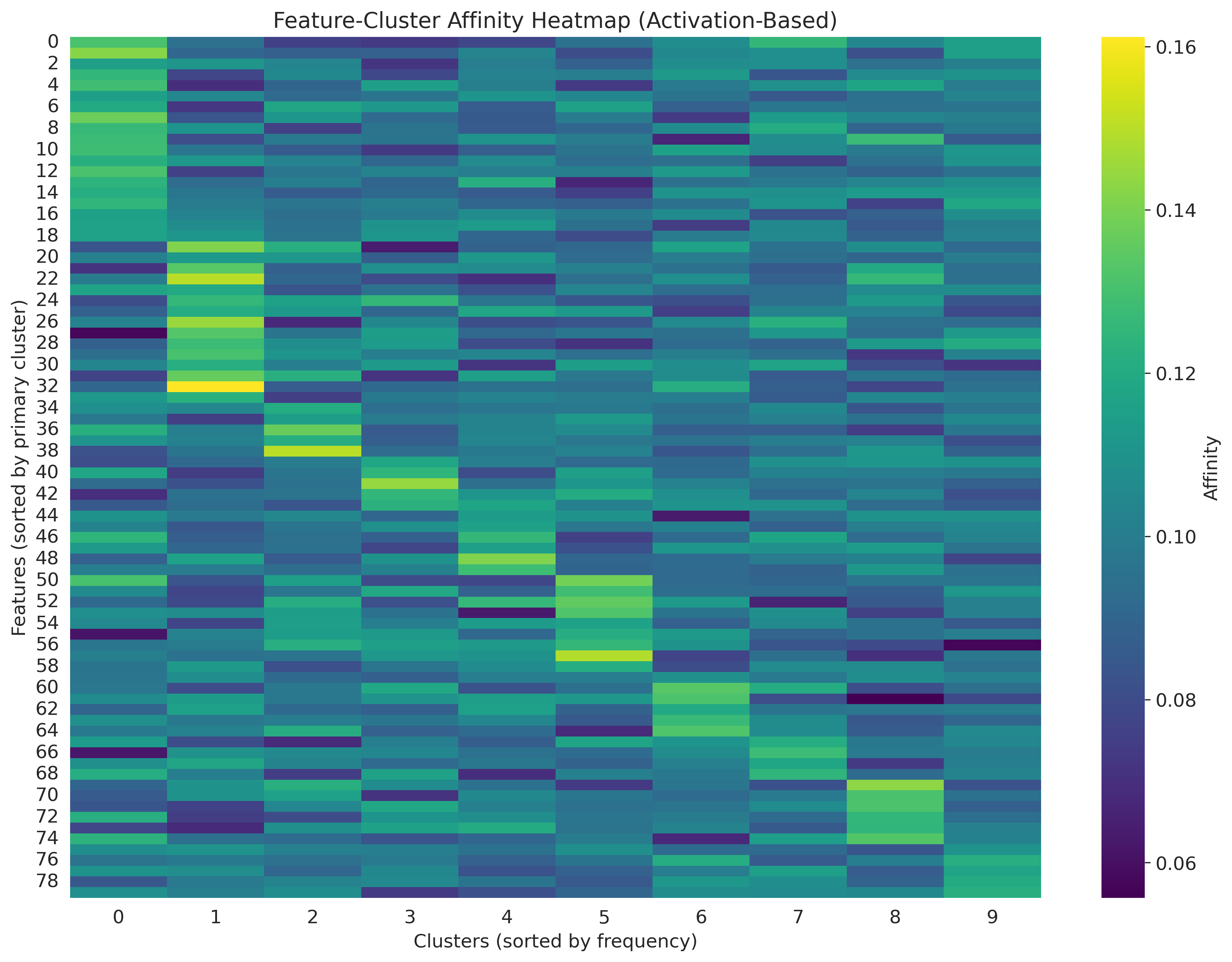}
    \end{minipage}
    \hfill
    \begin{minipage}{0.48\textwidth}
        \centering
        \includegraphics[width=\textwidth]{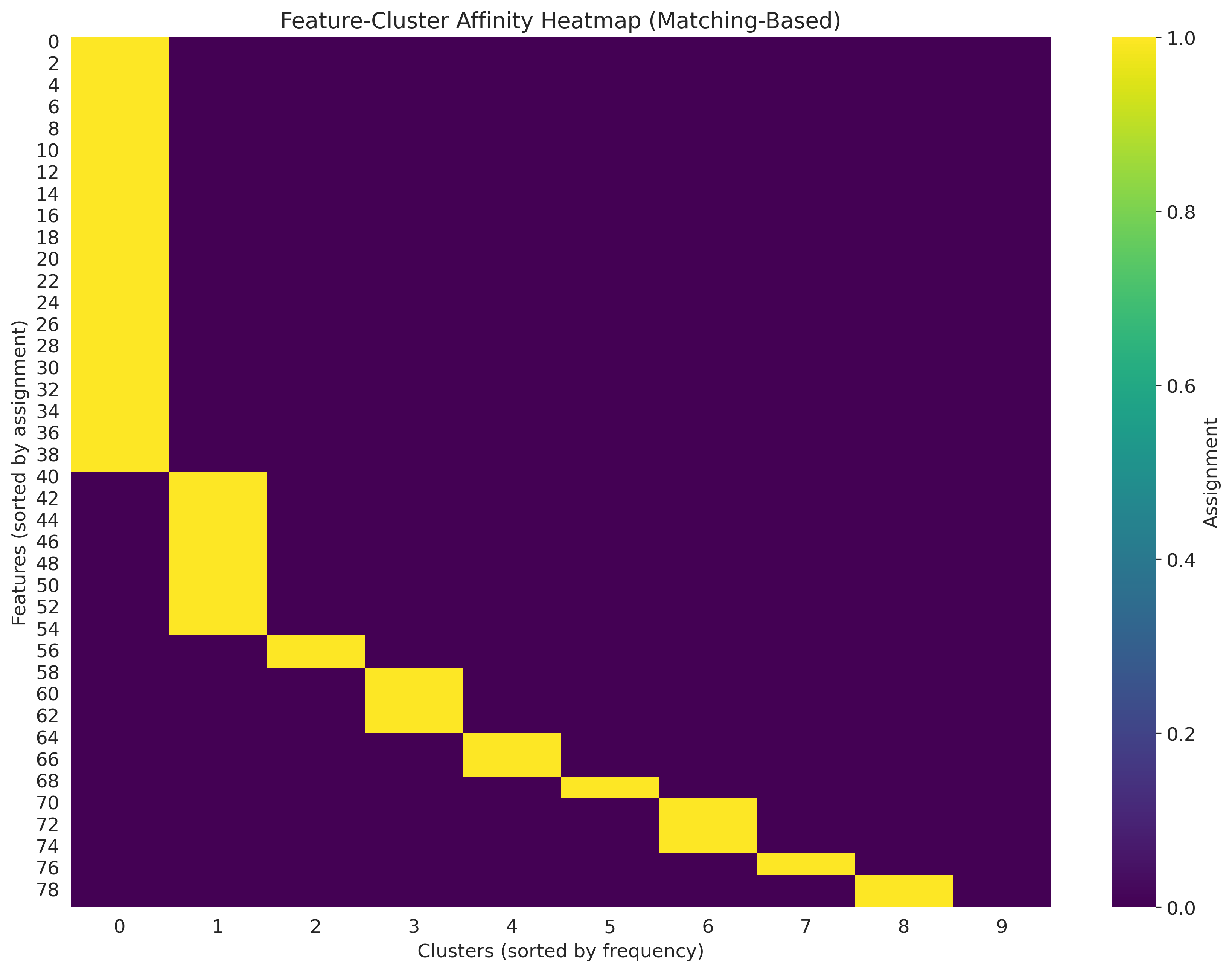}
    \end{minipage}
    \caption{Feature-cluster relationships for Zipf distribution with $\alpha=1.1$. Left: Activation-based affinity heatmap showing stronger feature-to-cluster specialization. compared to $\alpha=1.0$. Right: Matching-based affinity heatmap showing increased skew in feature assignments, with high-probability clusters receiving disproportionately more feature allocations compared to $\alpha=1.0$.}
    \label{fig:affinity_alpha1_1}
\end{figure}

Increasing the exponent to $\alpha=1.1$ creates a slightly more skewed distribution. Comparing Figure \ref{fig:capacity_similarity_alpha1_1} with Figure \ref{fig:capacity_similarity_alpha1}, we observe a steeper power law curve in the capacity allocation model ($D_i \propto p_i^{\beta}$ with $\beta \approx 1.455$), indicating that high-probability clusters now receive an even larger share of the dictionary capacity. We see a similar positive trend between feature activation frequency and feature reproducibility, with no noticeable increase in the trend.

Figure \ref{fig:cluster_metrics_alpha1_1} reveals a more dramatic drop-off in feature recoverability for lower-probability clusters, and Figure \ref{fig:affinity_alpha1_1} shows increased skew in feature allocation, with high-probability clusters receiving proportionally more features than in the $\alpha=1.0$ case. This skew is less pronounced when measured through activation-based affinity (left) compared to the more extreme allocation in the Hungarian matching assignments (right).

\subsubsection{Zipf Distribution with $\alpha=1.5$}
\label{app:zipf_alpha1_5}

\begin{figure}[htbp]
    \centering
    \begin{minipage}{0.48\textwidth}
        \centering
        \includegraphics[width=\textwidth]{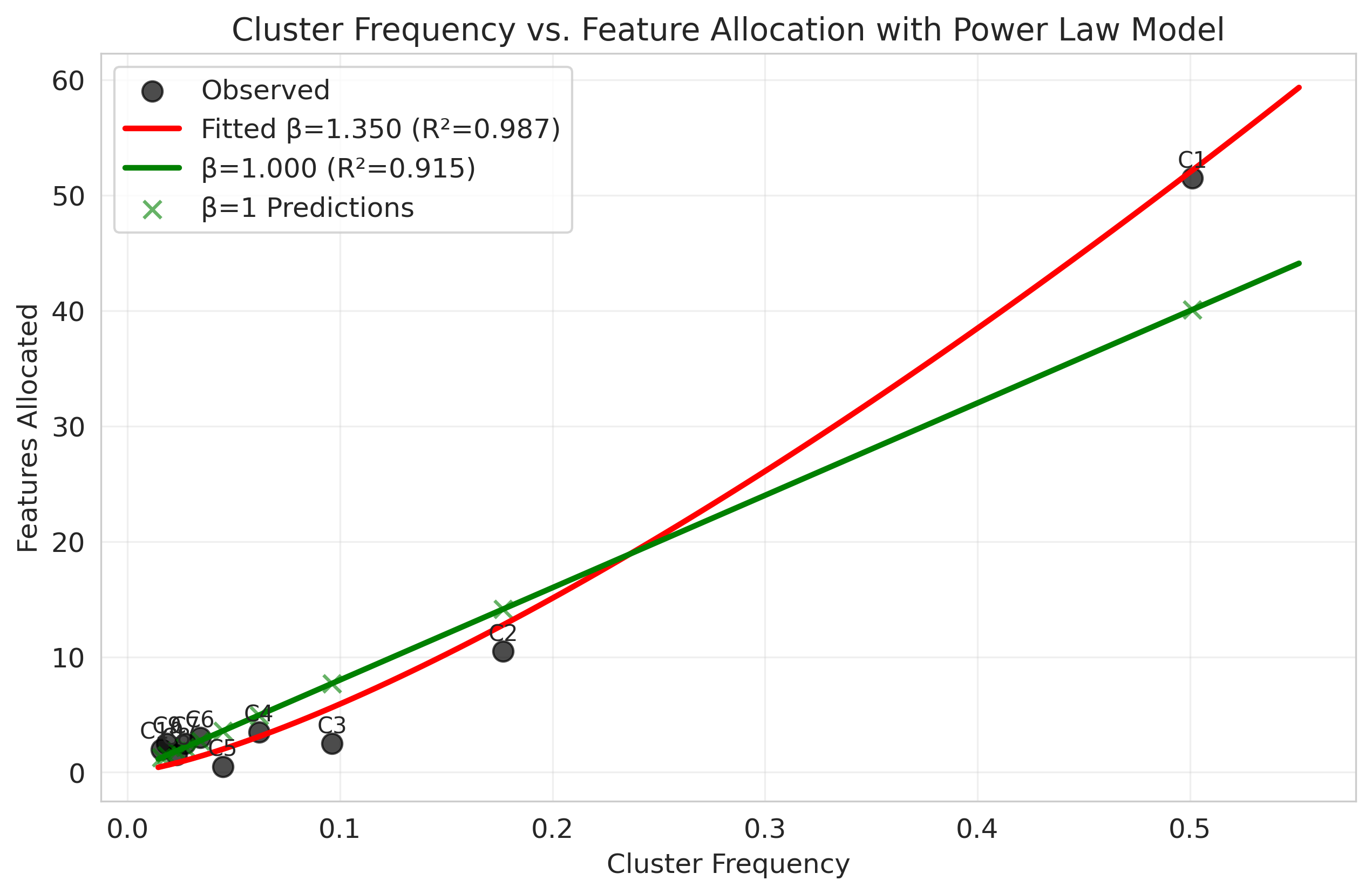}
    \end{minipage}
    \hfill
    \begin{minipage}{0.48\textwidth}
        \centering
        \includegraphics[width=\textwidth]{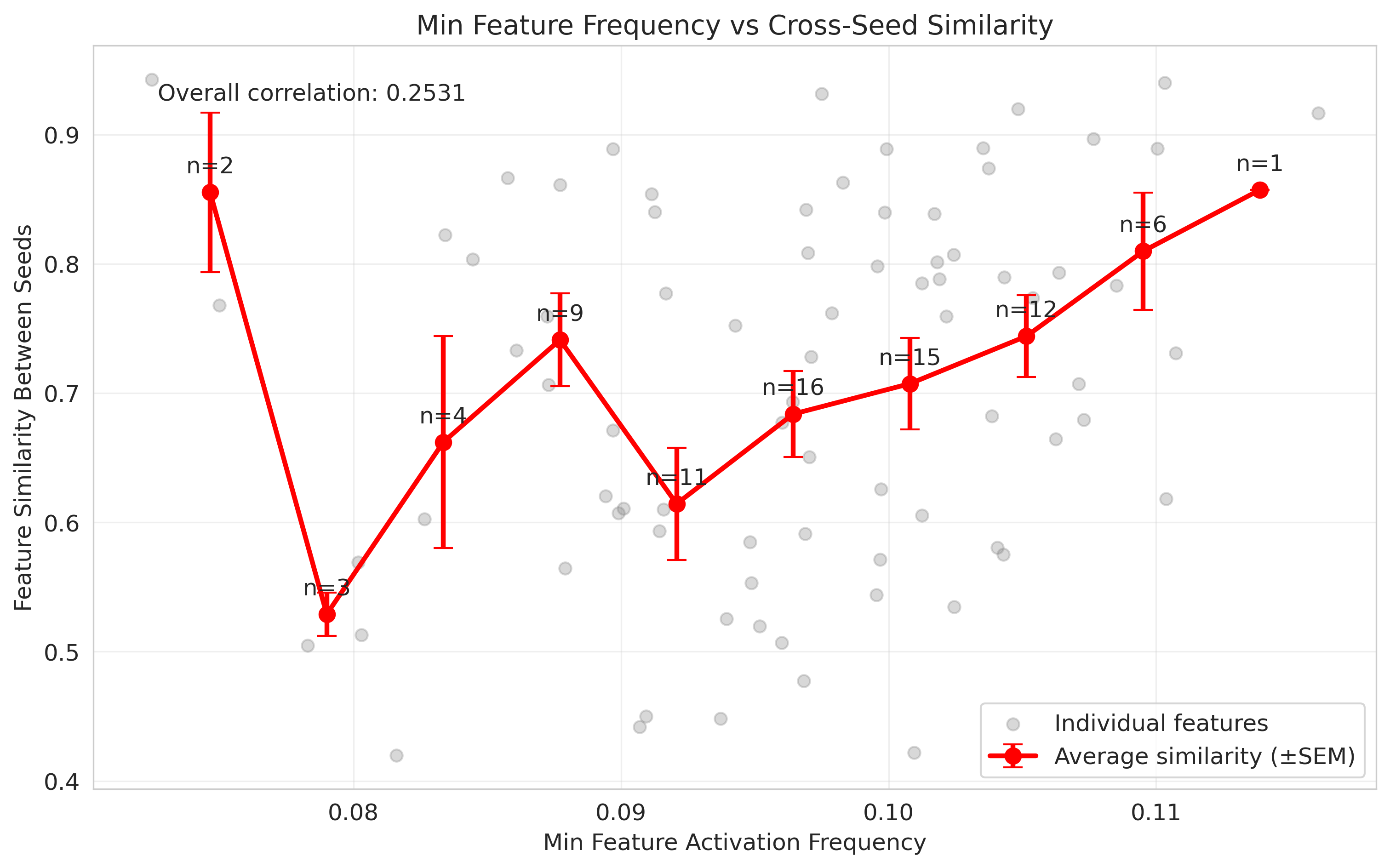}
    \end{minipage}
    \caption{Left: Capacity allocation model for Zipf distribution with $\alpha=1.5$, showing significantly more skewed allocation of SAE features to clusters. Red curve shows fitted power law model with $D_i \propto p_i^{\beta}$ where $\beta \approx 1.35$. Right: Feature similarity between independently trained SAEs as a function of minimum feature activation frequency showing a weak positive trend.}
    \label{fig:capacity_similarity_alpha1_5}
\end{figure}

\begin{figure}[htbp]
    \centering
    \includegraphics[width=0.9\textwidth]{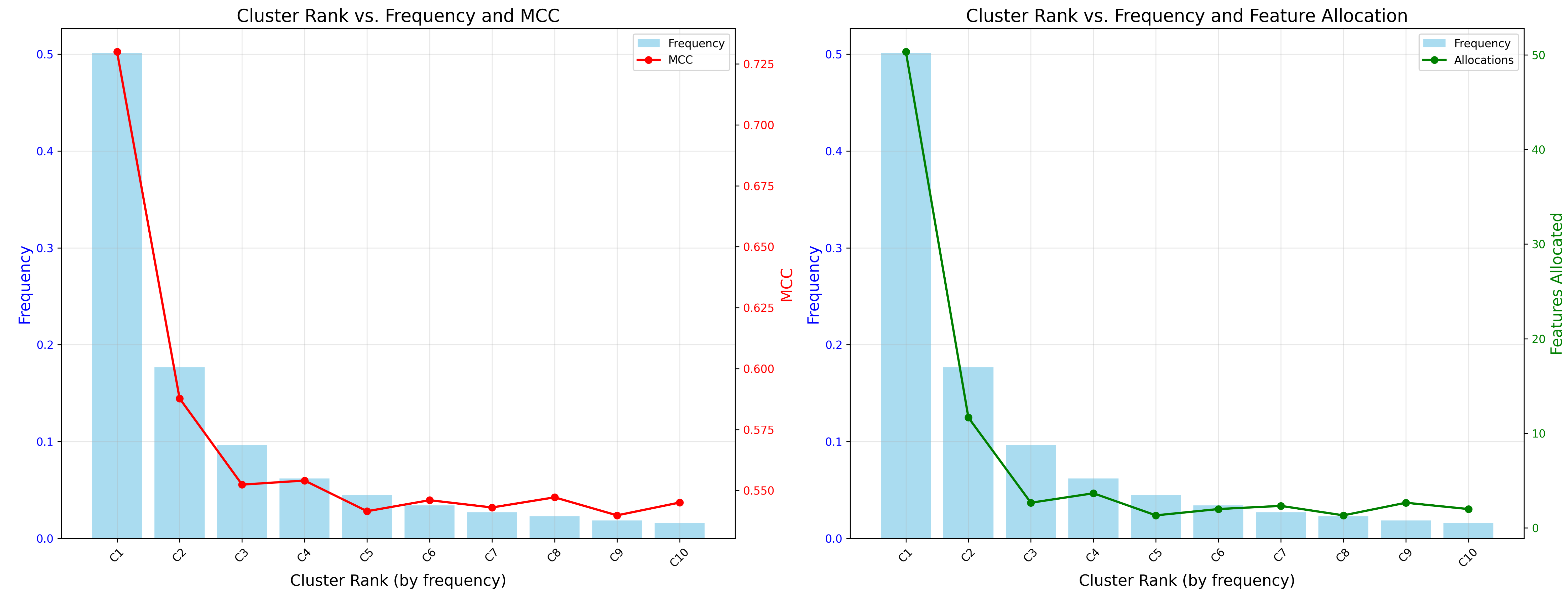}
    \caption{Cluster metrics for Zipf distribution with $\alpha=1.5$. Left: Cluster rank vs. probability (blue bars) and MCC scores (red line), showing a sharp threshold effect where feature recovery quality drops dramatically beyond the highest-probability clusters. Right: Cluster rank vs. probability (blue bars) and feature allocation (green line), demonstrating highly concentrated allocation of dictionary features to the most probable clusters.}
    \label{fig:cluster_metrics_alpha1_5}
\end{figure}

\begin{figure}[htbp]
    \centering
    \begin{minipage}{0.48\textwidth}
        \centering
        \includegraphics[width=\textwidth]{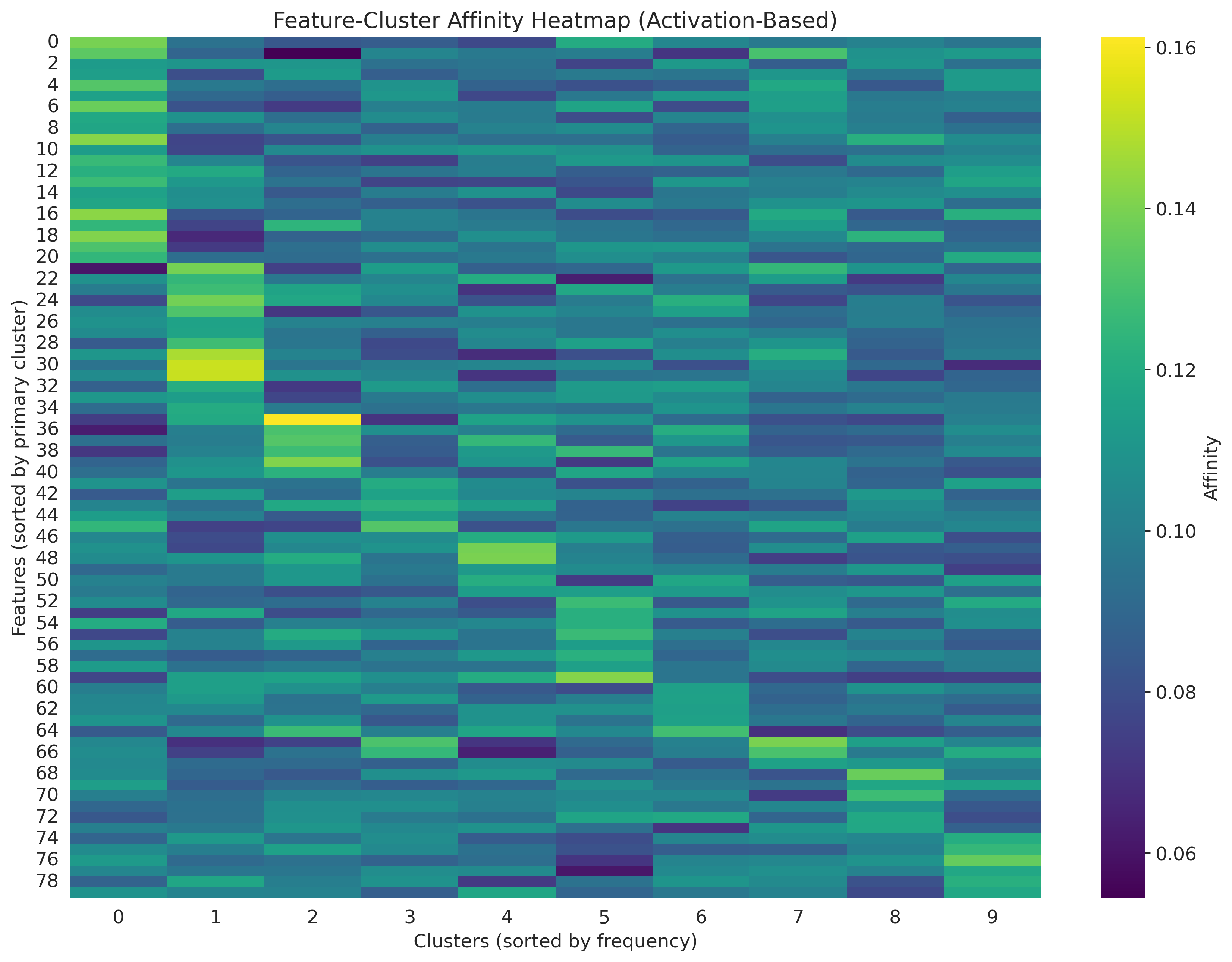}
    \end{minipage}
    \hfill
    \begin{minipage}{0.48\textwidth}
        \centering
        \includegraphics[width=\textwidth]{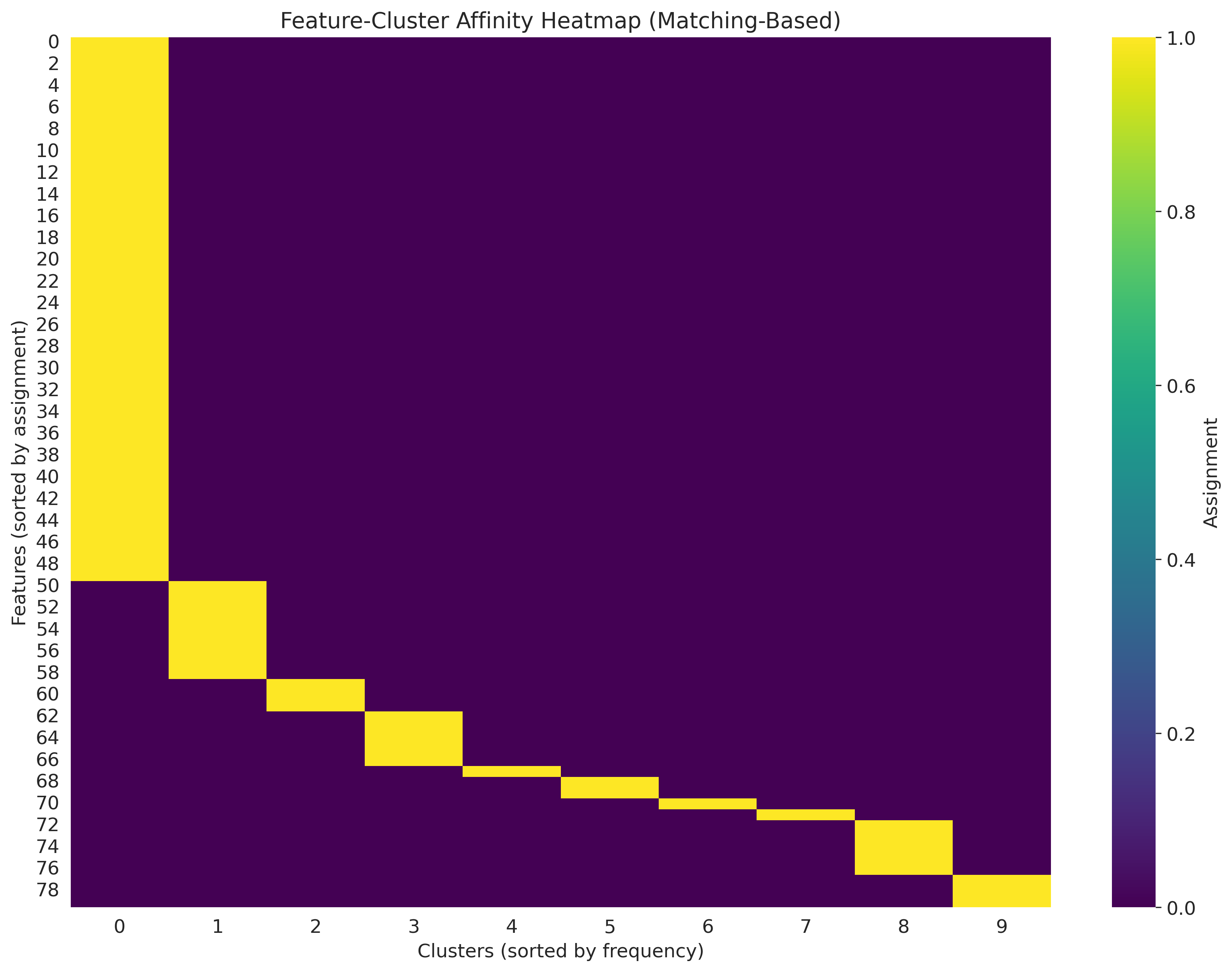}
    \end{minipage}
    \caption{Feature-cluster relationships for Zipf distribution with $\alpha=1.5$. Left: Activation-based affinity heatmap showing high feature specialization with minimal cross-activation. Right: Matching-based affinity heatmap showing strong one-to-one mapping for high-probability clusters but poor assignment for low-probability clusters.}
    \label{fig:affinity_alpha1_5}
\end{figure}

With $\alpha=1.5$, we observe a highly skewed distribution where a small number of probable clusters dominate. Figure \ref{fig:capacity_similarity_alpha1_5} shows that the capacity allocation follows a power law with $D_i \propto p_i^{\beta}$ where $\beta \approx 1.35$, with the majority of dictionary features allocated to the highest-probability clusters. The feature similarity plot shows a positive relationship between the feature activation frequency and feature similarity between independently trained SAEs, but the effect is not much stronger than the $\alpha=1$ or $\alpha=1.1$ case.

The cluster metrics in Figure \ref{fig:cluster_metrics_alpha1_5} show MCC scores dropping precipitously beyond the most probable clusters. The feature-cluster affinity maps in Figure \ref{fig:affinity_alpha1_5} also show highly specialized features based on frequency but the skew when measured through activation-based affinity is less pronounced as compared to Hungarian matching assignments.

\subsubsection{Zipf Distribution with $\alpha=2.0$}
\label{app:zipf_alpha2}

\begin{figure}[htbp]
    \centering
    \begin{minipage}{0.48\textwidth}
        \centering
        \includegraphics[width=\textwidth]{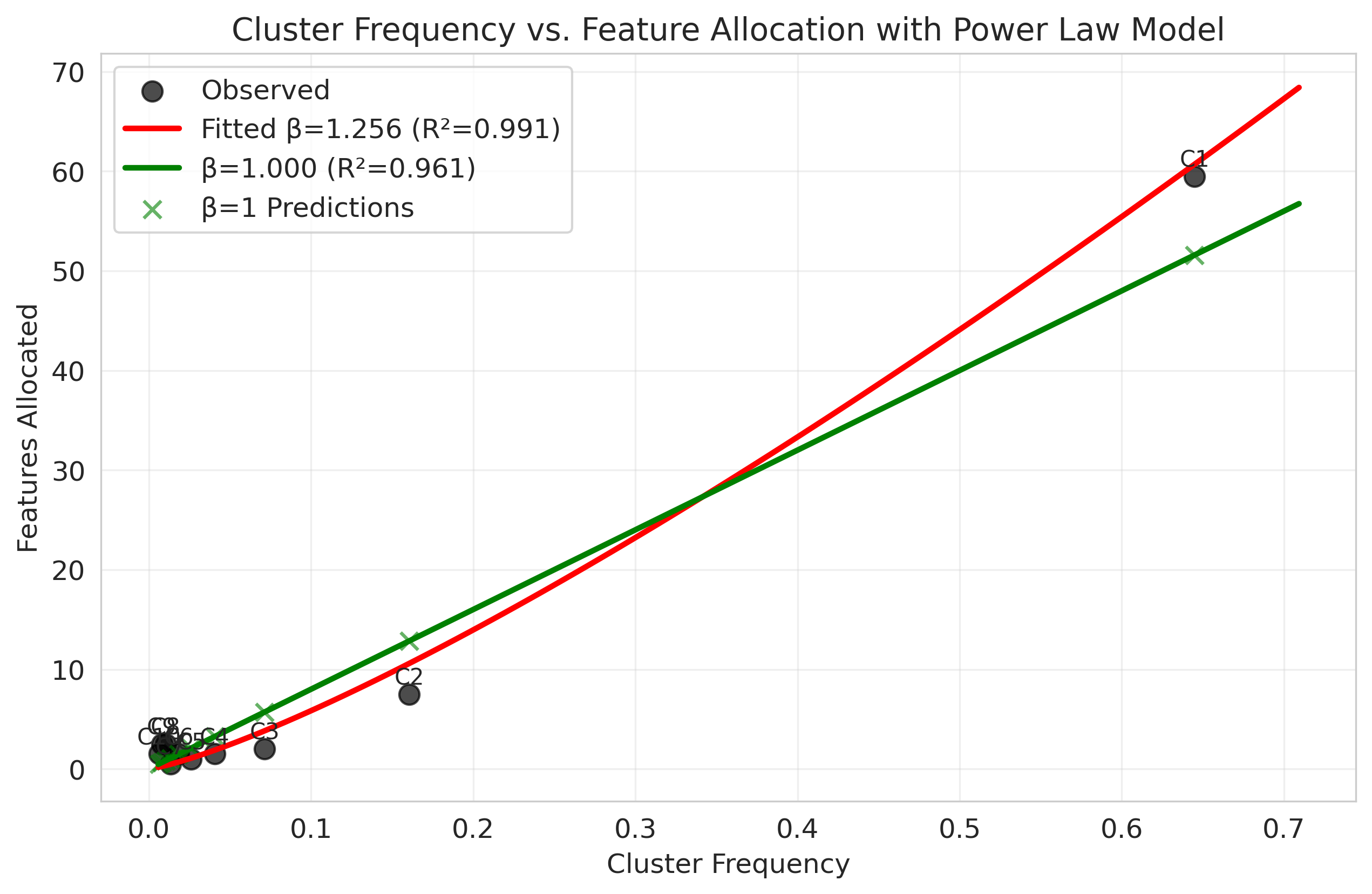}
    \end{minipage}
    \hfill
    \begin{minipage}{0.48\textwidth}
        \centering
        \includegraphics[width=\textwidth]{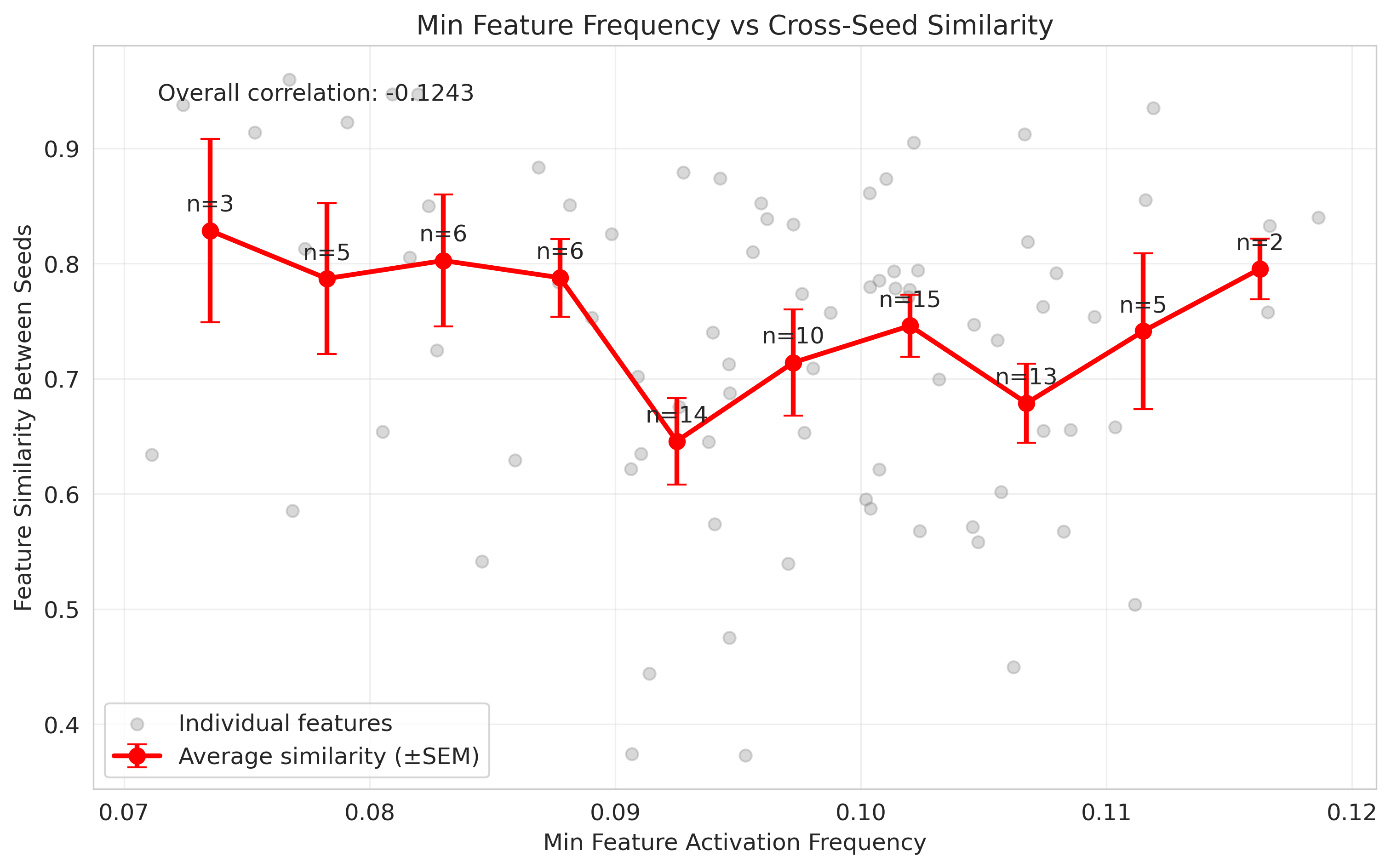}
    \end{minipage}
    \caption{Left: Capacity allocation model for Zipf distribution with $\alpha=2.0$, showing extreme concentration of SAE features to the highest-probability clusters. Red curve shows fitted power law model with $D_i \propto p_i^{\beta}$ where $\beta \approx 1.256$. Right: Feature similarity between independently trained SAEs as a function of minimum feature activation frequency showing a flat to weak positive trend.}
    \label{fig:capacity_similarity_alpha2}
\end{figure}

\begin{figure}[htbp]
    \centering
    \includegraphics[width=0.9\textwidth]{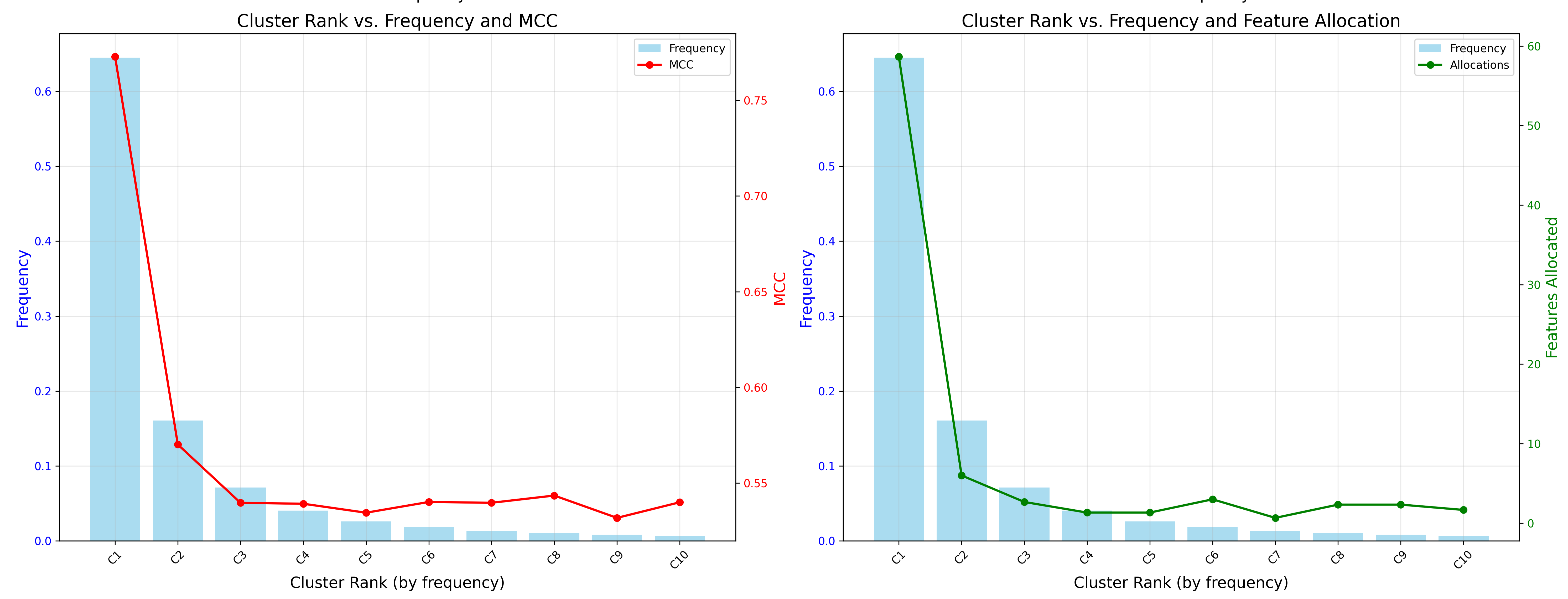}
    \caption{Cluster metrics for Zipf distribution with $\alpha=2.0$. Left: Cluster rank vs. probability (blue bars) and MCC scores (red line), showing that the very highest-probability clusters achieve good feature recovery. Right: Cluster rank vs. probability (blue bars) and feature allocation (green line), demonstrating that dictionary features are almost exclusively allocated to the top clusters, with negligible capacity for the long tail.}
    \label{fig:cluster_metrics_alpha2}
\end{figure}

\begin{figure}[htbp]
    \centering
    \begin{minipage}{0.48\textwidth}
        \centering
        \includegraphics[width=\textwidth]{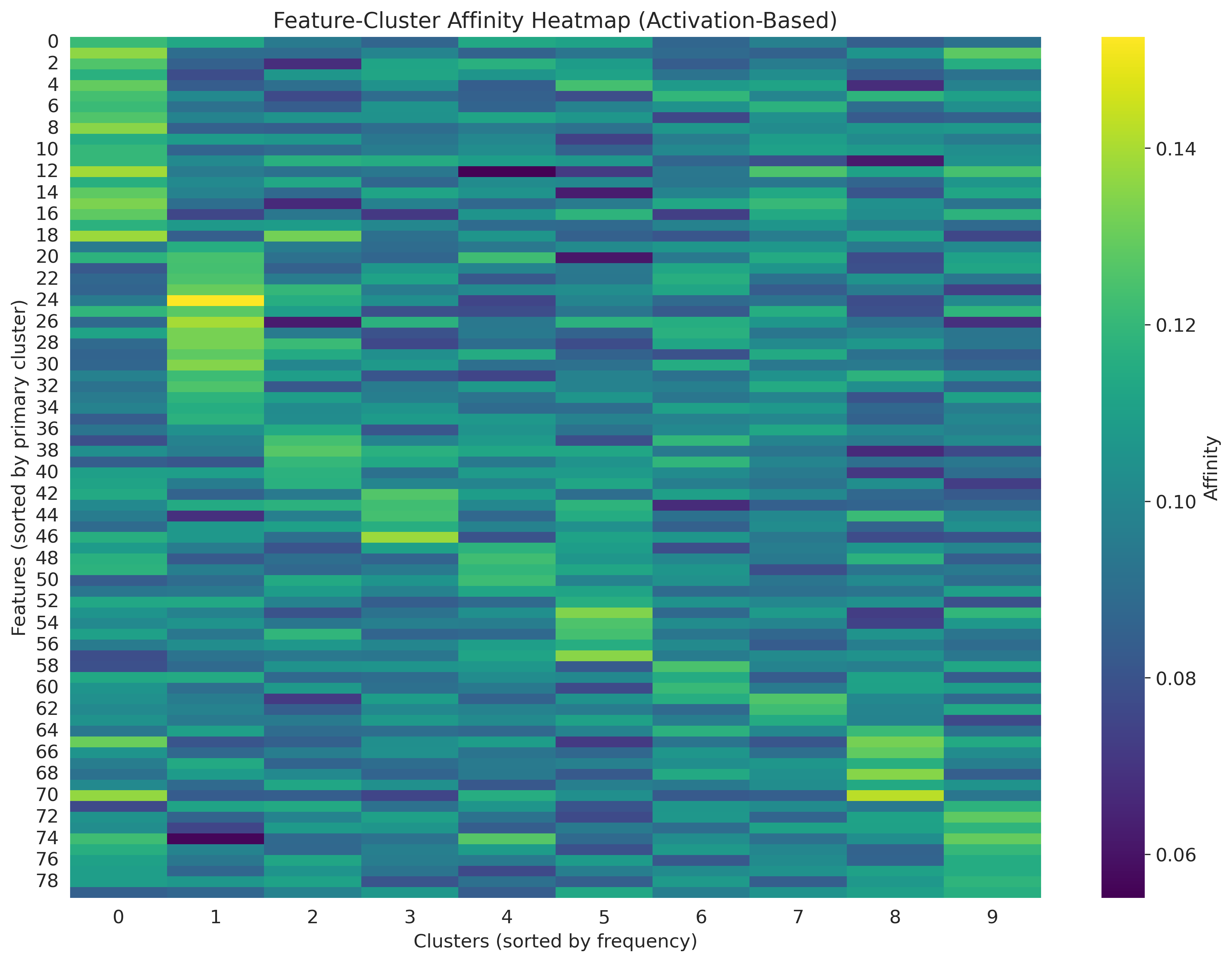}
    \end{minipage}
    \hfill
    \begin{minipage}{0.48\textwidth}
        \centering
        \includegraphics[width=\textwidth]{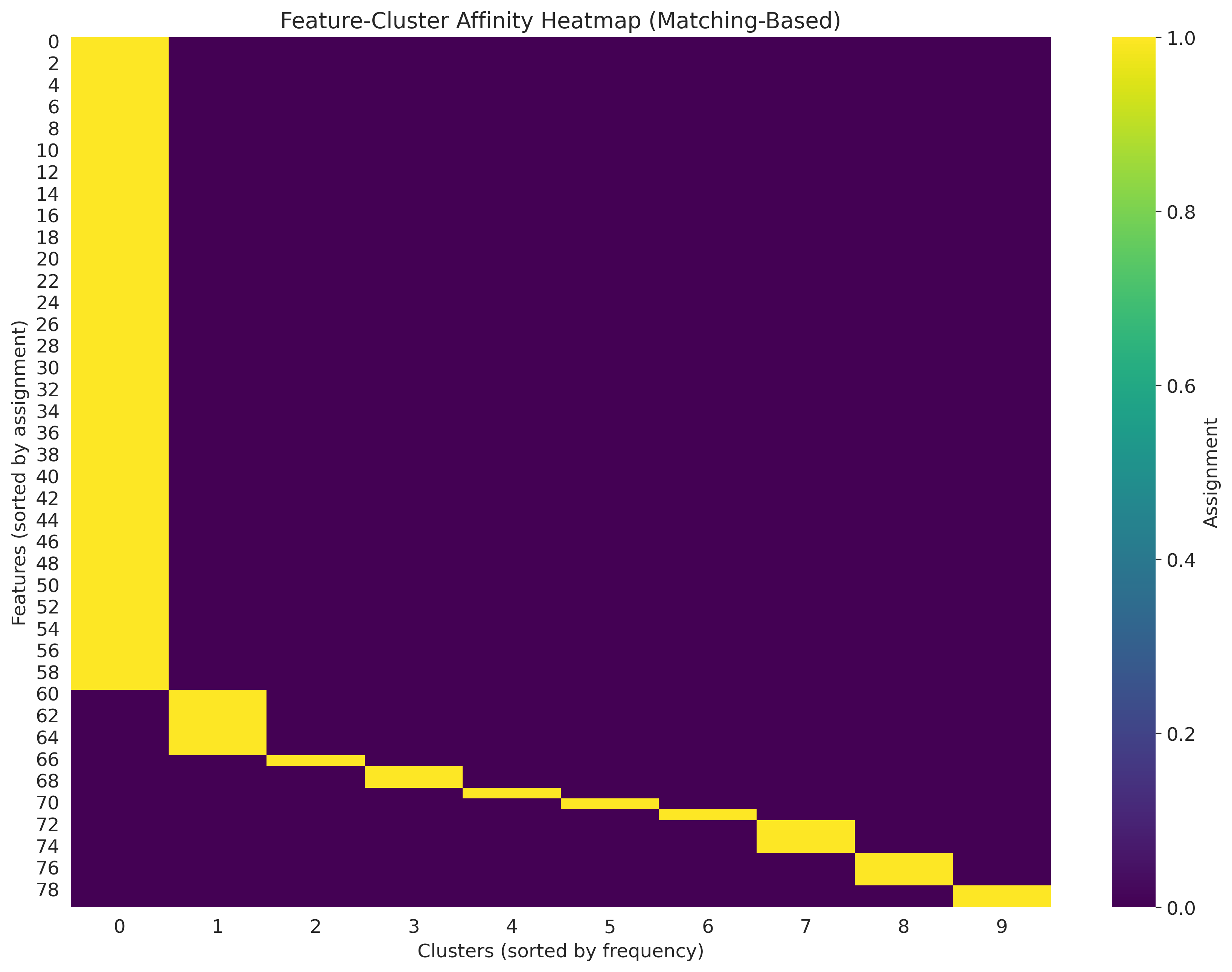}
    \end{minipage}
    \caption{Feature-cluster relationships for Zipf distribution with $\alpha=2.0$. Left: Activation-based affinity heatmap showing specialization to high-probability clusters. Right: Matching-based affinity heatmap showing strong assignment for only the highest-probability clusters, with the majority of clusters receiving minimal or no feature representation.}
    \label{fig:affinity_alpha2}
\end{figure}

At $\alpha=2.0$, we observe an extremely skewed distribution where a handful of clusters dominate the probability distribution. Figure \ref{fig:capacity_similarity_alpha2} shows that dictionary capacity is allocated according to a power law with $D_i \propto p_i^{\beta}$ where $\beta \approx 1.256$, with capacity concentrated in the highest-probability clusters. The feature similarity plot shows a flat to weak positive relationship between the feature activation frequency and feature similarity between independently trained SAEs.

The cluster metrics in Figure \ref{fig:cluster_metrics_alpha2} confirm that only the very top clusters achieve meaningful feature recovery, with the vast majority of clusters poorly represented.

\subsection{Analysis of Dictionary Size Effects in Two-Phase Distributions}
\label{app:two_phase}

To better approximate real language data distributions, we developed a two-phase model that combines different power laws. This section examines how dictionary size affects feature learning in this more realistic distribution. We enhanced our model organism with a two-phase feature cluster frequency distribution combining a Mandelbrot-Zipf function ($g(r; s_1, q) = (r+q)^{-s_1}$, $s_1=1.05$, $q=5.0$) for common concepts ($r < 40,000$) and a steeper power law ($g(r; s_2) \propto r^{-s_2}$, $s_2=30.0$) for the long tail. For all experiments in this section, we set the ground truth dimension $d_{gt}=400000$, with 50000 clusters (each cluster represented by 8 ground truth features on average), and maintain a constant TopK sparsity parameter $k=8$, while varying the SAE dictionary size $d_{sae} \in \{80, 160, 1000, 10000\}$ to analyze the impact of SAE capacity on feature learning. Table \ref{tab:two_phase_params} summarizes the hyperparameters used in these experiments.

\begin{table}[h]
    \small
    \centering
    \caption{Hyperparameters for Two-Phase Distribution Experiments}
    \label{tab:two_phase_params}
    \begin{tabular}{lr}
    \toprule
    \textbf{Parameter} & \textbf{Value} \\
    \midrule
    TopK sparsity parameter ($k$) & 8 \\
    Activation dimension & 20 \\
    Dictionary size ($d_{sae}$) & varies (80, 160, 1000, 10000) \\
    Training examples & 100,000 \\
    Training steps & 20,000 \\
    Learning rate & 0.04 \\
    Learning rate decay factor & 0.1 \\
    Learning rate decay steps & [20,000] \\
    Warmup steps & 1,000 \\
    Minimum learning rate & 1e-05 \\
    L1 coefficient & 0.1 \\
    Batch size & 4,096 \\
    Number of models trained & 3 \\
    Number of clusters & 50,000 \\
    Cluster dimensions & 8 per cluster \\
    First power law exponent ($s_1$) & 1.05 \\
    Second power law exponent ($s_2$) & 30.0 \\
    Transition rank & 40,000 \\
    $q$ parameter & 5.0 \\
    \bottomrule
    \end{tabular}
\end{table}

\begin{figure}[htbp]
    \centering
    \includegraphics[width=0.7\textwidth]{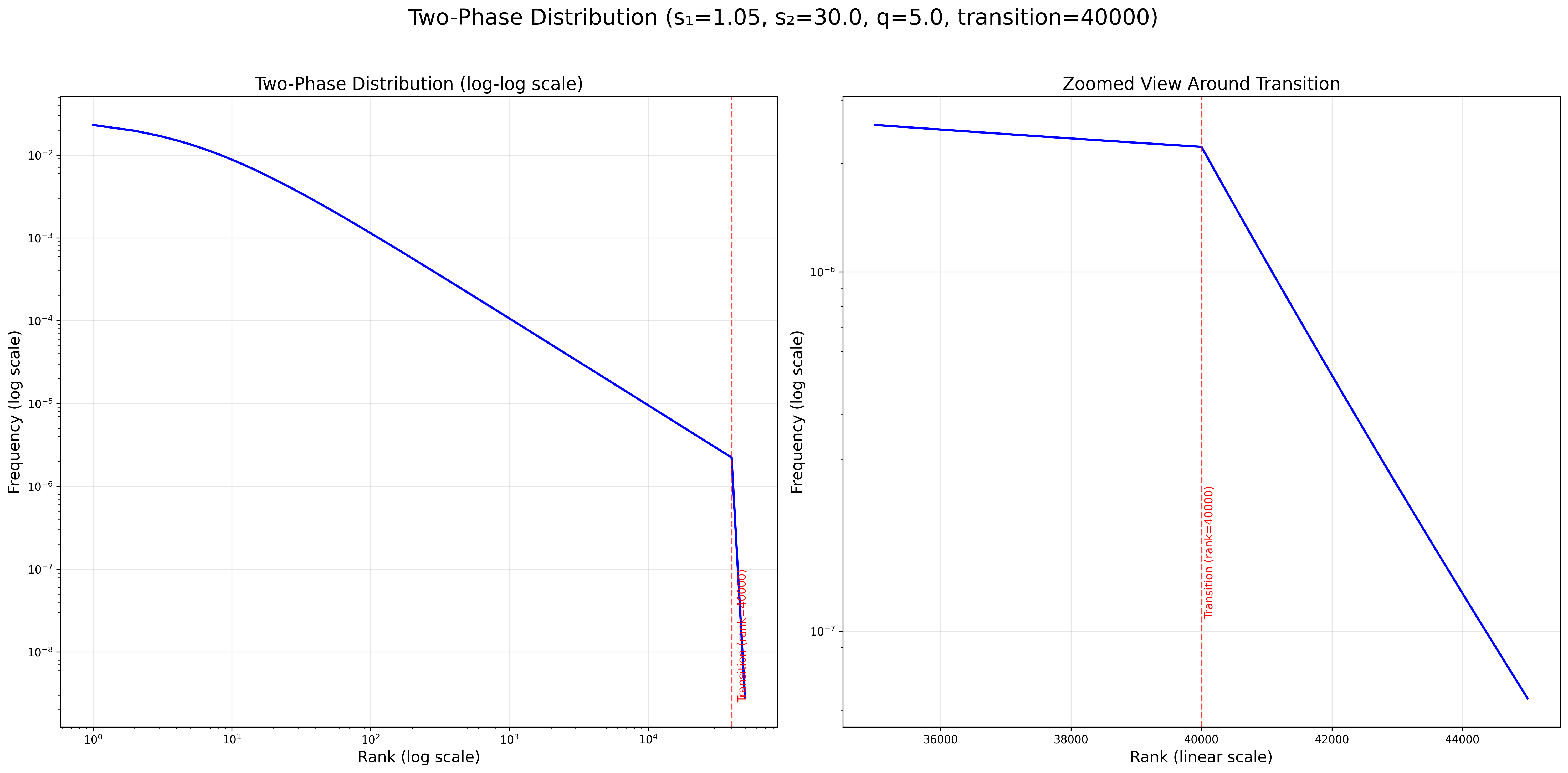}
   \caption{\small Two-phase cluster probability distribution used to approximate real language data. The distribution follows a Mandelbrot-Zipf pattern ($s_1=1.05$) until rank 40,000, then transitions to a steeper power law ($s_2=30.0$) capturing the long tail characteristics of natural language.}
    \label{fig:synthetic_token_distribution}
\end{figure}

Our two-phase distribution depicted in Figure \ref{fig:synthetic_token_distribution} combines a shallow power law for frequent tokens ($s_1=1.05$ until rank 40,000) with a much steeper power law ($s_2=30.0$) for the long tail. This creates a realistic approximation of natural language distributions, which exhibit similar two-phase characteristics (\autoref{fig:synthetic_token_distribution_main_text_pos_v5}).

\begin{figure}[htbp]
    \centering
    \begin{minipage}{0.48\textwidth}
        \centering
        \includegraphics[width=\textwidth]{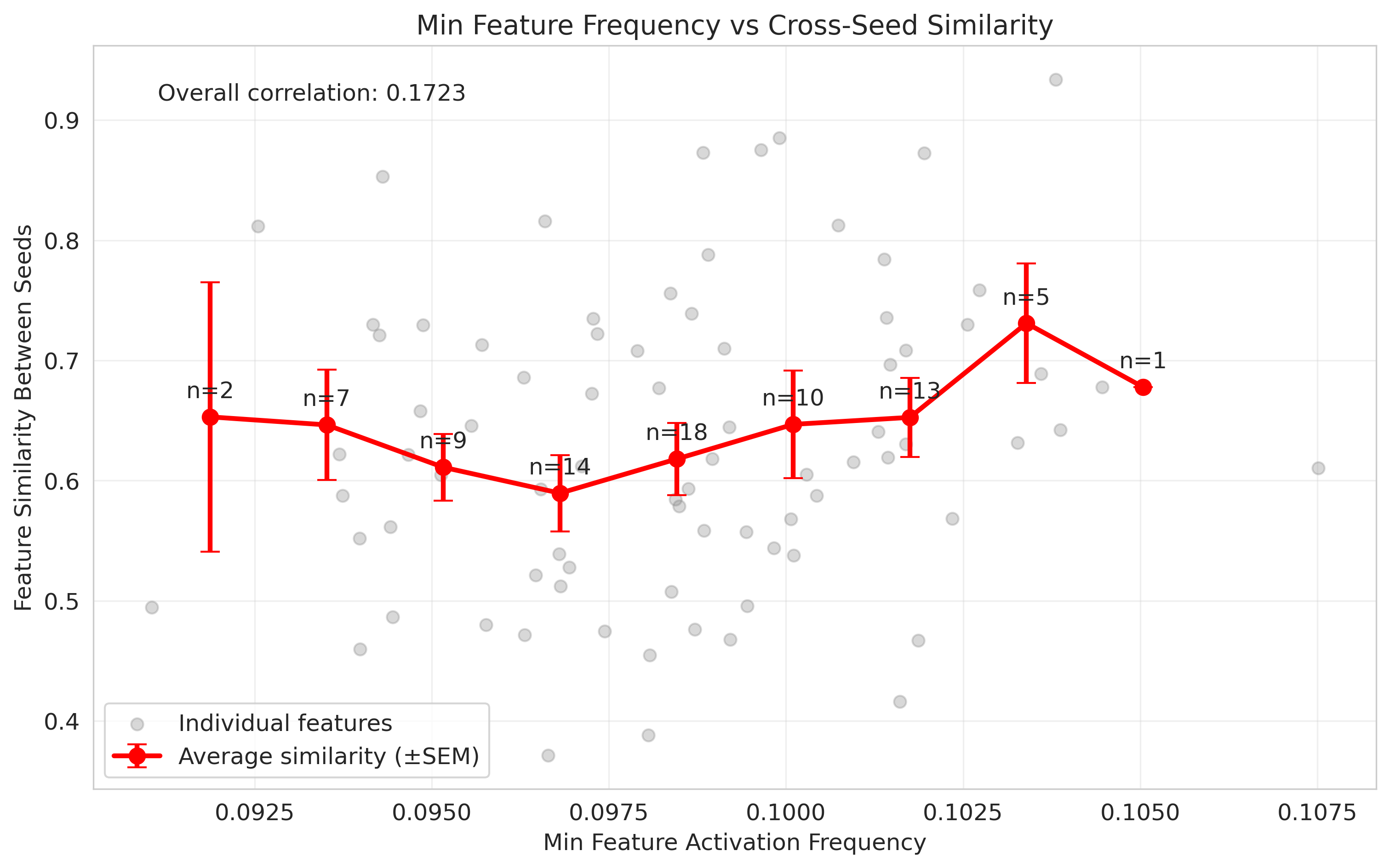}
  \caption{Two-phase model with dictionary size 80. Feature reproducibility shows a weak positive relationship with activation frequency.}
        \label{fig:two_phase_80}
    \end{minipage}
    \hfill
    \begin{minipage}{0.48\textwidth}
        \centering
        \includegraphics[width=\textwidth]{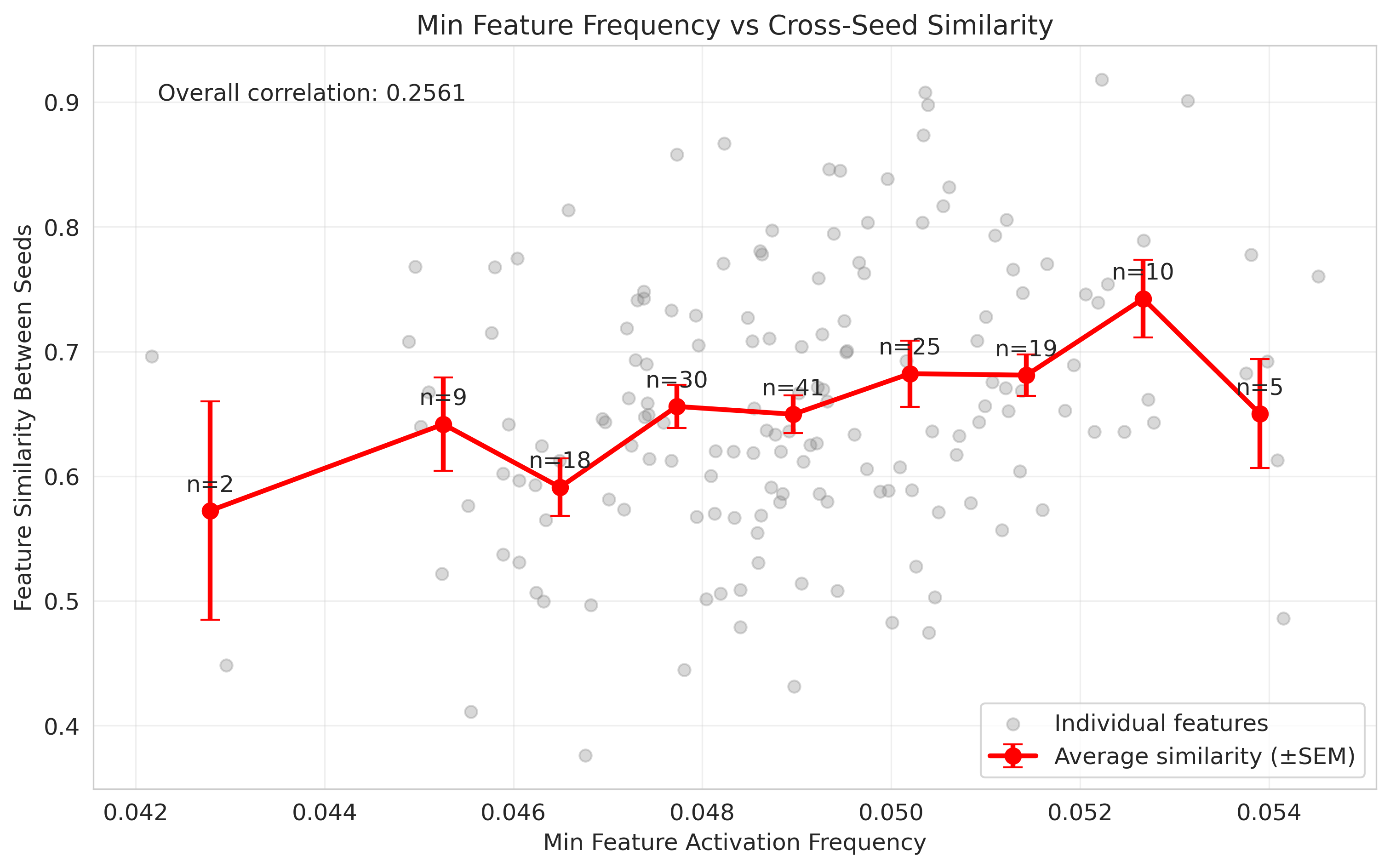}
 \caption{Two-phase model with dictionary size 160. The relationship between activation frequency and feature reproducibility remains weak but becomes slightly more pronounced compared to dictionary size 80.}
        
        \label{fig:two_phase_160}
    \end{minipage}
\end{figure}

\begin{figure}[htbp]
    \centering
    \begin{minipage}{0.48\textwidth}
        \centering
        \includegraphics[width=\textwidth]{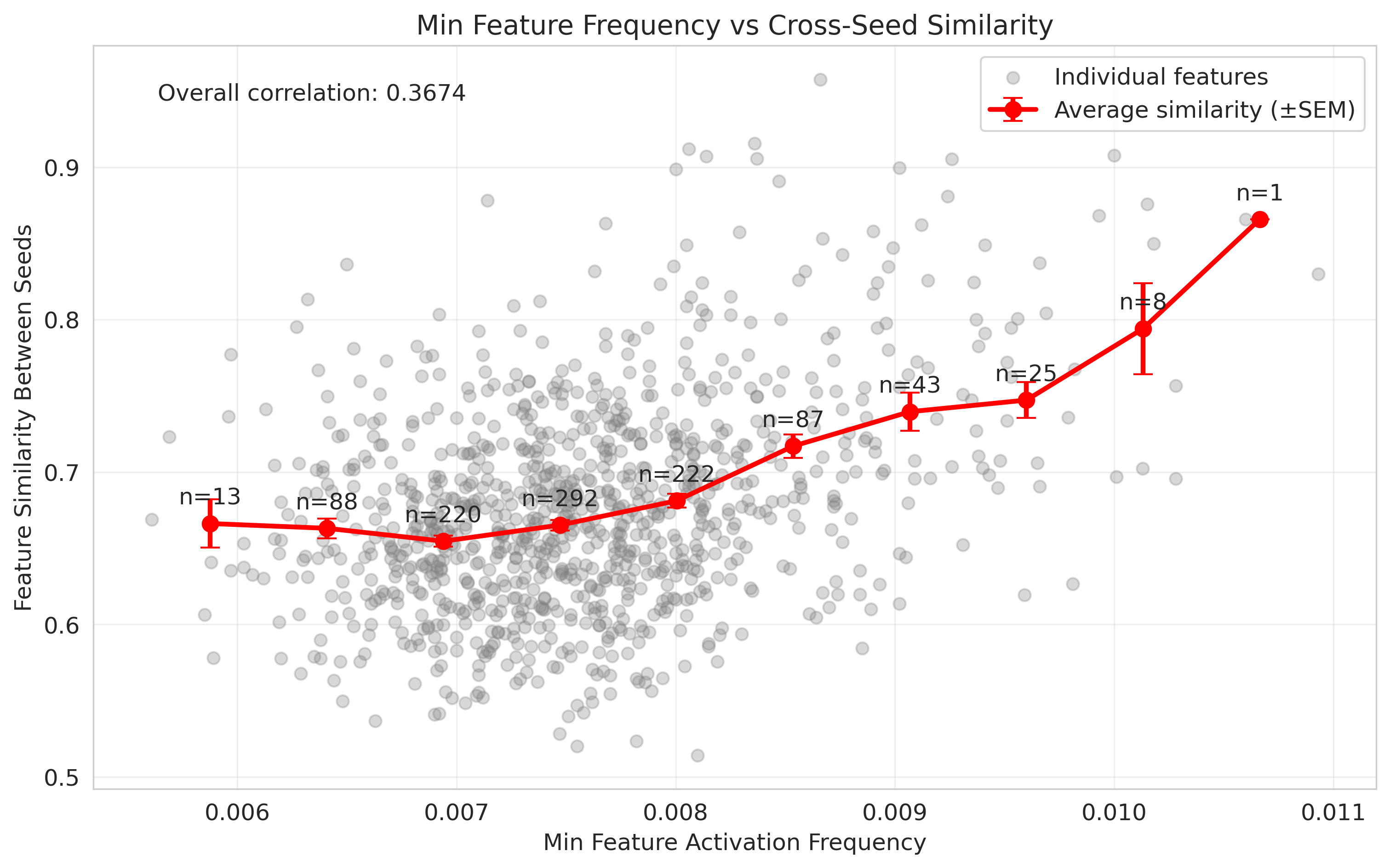}
   \caption{Two-phase model with dictionary size 1000. Feature reproducibility shows a moderately strong positive correlation with activation frequency especially at higher activation frequencies. Increased model capacity creates sufficient local redundancy for high probability clusters.}
        \label{fig:two_phase_1000}
    \end{minipage}
    \hfill
    \begin{minipage}{0.48\textwidth}
        \centering
        \includegraphics[width=\textwidth]{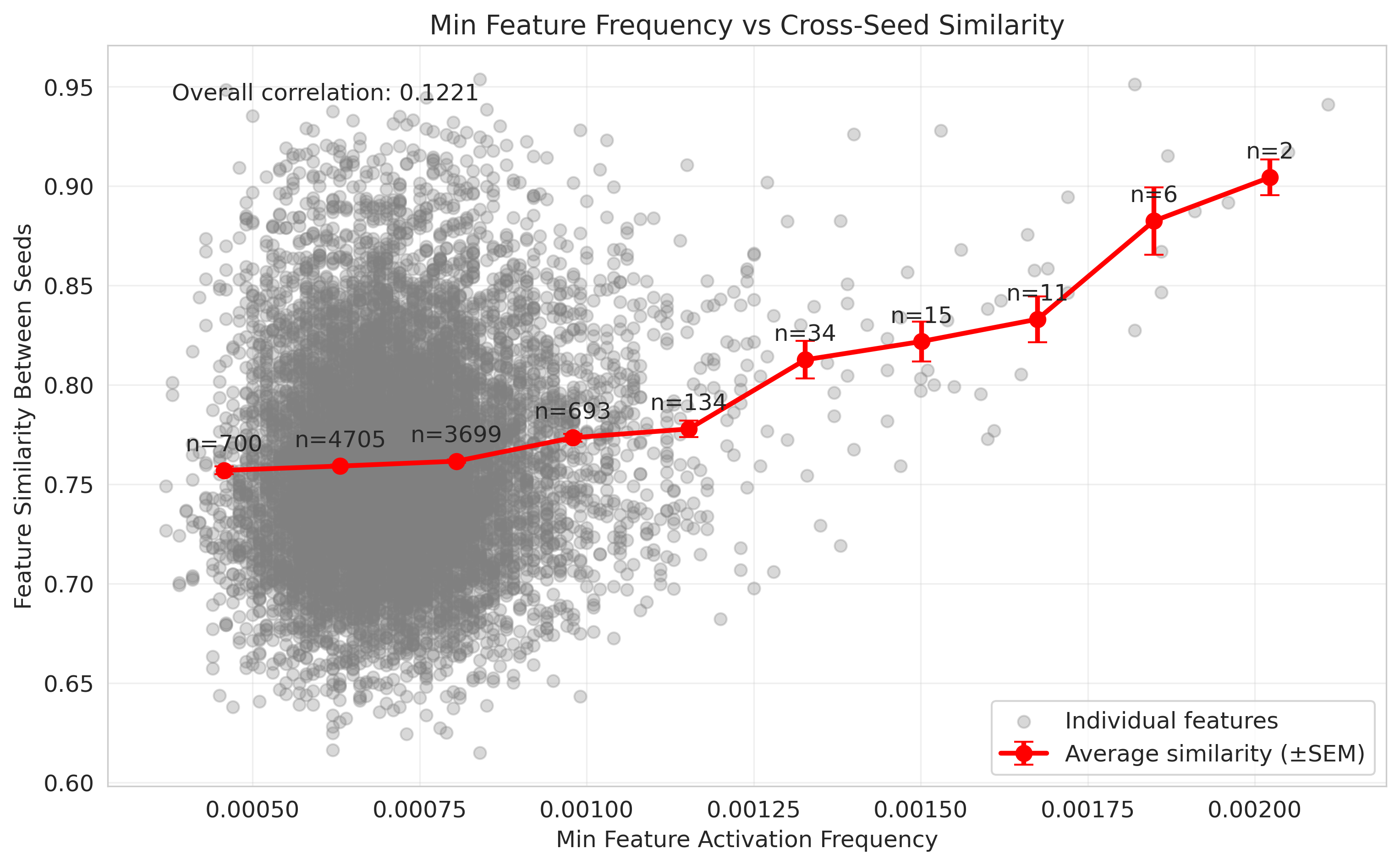}
\caption{Two-phase model with dictionary size 10000. With substantially increased capacity, feature reproducibility exhibits a strong positive correlation with activation frequency across a wide frequency range.  Increased model capacity creates sufficient local redundancy for high probability clusters.}
        \label{fig:two_phase_10000}
    \end{minipage}
\end{figure}

Figures \ref{fig:two_phase_80} through \ref{fig:two_phase_10000} demonstrate how dictionary size affects feature reproducibility across the activation frequency spectrum. Several key trends emerge:

\begin{enumerate}
    \item With small dictionary sizes (80-160 features), we observe only a weak relationship between activation frequency and feature reproducibility. Only the most frequent clusters show consistent reproducibility, indicating severe capacity limitations where the dictionary should prioritize only the dominant clusters.
    
    \item As dictionary size increases to 1000 features, the relationship between activation frequency and reproducibility becomes more pronounced. A wider range of moderately frequent features begins to show improved reproducibility, as the increased capacity allows the model to represent more clusters with sufficient local redundancy.
    
    \item At dictionary size 10000, we observe a positive relationship between activation frequency and reproducibility across a wide frequency range. The substantial increase in capacity creates local redundancy for many more clusters, enabling consistent representation of features.
\end{enumerate}

These results demonstrate that dictionary size relative to the distribution characteristics plays a crucial role in determining which features can be consistently learned. In particular, the local redundancy—defined as the ratio of dictionary size to the effective number of clusters above a certain frequency threshold—determines the model's ability to learn reproducible features across different frequency bands. Rather than a threshold effect, we observe a continuum where the strength of correlation between feature frequency and reproducibility increases as more clusters benefit from sufficient local redundancy.

These results demonstrate that dictionary size relative to the distribution characteristics plays an important role in determining which features can be consistently learned. In particular, local redundancy---defined as the ratio of dictionary capacity allocated to a cluster relative to the cluster dimension---determines the model's ability to learn reproducible features across different frequency bands. With larger dictionaries, more clusters achieve sufficient local redundancy, leading to a continuum where the strength of correlation between feature frequency and reproducibility increases as dictionary capacity expands.

\section{Effect of Misspecifying Encoder Sparsity Parameter}
\label{app:misspecifying_k}

In practical applications of TopK SAEs, practitioners must choose the sparsity parameter $k$ without knowledge of the true underlying sparsity $s$ of the data-generating process. This section investigates how misspecification of $k$ relative to the optimal value $s$ affects dictionary recovery quality and training stability.

Consider a conceptual cluster $i$ within the data with ground-truth dictionary $A_i^\star \in \mathbb{R}^{m \times d_{\text{gt}}}$ having unit-norm columns. Let $s < d_{\text{gt}}$ be the true sparsity of coefficient vectors $s^\star$ for signals $x = A_i^\star s^\star$ from this cluster. The SAE uses a TopK encoder with parameter $k$. We define the \textbf{sparsity ratio} as $\rho := k/s$ and focus on understanding the asymmetric effects of under-sparsity ($\rho < 1$) versus over-sparsity ($\rho > 1$) on feature learning.

We conduct experiments in the matched regime where $d_{\text{sae}} = d_{\text{gt}}$, generating synthetic data by first sampling the ground-truth dictionary $\mathbf{A}_{\text{gt}} \in \mathbb{R}^{m \times d_{\text{gt}}}$ from a standard normal distribution with unit-norm columns. For each data point, we randomly select exactly $s$ features and set their activations to independent Gaussian samples, yielding coefficient vector $\mathbf{f}_{\text{gt}}(\mathbf{x}) \in \mathbb{R}^{d_{\text{gt}}}$ with $s$ non-zero entries, then construct data points as $\mathbf{x} = \mathbf{A}_{\text{gt}}\mathbf{f}_{\text{gt}}(\mathbf{x})$. Our experimental configuration uses input dimension $m = 8$, dictionary sizes $d_{\text{gt}} = d_{\text{sae}} = 40$, true sparsity $s = 8$, training samples $N = 50,000$, TopK parameter range $k \in \{2, 4, 6, 8, 10, 12, 14, 16\}$, and 5 independent runs per configuration. We evaluate dictionary recovery using GT-MCC, which measures correlation between learned dictionary $\mathbf{A}$ and ground-truth $\mathbf{A}_{\text{gt}}$ using optimal permutation matching, computed over the final 100 training steps and averaged to reduce noise.

\begin{figure}[htbp]
\centering
\includegraphics[width=0.48\textwidth]{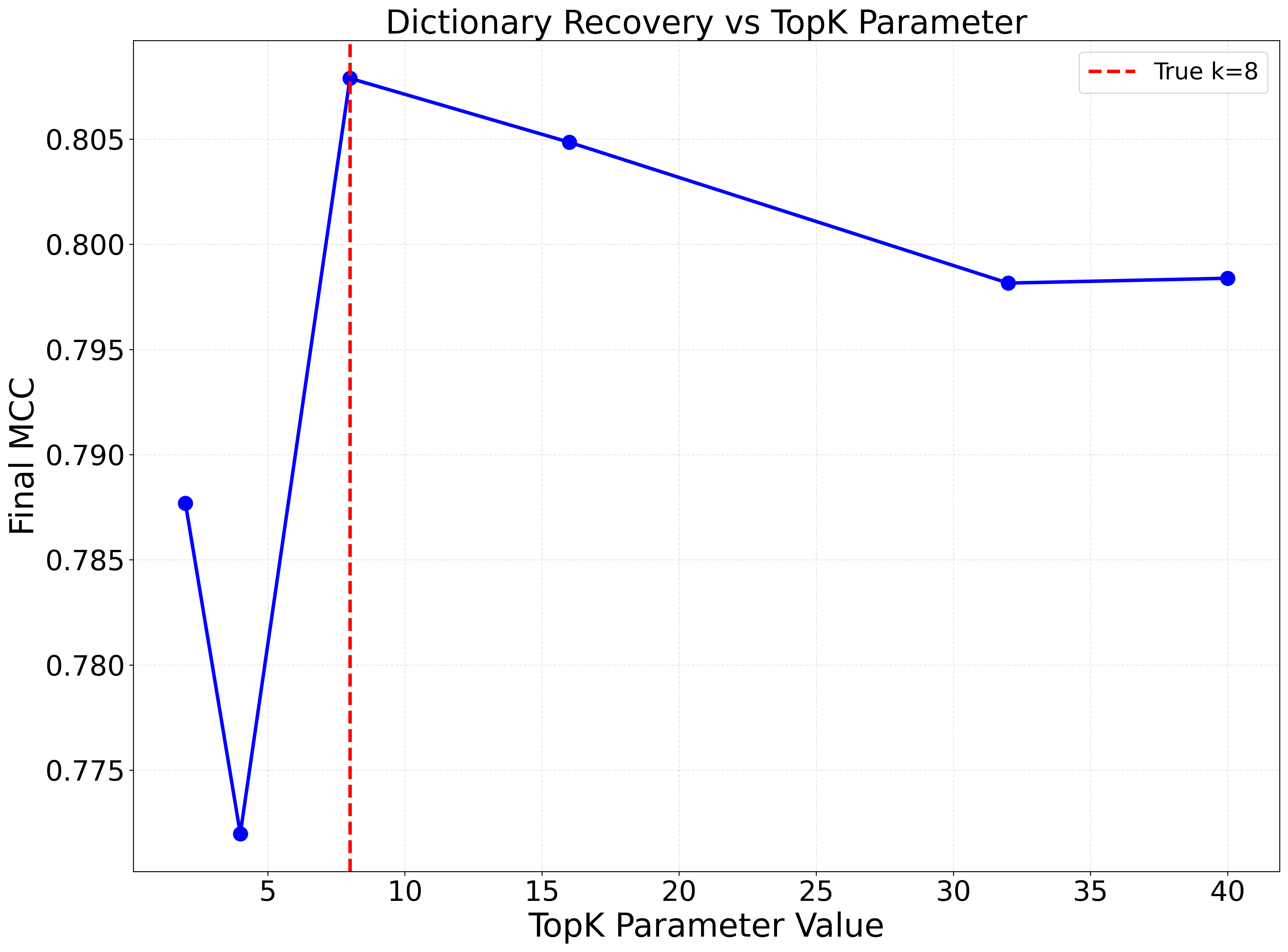}
\caption{Effect of Activation Sparsity $k$ in TopK SAE in the Matched Regime ($d_{\text{gt}} = d_{\text{sae}} = 40$, true $s = 8$). We plot final GT-MCC (averaged over the last 100 steps) vs. $k$. Performance peaks at $k = s = 8$, with underestimating $k$ being more harmful than overestimating it.}
\label{fig:mcc_vs_k}
\end{figure}

Figure~\ref{fig:mcc_vs_k} demonstrates that GT-MCC peaks precisely at $k = s = 8$, confirming that matching encoder sparsity to true data sparsity yields optimal dictionary recovery. The performance curve exhibits notable asymmetry around the optimal point: under-sparsity ($k < s$) causes sharp degradation in GT-MCC as $k$ decreases below $s$, while over-sparsity ($k > s$) leads to more gradual decline as $k$ increases beyond $s$. This asymmetry reflects fundamental differences in how sparsity misspecification affects the optimization process.

When $k < s$, the SAE lacks sufficient representational capacity to capture all active features in the ground-truth generating process. This constraint forces the model to either merge multiple ground-truth features into single learned features or systematically ignore some ground-truth features, both resulting in poor dictionary recovery. The sharp performance degradation occurs because the model cannot represent the true complexity of the data-generating process, leading to fundamental representational limitations that cannot be overcome through better optimization.

Conversely, when $k > s$, the SAE has excess representational capacity that allows recovery of all true features but may also lead to learning near-duplicate features representing similar directions, increased sensitivity to initialization, and selection ambiguity among competing feature representations. While this reduces run-to-run consistency, the gradual performance decline suggests that over-sparsity is less immediately harmful than under-sparsity, as the model can still capture the essential structure of the data even if it learns redundant or unstable features.

These findings have important practical implications. Because practitioners never observe $A_i^\star$ in real applications, measuring PW-MCC across multiple training runs becomes an important diagnostic tool. As demonstrated earlier in this paper, PW-MCC correlates strongly with GT-MCC, especially in the matched regime, providing a practical proxy for dictionary quality when ground truth is unavailable. Our observations reveal that when $k$ is too low, the SAE lacks sufficient representational capacity, leading to poor feature recovery. When $k$ is too high, features may become unstable across runs due to selection ambiguity among near-duplicates, even when reconstruction loss appears acceptable. Therefore, sweeping over $k$ values while monitoring PW-MCC provides a principled approach to approximate effective sparsity for a given dataset and SAE architecture.

For practitioners selecting appropriate sparsity parameters, these results suggest that when uncertain about true sparsity, a conservative approach favoring slight over-sparsity ($\rho \approx 1.2\text{--}1.5$) is preferable to under-sparsity, as performance degradation is more gradual and reversible. Practitioners should systematically vary $k$ while monitoring PW-MCC to identify regions of high consistency. Warning signs include low PW-MCC coupled with acceptable reconstruction loss, which may indicate over-sparsity leading to feature instability, and poor reconstruction performance combined with high PW-MCC, which may indicate under-sparsity with consistent but incomplete feature recovery.

These observations raise important theoretical questions for future research. Can we precisely characterize how dictionary learning error scales with the degree of sparsity misspecification under different conditions, potentially leading to theoretical bounds on performance degradation? What is the exact relationship between mutual coherence, sparsity ratio, and feature stability? How do these effects manifest in real LLM representations, where underlying sparsity may vary across different data regimes and semantic contexts? Can we develop methods that adaptively determine $k$ or accommodate varying sparsities within datasets to better capture heterogeneous features? 

A theoretical analysis of these questions, combined with further empirical investigation, could substantially advance our understanding of SAE training dynamics and improve feature extraction for mechanistic interpretability.

\section{Supplementary Analysis of SAEs Trained on Language Model Activations}
\label{app:real-world}

This appendix provides additional experimental details, visualizations, and quantitative results for SAEs trained on activations extracted from LLMs. The primary objective of these experiments is to empirically evaluate the feature consistency and learned characteristics of different SAE architectures when applied to complex, real-world data. In the absence of ground-truth features for LLM activations, feature consistency is primarily assessed by training multiple instances of each SAE configuration (differing only by random initialization seeds) and subsequently quantifying the similarity between their learned dictionaries using PW-MCC.

\subsection{Experimental Setup for Real Data Analysis}

\subsubsection{General Methodology and Data}

SAEs of various architectures were trained using activations derived from LLMs processing text from the \texttt{monology/pile-uncopyrighted} dataset. For each specific SAE architecture and hyperparameter set, three independent training runs were conducted using distinct random seeds (\texttt{random\_seeds = [42, 43, 44]}) to evaluate consistency. The learned dictionaries from pairs of these runs were compared by first finding an optimal one-to-one matching between their features using the Hungarian algorithm, and then calculating the cosine similarity for each matched pair. The average of these similarities constitutes the PW-MCC for the dictionary. Separately, we plot the individual feature similarity and its correlation with feature activation statistics.

All SAEs were trained on 500 million tokens from the source dataset. Common training parameters, consistent across these experiments unless otherwise specified, are detailed in Table~\ref{tab:common_training_params_real}. Feature activation frequency for a given SAE feature is defined as the proportion of input tokens (within a representative sample of the training data) on which that feature exhibits a non-zero activation. When comparing a matched pair of features from two independently trained SAEs (run 1 and run 2), their joint activation behavior is characterized by \texttt{min(freq\_run1, freq\_run2)}. This metric provides a conservative estimate of their shared activity level, as a feature pair representing a truly consistent underlying concept should ideally be active on a substantial and largely overlapping set of inputs.

\begin{table}[htbp]
\centering
\small
\caption{Common Training Parameters for SAEs on LLM Activations.}
\label{tab:common_training_params_real}
\begin{tabular}{@{}ll@{}}
\toprule
Parameter & Value \\
\midrule
Dataset for Activations & \texttt{monology/pile-uncopyrighted} \\
Total Training Tokens & $5 \times 10^8$ \\
SAE Batch Size  & 2048 \\
Warmup Steps  & 1000 \\
Sparsity Warmup Steps & 5000 (for L1-based and JumpReLU ) \\
Learning Rate Decay Start  & 0.8 (of total training steps) \\
Number of Random Seeds per Configuration & 3 (42, 43, 44) \\
Activation Normalization & Applied (before SAE input) \\
Autocast Data Type & \texttt{torch.bfloat16} \\
\bottomrule
\end{tabular}
\end{table}

\subsubsection{Configuration for Pythia-160M Experiments}

The detailed visualizations and quantitative comparisons presented in Section~\ref{sec:pythia_results} pertain to SAEs trained on activations extracted from the \texttt{EleutherAI/pythia-160m-deduped} model. These SAEs were trained on the residual stream activations output by layer 8 of the LLM (\texttt{resid\_post\_layer\_8}). Specific parameters for this experimental setup are provided in Table~\ref{tab:pythia160m_params} and inspired by the sweep used in \citep{karvonen2024saebench}.

For the specific configurations analyzed from this Pythia-160M setup, the SAE architectures demonstrated varying levels of pairwise dictionary consistency as measured by PW-MCC. TopK SAEs achieved the highest overall performance with a PW-MCC of 0.8181 using a target $k$ of 20. Batch TopK SAEs followed with a PW-MCC of 0.7656, also at target $k$ of 20. Gated SAEs produced a PW-MCC of 0.7370 with a sparsity penalty of 0.06. Among the remaining architectures, Matryoshka Batch TopK SAEs achieved 0.6267 at target $k$ of 20, P-Anneal SAEs reached 0.6113 with an initial sparsity penalty of 0.025, JumpReLU SAEs attained 0.4947 at target $L_0$ of 40, and Standard SAEs achieved 0.4739 with a sparsity penalty of 0.06. All optimal results were observed at training step 244,140. In terms of feature utilization, TopK and Gated SAEs consistently produced fewer \emph{dead} features (features with very low or zero activation rates across the evaluation data) compared to the Standard and JumpReLU variants, with the $L_0$-constrained architectures (TopK, Batch TopK, JumpReLU, and Matryoshka Batch TopK) demonstrating more controlled sparsity patterns than their $L_1$-penalized counterparts.

\begin{table}[htbp]
\centering
\small
\caption{SAE Training Parameters for EleutherAI/pythia-160m-deduped (Layer 8).}
\label{tab:pythia160m_params}
\begin{tabular}{@{}ll@{}}
\toprule
Parameter & Value \\
\midrule
LLM Model & \texttt{EleutherAI/pythia-160m-deduped} \\
Targeted Layer & 8 (residual stream output) \\
LLM Activation Dimension (m) & 768 \\
LLM Batch Size (for activation generation) & 32 \\
LLM Context Length & 1024 \\
LLM Data Type & \texttt{torch.float32} \\
SAE Dictionary Width  & $2^{14}$ (16,384) \\
SAE Learning Rate & $3 \times 10^{-4}$ \\
Architectures Evaluated & Standard, TopK, BatchTopK, Gated, \\
 & P-Anneal, JumpReLU, Matryoshka BatchTopK \\
Sparsity Penalties Sweep (for L1-based architectures): & \\
\quad Standard & [0.012, 0.015, 0.02, 0.03, 0.04, 0.06] \\
\quad P-Anneal (initial penalty) & [0.006, 0.008, 0.01, 0.015, 0.02, 0.025] \\
\quad Gated & [0.012, 0.018, 0.024, 0.04, 0.06, 0.08] \\
Target L0s / $k$ (for L0-based architectures): & [20, 40, 80, 160, 320, 640] \\
\bottomrule
\end{tabular}
\end{table}

\subsubsection{Configuration for Gemma-2-2B Experiments}

Additional experiments were conducted using activations from the \texttt{google/gemma-2-2B} model \citep{team2024gemma}, targeting the residual stream output of layer 12 (\texttt{resid\_post\_layer\_12}). Table~\ref{tab:gemma2b_params} documents the pertinent hyperparameters for these larger-scale runs.

For the Gemma-2-2B experimental configuration, the SAE architectures exhibited similar performance characteristics in terms of PW-MCC. TopK SAEs achieved the highest pairwise dictionary consistency with a PW-MCC of 0.7898 using a target $k$ of 80. JumpReLU SAEs demonstrated competitive performance with a PW-MCC of 0.7405 at target $k$ of 40, closely followed by Batch TopK SAEs with a PW-MCC of 0.7403 at target $k$ of 80. Gated SAEs produced a PW-MCC of 0.7033 with a sparsity penalty of 0.04. The remaining architectures showed more modest performance levels, with Matryoshka Batch TopK SAEs achieving 0.5842 at target $k$ of 80, P-Anneal SAEs reaching 0.5731 with an initial sparsity penalty of 0.025, and Standard SAEs attaining 0.5717 with a sparsity penalty of 0.03. The performance patterns observed in the larger Gemma-2-2B model generally maintained the relative ordering established in the Pythia-160M experiments, with $L_0$-constrained architectures consistently outperforming their $L_1$-penalized counterparts in terms of dictionary consistency metrics.
\begin{table}[htbp]
\centering
\small
\caption{SAE Training Parameters for google/gemma-2-2B (Layer 12).}
\label{tab:gemma2b_params}
\begin{tabular}{@{}ll@{}}
\toprule
Parameter & Value \\
\midrule
LLM Model & \texttt{google/gemma-2-2B} \\
Targeted Layer & 12 (residual stream output) \\
LLM Activation Dimension (m) & 2304 \\
LLM Batch Size (for activation generation) & 4 \\
LLM Context Length & 1024 \\
LLM Data Type & \texttt{torch.bfloat16} \\
SAE Dictionary Width  & $2^{16}$ (65,536) \\
SAE Learning Rate & $3 \times 10^{-4}$ \\
Random Seeds Used & 42, 43 \\
Architectures Evaluated & Standard, TopK, BatchTopK, Gated, \\
 & P-Anneal, JumpReLU, Matryoshka BatchTopK \\
& Standard \\
Sparsity Penalties Sweep (for L1-based architectures): & \\
\quad Standard & [0.012, 0.015, 0.02, 0.03, 0.04, 0.06] \\
\quad P-Anneal (initial penalty) & [0.006, 0.008, 0.01, 0.015, 0.02, 0.025] \\
\quad Gated & [0.012, 0.018, 0.024, 0.04, 0.06, 0.08] \\
Target L0s / $k$ (for L0-based architectures): & [20, 40, 80, 160, 320, 640] \\
\bottomrule
\end{tabular}
\begin{tabular}{p{0.95\linewidth}}
\end{tabular}
\end{table}

\subsection{Detailed Analysis of SAEs Trained on Pythia-160M Layer 8 Activations}
\label{sec:pythia_results}

This section presents a comparative analysis of feature consistency and activation patterns for four different SAE architectures (TopK, Gated, Standard, JumpReLU) trained on Pythia-160M layer 8 activations.

\subsubsection{Cross-Architectural Comparison of Feature Similarity and Activation Patterns}

\begin{figure}[htbp]
    \centering
    \begin{minipage}{0.48\textwidth}
        \centering
        \includegraphics[width=\textwidth]{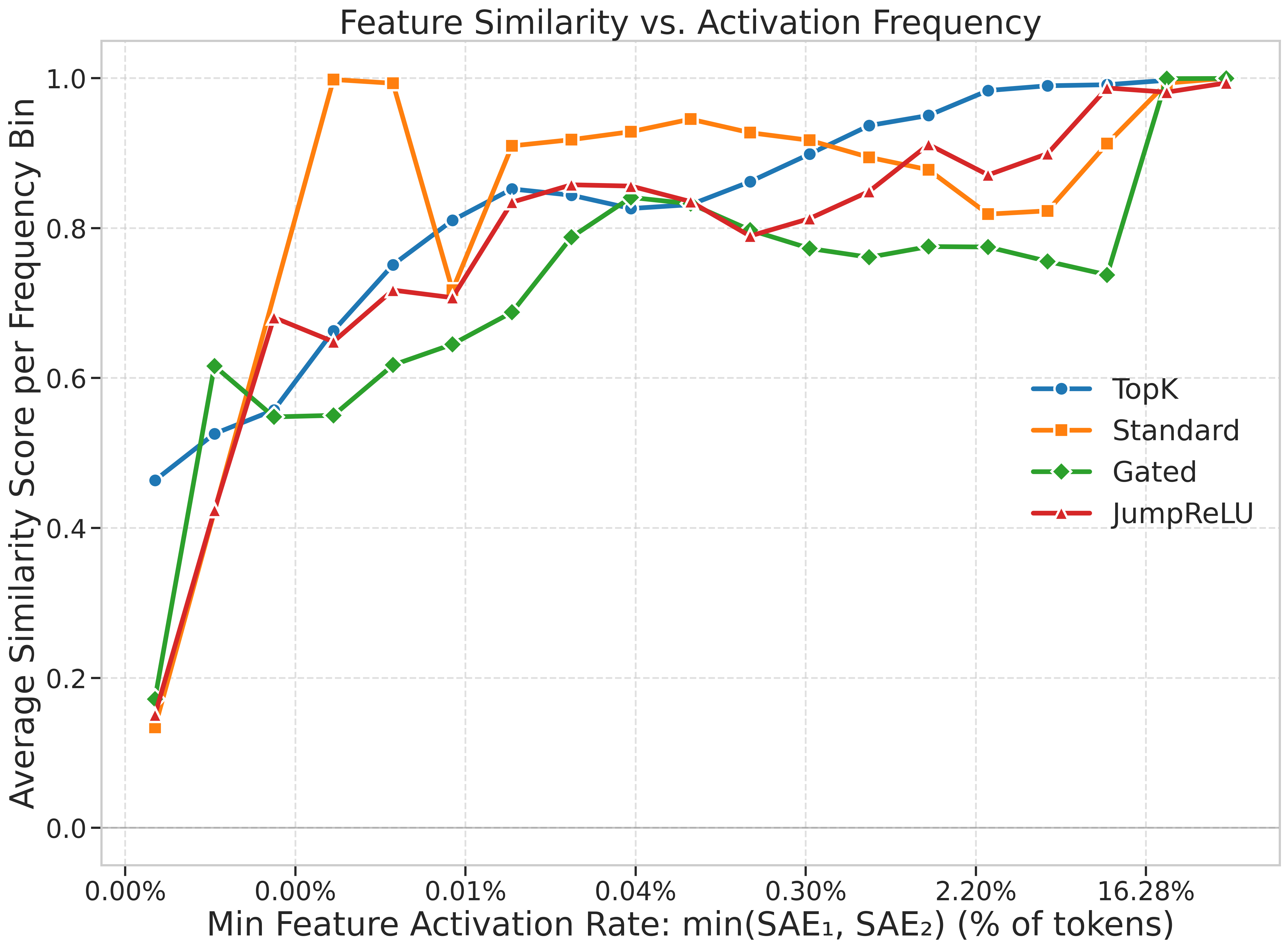}
        \caption{Average feature similarity (PW-MCC of matched individual features) versus activation rate for four SAE architectures. The activation rate is defined as \texttt{min(freq\_run1, freq\_run2)}, representing the minimum percentage of tokens activating the feature across two independent runs. Data is from SAEs trained on Pythia-160M layer 8 activations. A strong positive correlation is evident, indicating that features with higher shared activation rates are learned more consistently.}
        \label{fig:avg_similarity}
    \end{minipage}
    \hfill
    \begin{minipage}{0.48\textwidth}
        \centering
        \includegraphics[width=\textwidth]{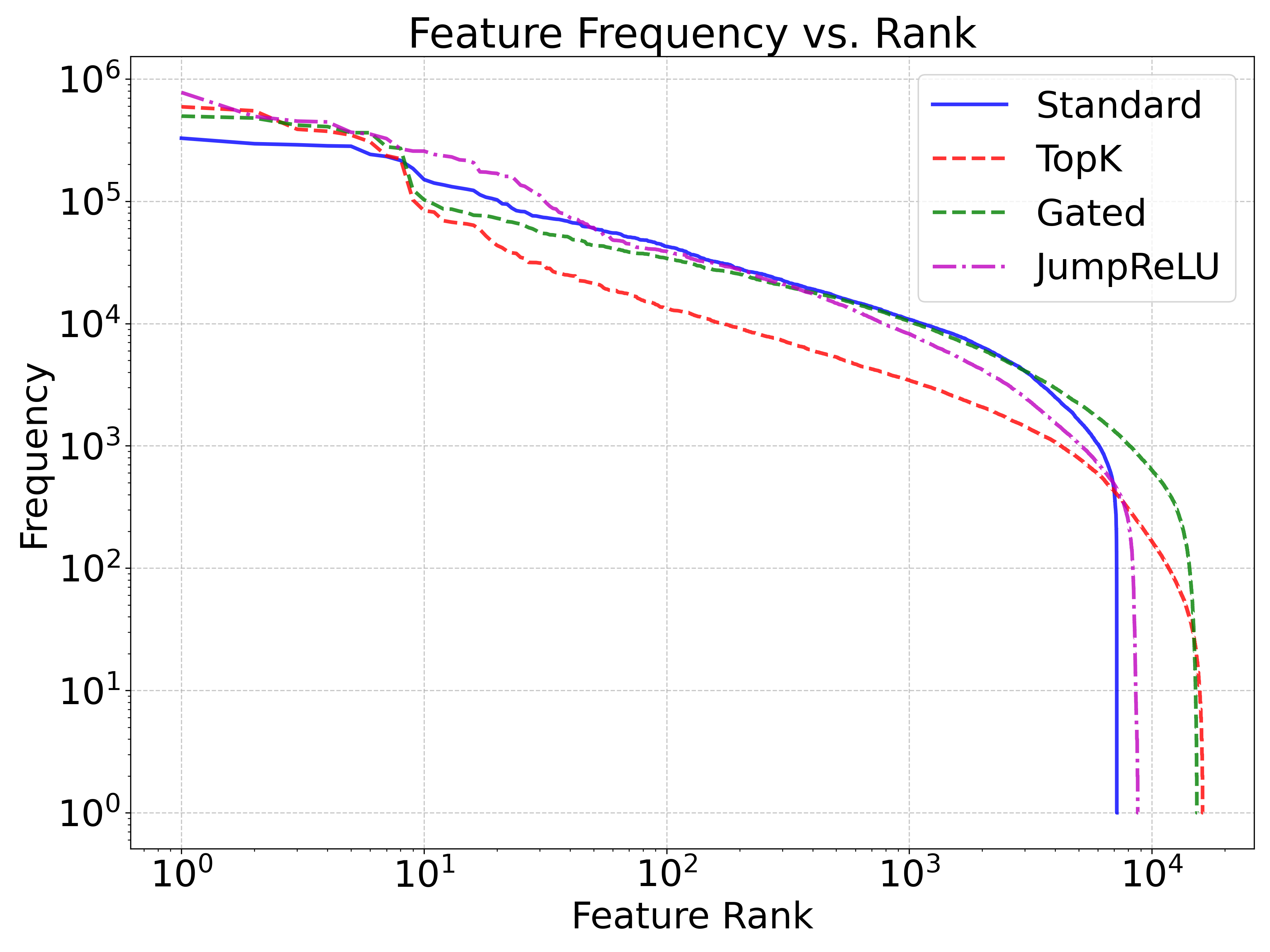}
        \caption{Log-log plot of feature activation frequency versus feature rank for four SAE architectures trained on Pythia-160M layer 8 activations. All architectures exhibit power-law-like distributions of feature usage, but with distinct slopes and differing prevalence of very low-frequency (potentially dead) features. Notably, TopK and Gated architectures tend to show fewer dead features.}
        \label{fig:activation_frequency}
    \end{minipage}
\end{figure}

A general trend observed across SAE architectures is the positive correlation between a feature's activation frequency and its consistency across independent training runs. Figure~\ref{fig:avg_similarity} illustrates this for Standard, TopK, Gated, and JumpReLU SAEs. The x-axis, \texttt{min(freq\_run1, freq\_run2)}, represents a condition for shared activity; features satisfying this with higher values (i.e., both are frequently active) exhibit higher pairwise similarity. This suggests that features corresponding to more prevalent patterns in the LLM's activations are more robustly learned. While this correlation is seen in all SAEs, the overall level of consistency varies, with TopK SAEs achieving the highest aggregate PW-MCC.

Concurrently, Figure~\ref{fig:activation_frequency} reveals that different SAE architectures induce distinct feature utilization profiles. While all exhibit power-law-like distributions for feature activation frequencies (when features are ranked by frequency), the slopes of these distributions and the number of extremely low-frequency or dead features vary. Architectures such as TopK and Gated SAEs tend to result in fewer dead features compared to Standard and JumpReLU SAEs, indicating potentially more efficient use of their learned dictionaries.

\subsubsection{Standard SAE}

\begin{figure}[htb]
    \centering
    \includegraphics[width=0.7\textwidth]{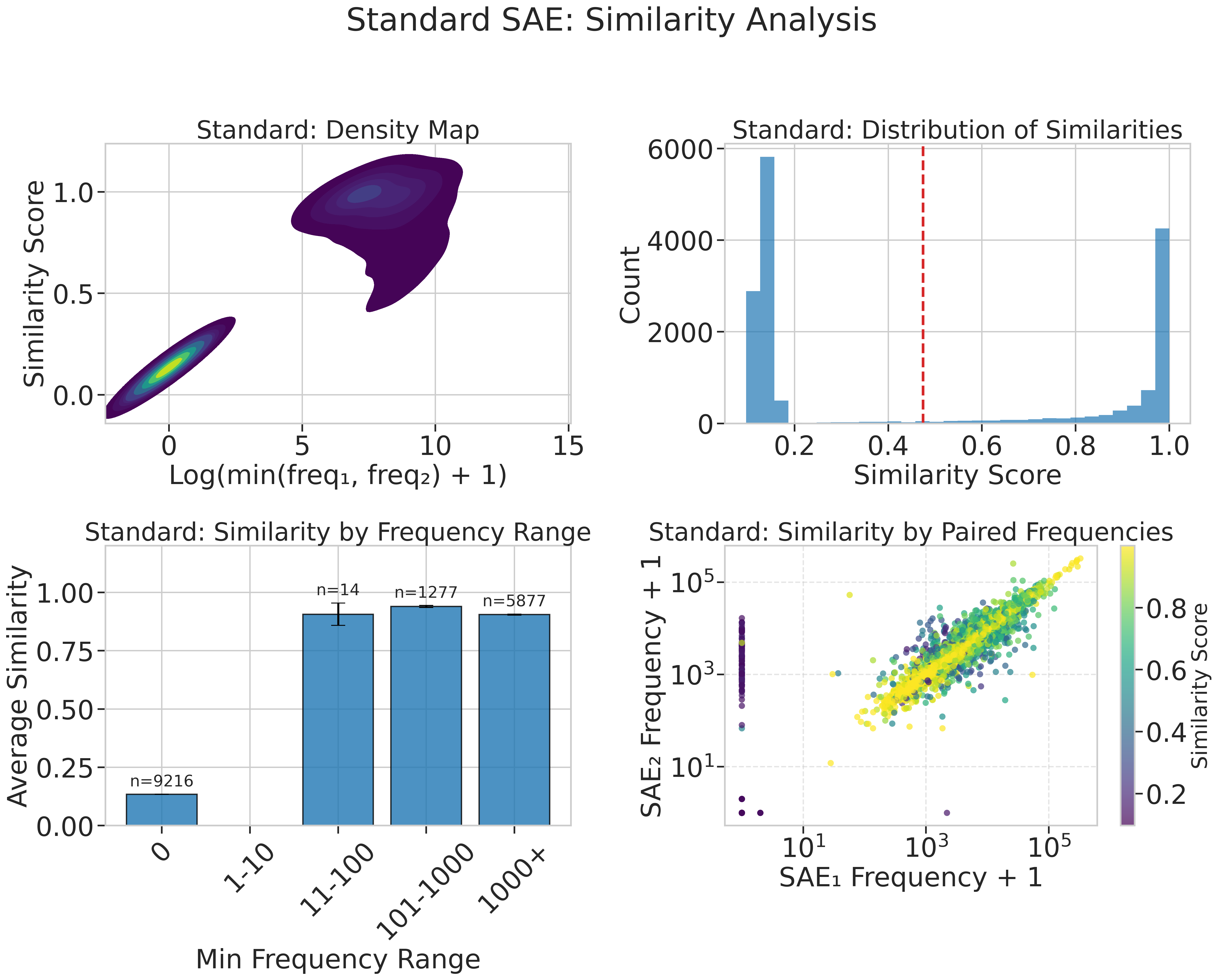}
    \caption{Feature similarity analysis for a Standard SAE (L1-penalized) trained on Pythia-160M layer 8 activations. The overall PW-MCC for this configuration was approximately $0.4739$. Top left: Density map of pairwise feature similarity vs. log minimum activation frequency (\texttt{min(freq\_run1, freq\_run2)}). Top right: Histogram of pairwise feature similarity scores. Bottom left: Average similarity bucketed by log minimum frequency range. Bottom right: Scatter plot of feature activation frequencies from two runs (\texttt{freq\_run1} vs. \texttt{freq\_run2}), colored by their pairwise similarity score. Points along the diagonal represent features with similar activation frequencies in both runs.}
    \label{fig:similarity_analysis_standard}
\end{figure}

\begin{figure}[htb]
    \centering
    \includegraphics[width=0.5\textwidth]{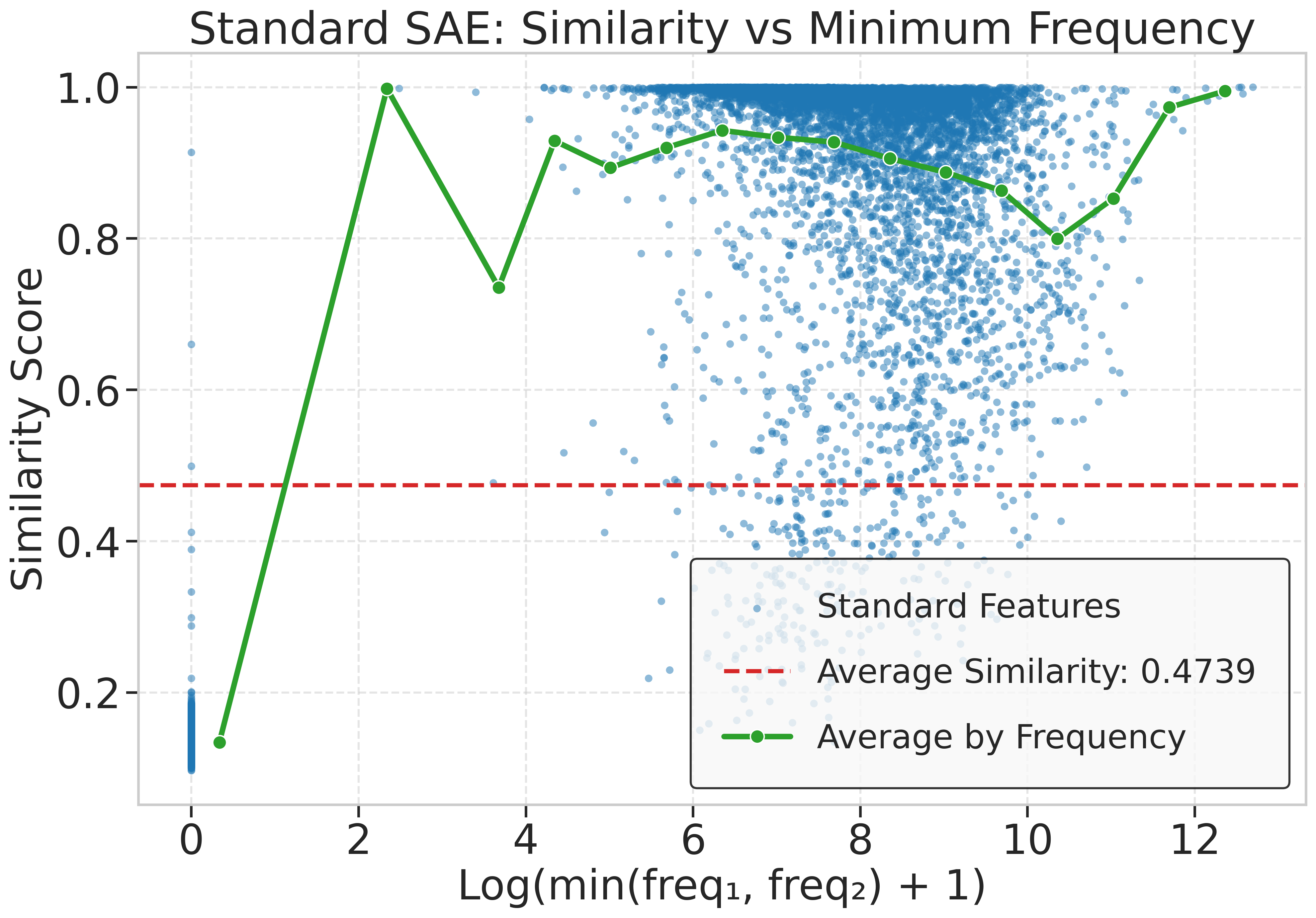}
    \caption{Average feature similarity versus \texttt{min(freq\_run1, freq\_run2)} for the Standard SAE. This plot highlights the positive trend: features with higher shared activation levels tend to exhibit greater pairwise similarity, though the overall consistency for this architecture is modest.}
    \label{fig:similarity_vs_freq_standard}
\end{figure}

The Standard SAE architecture, characterized by an L1 penalty on activations, achieved a relatively low overall PW-MCC of approximately $0.4739$ in these experiments. A detailed breakdown of its feature similarity characteristics is provided in Figure~\ref{fig:similarity_analysis_standard}. The density map (top-left) and the bucketed average similarity (bottom-left) illustrate that while higher minimum activation frequencies correlate with improved similarity, a significant portion of features exhibit low similarity. The histogram of similarity scores (top-right) is also peaked at zero, indicating several dead features. The scatter plot of paired frequencies (bottom-right) shows that features matched by the Hungarian algorithm (and thus contributing to the similarity score) often, but not always, possess similar activation frequencies across the two runs. Features lying closer to the diagonal (similar frequencies in both runs) and having higher joint frequencies (top-right of this sub-panel) tend to exhibit higher similarity.
Figure~\ref{fig:similarity_vs_freq_standard} isolates and clearly depicts the positive relationship between shared activation frequency and pairwise similarity for this architecture.

\subsubsection{TopK SAE}

\begin{figure}[htb]
    \centering
    \includegraphics[width=0.7\textwidth]{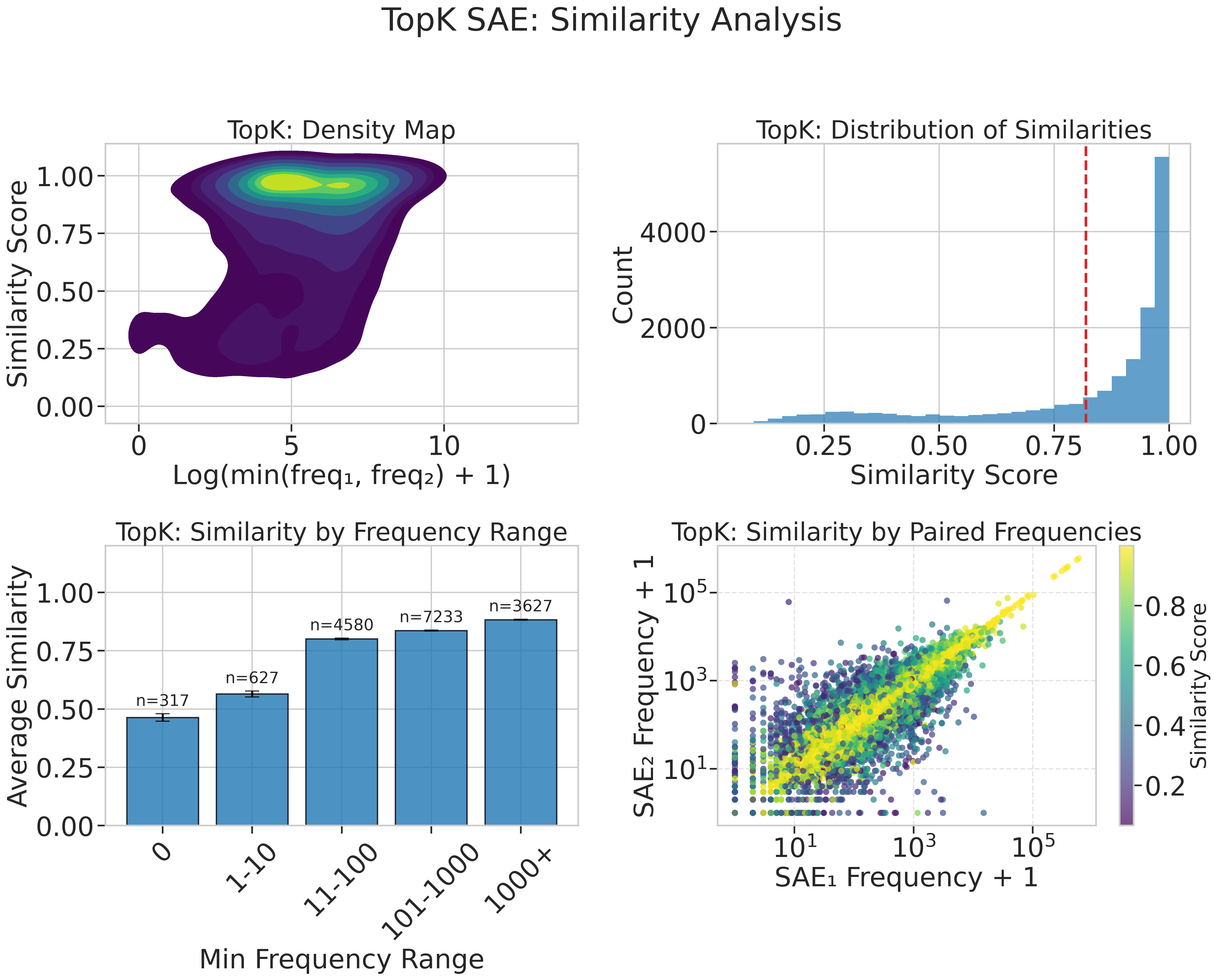}
    \caption{Feature similarity analysis for a TopK SAE trained on Pythia-160M layer 8 activations. This architecture achieved a high overall PW-MCC of approximately $0.8188$. Panels are analogous to Figure~\ref{fig:similarity_analysis_standard}.}
    \label{fig:similarity_analysis_topk}
\end{figure}

\begin{figure}[htb]
    \centering
    \includegraphics[width=0.5\textwidth]{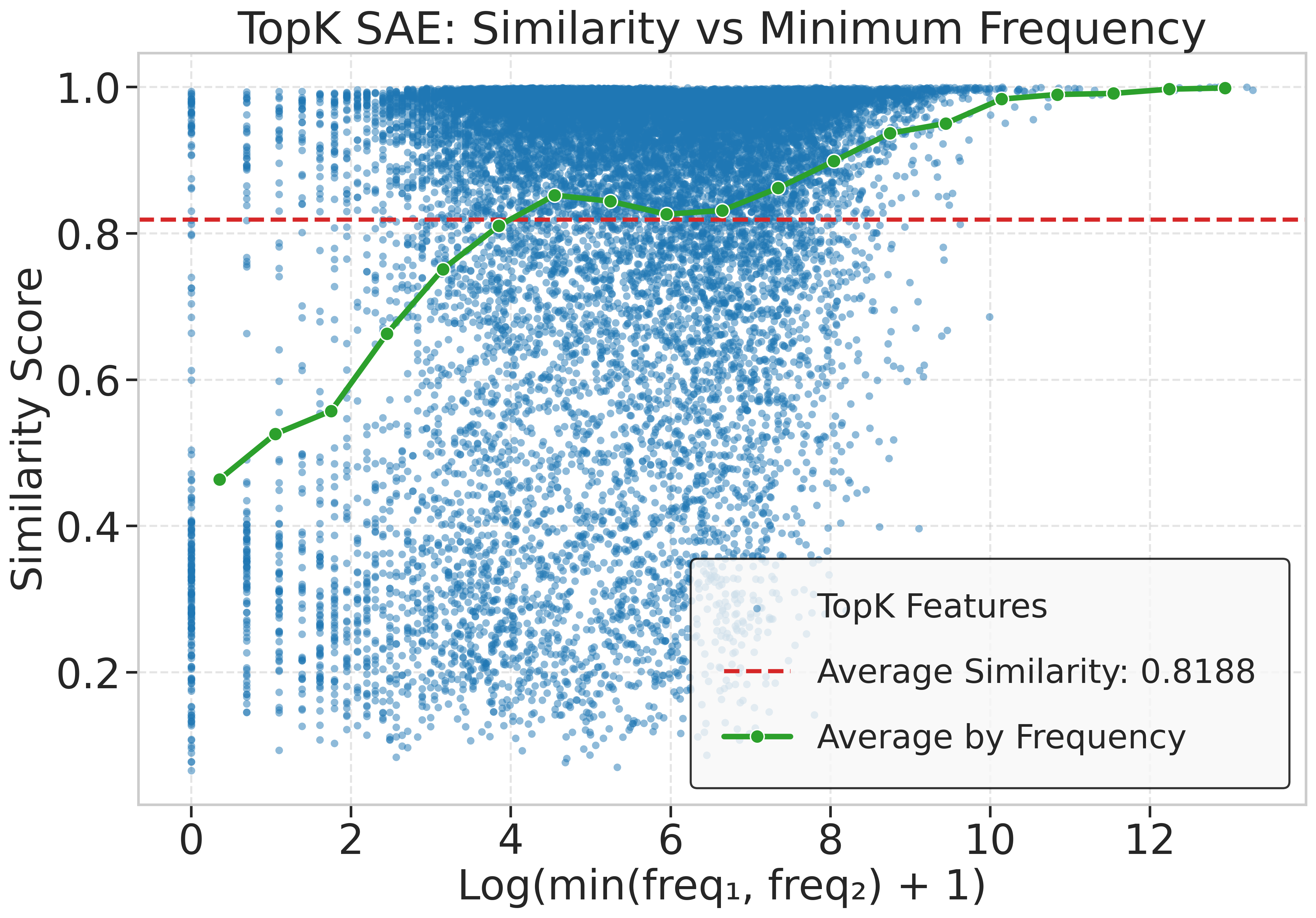}
    \caption{Average feature similarity versus \texttt{min(freq\_run1, freq\_run2)} for the TopK SAE. This architecture demonstrates both high overall similarity levels and a strong positive correlation between shared activation frequency and feature reproducibility.}
    \label{fig:similarity_vs_freq_topk}
\end{figure}

In contrast, the TopK SAE architecture demonstrated superior performance, achieving the highest overall PW-MCC of approximately $0.8188$. The analyses in Figure~\ref{fig:similarity_analysis_topk} and Figure~\ref{fig:similarity_vs_freq_topk} illustrate this. The density map in Figure~\ref{fig:similarity_analysis_topk} (top-left) shows a strong concentration of features in the high-similarity, high-shared-frequency region. The histogram of similarity scores (top-right) is markedly skewed towards higher values compared to the Standard SAE, indicating more consistently learned features across the dictionary. The scatter plot of paired frequencies (bottom-right) again suggests that matched features tend to have similar activation rates especially at higher frequencies. Figure~\ref{fig:similarity_vs_freq_topk} clearly shows a robust positive correlation between shared activation frequency and high pairwise similarity, underscoring the stability of frequently used features learned by TopK SAEs.

\subsubsection{Gated SAE}

\begin{figure}[htbp]
    \centering
    \includegraphics[width=0.7\textwidth]{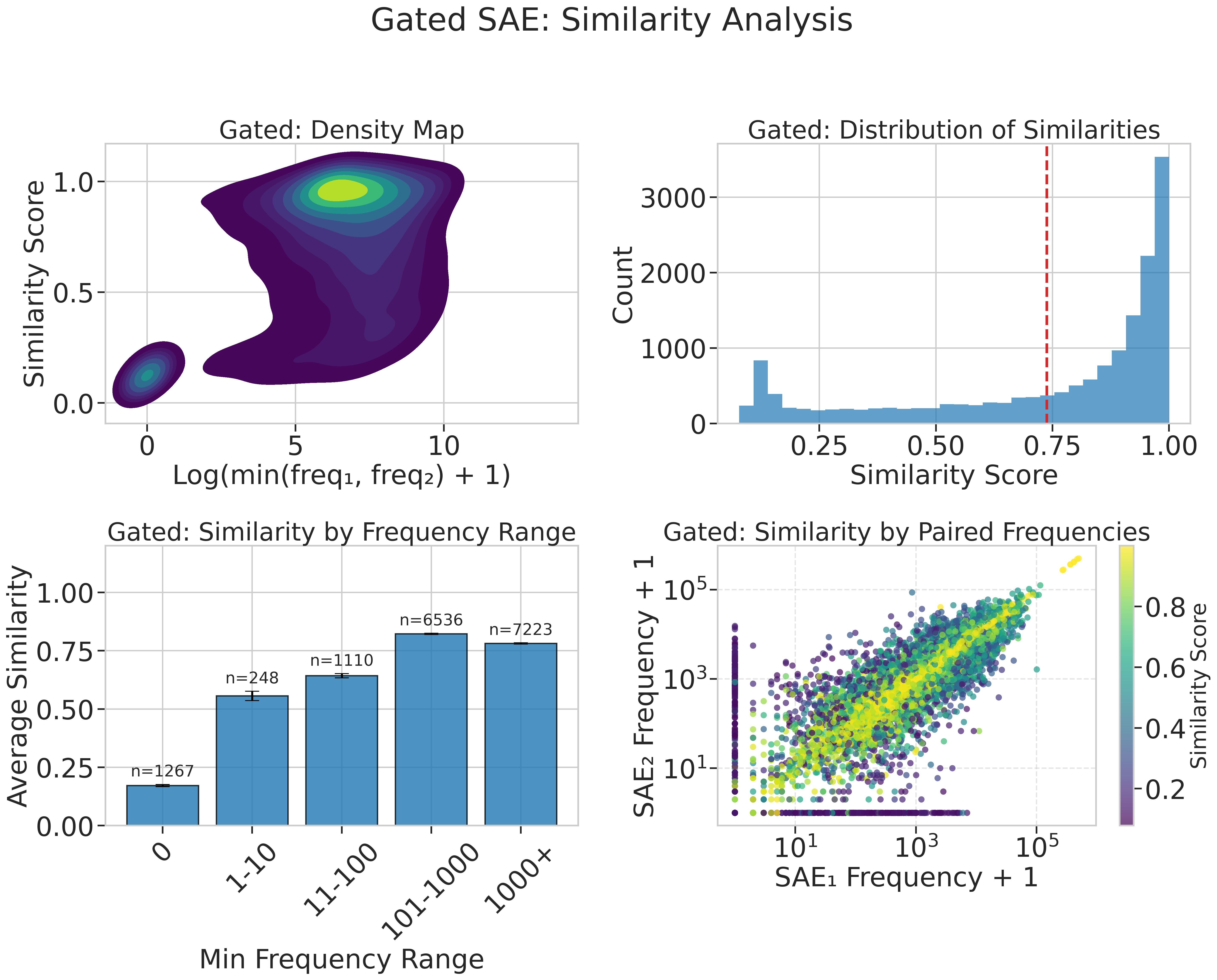}
    \caption{Feature similarity analysis for a Gated SAE trained on Pythia-160M layer 8 activations, with an overall PW-MCC of approximately $0.7378$. Panels are analogous to Figure~\ref{fig:similarity_analysis_standard}.}
    \label{fig:similarity_analysis_gated}
\end{figure}

\begin{figure}[htbp]
    \centering
    \includegraphics[width=0.5\textwidth]{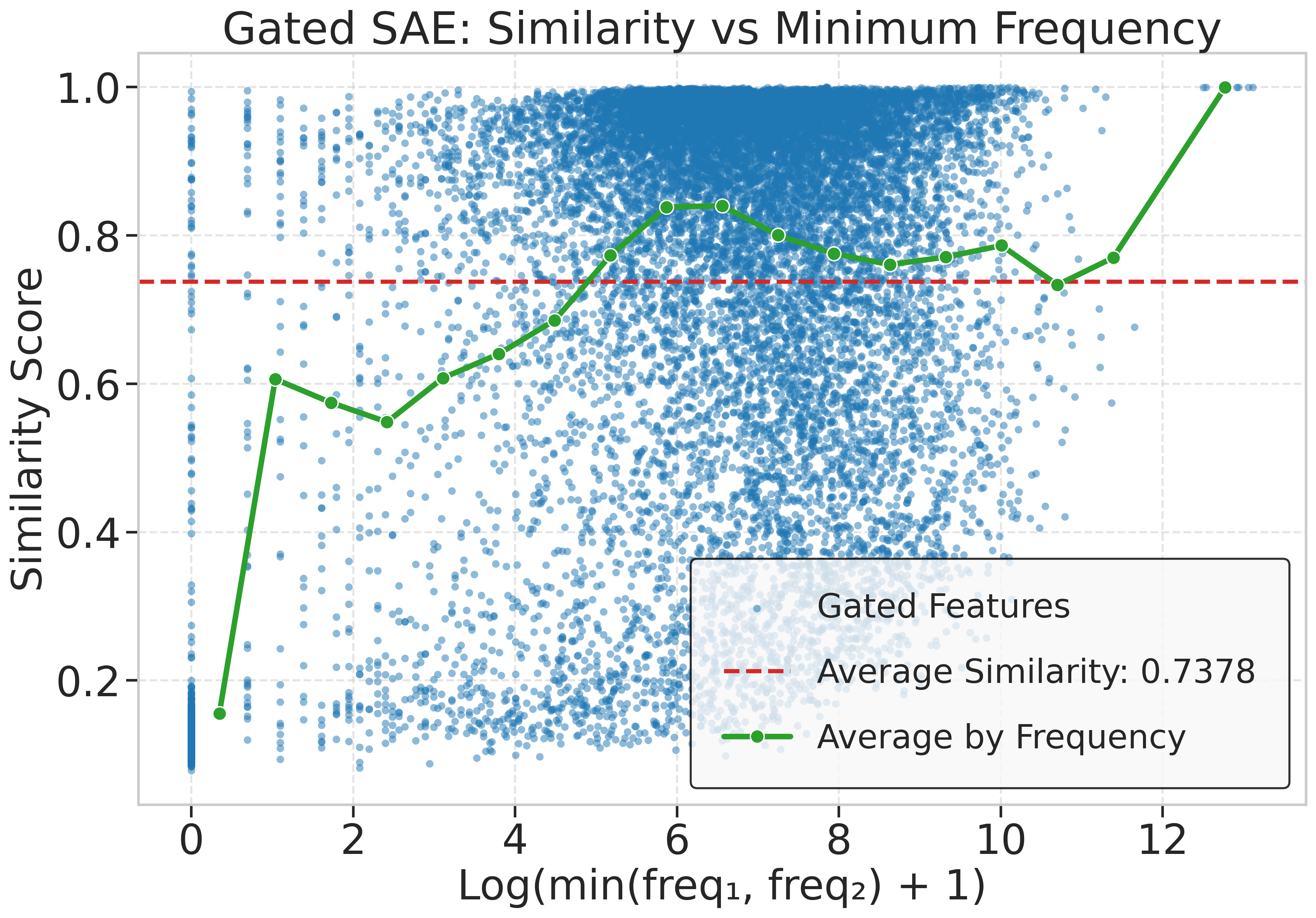}
    \caption{Average feature similarity versus \texttt{min(freq\_run1, freq\_run2)} for the Gated SAE. The overall dictionary PW-MCC for this configuration is $0.7378$. A strong positive correlation is evident between shared activation frequency and individual feature similarity.}
    \label{fig:similarity_vs_freq_gated}
\end{figure}

The Gated SAE architecture yielded an overall PW-MCC of approximately $0.7378$, exhibiting good consistency, second only to TopK SAEs in this comparison. Figure~\ref{fig:similarity_analysis_gated} and Figure~\ref{fig:similarity_vs_freq_gated} detail its characteristics. The four-panel plot in Figure~\ref{fig:similarity_analysis_gated} displays trends consistent with other architectures regarding the relationship between shared frequency and similarity. The distribution of similarity scores (top-right) is more favorable than that of Standard SAEs, leaning towards higher consistency. The paired frequency scatter plot (bottom-right) also indicates that matched features tend to share similar activation levels. Figure~\ref{fig:similarity_vs_freq_gated} confirms the strong positive correlation between \texttt{min(freq\_run1, freq\_run2)} and feature similarity.

\subsubsection{JumpReLU SAE}

\begin{figure}[htbp]
    \centering
    \includegraphics[width=0.7\textwidth]{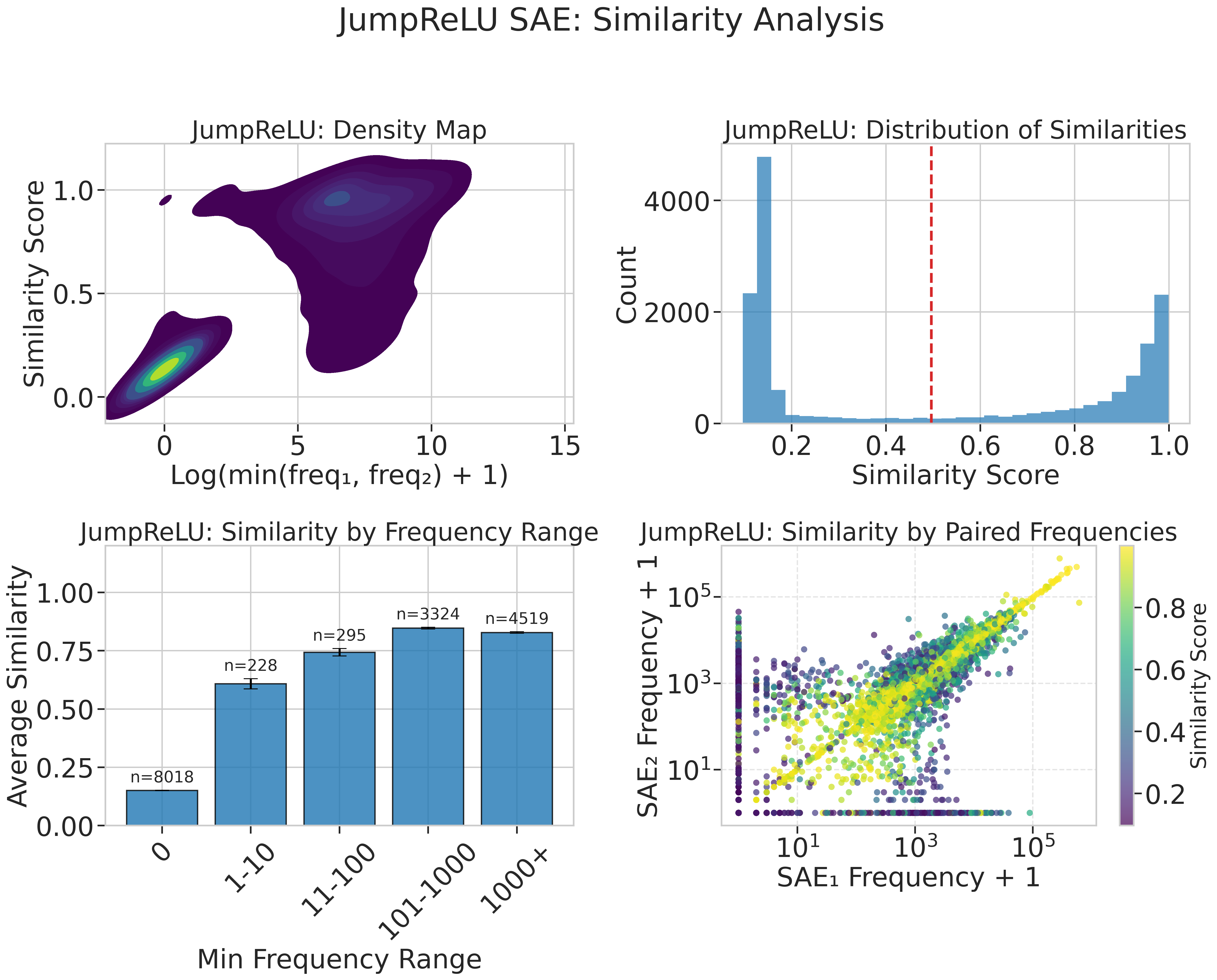}
    \caption{Feature similarity analysis for a JumpReLU SAE trained on Pythia-160M layer 8 activations. The overall PW-MCC was approximately $0.4957$. Panels are analogous to Figure~\ref{fig:similarity_analysis_standard}.}
    \label{fig:similarity_analysis_jumprelu}
\end{figure}

\begin{figure}[htbp]
    \centering
    \includegraphics[width=0.5\textwidth]{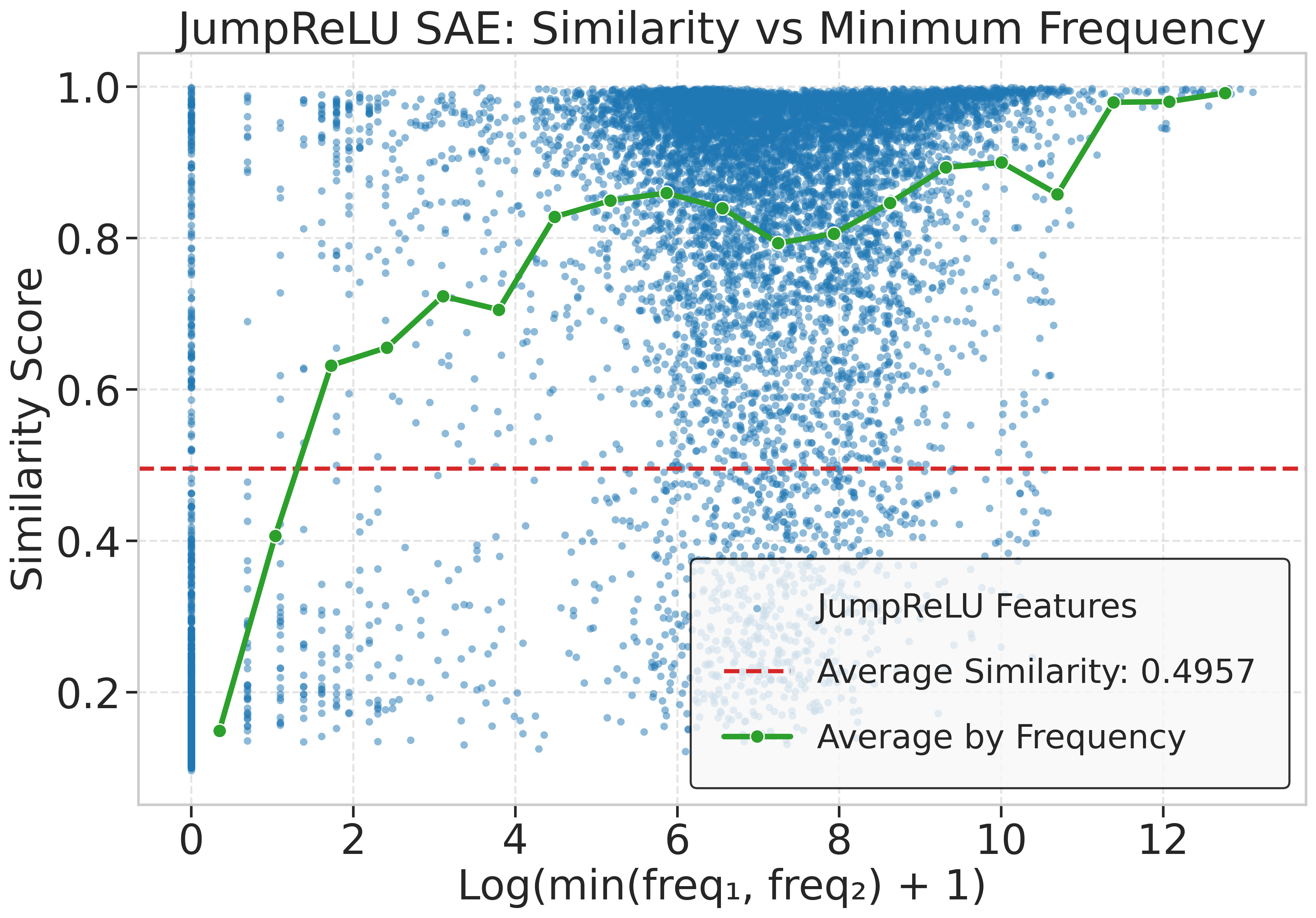}
    \caption{Average feature similarity versus \texttt{min(freq\_run1, freq\_run2)} for the JumpReLU SAE. This plot shows increasing similarity with higher shared activation frequency. The presence of two distinct clusters suggests potential subpopulations of features with differing learning or consistency characteristics within this architecture.}
    \label{fig:similarity_vs_freq_jumprelu}
\end{figure}

The JumpReLU SAE achieved an overall PW-MCC of approximately $0.4957$, placing its aggregate consistency on par with Standard SAEs in these experiments. The results are presented in Figure~\ref{fig:similarity_analysis_jumprelu} and Figure~\ref{fig:similarity_vs_freq_jumprelu}. The general trend of higher similarity for more frequently and jointly active features persists. The scatter plot of paired frequencies in Figure~\ref{fig:similarity_analysis_jumprelu} (bottom-right) also suggests that features matched by the Hungarian algorithm tend to possess similar activation frequencies across runs. Figure~\ref{fig:similarity_vs_freq_jumprelu} shows that JumpReLU also produces several dead features that affect the net consistency.

\subsubsection{Summary of Real Data Findings}

The empirical investigations on Pythia-160M activations consistently reveal a significant positive correlation between the activation frequency of learned features—particularly their shared activation level across independent runs, \texttt{min(freq\_run1, freq\_run2)}—and their inter-run similarity or consistency. This observation holds across diverse SAE architectures. The use of \texttt{min(freq\_run1, freq\_run2)} as a metric for joint activity is justified by the intuition that a feature representing a stable, underlying concept should be consistently activated by a similar set of inputs across different model initializations, thus implying that both individual frequencies should be high for a robust match. Furthermore, the scatter plots of paired feature frequencies (e.g., bottom-right panels in the four-panel figures) visually corroborate that features deemed ``the same'' by the Hungarian matching process (and thus contributing to similarity scores) indeed tend to exhibit comparable activation frequencies in the respective runs.

Among the architectures quantitatively compared, TopK SAEs demonstrated the highest overall dictionary consistency as measured by PW-MCC, followed in order by Gated, JumpReLU, and Standard SAEs. This ranking correlates with qualitative observations regarding feature utilization; TopK and Gated SAEs also tended to produce fewer ``dead'' or very sparsely used features, suggesting a more effective and stable learning dynamic that leverages a greater portion of their dictionary capacity.

While the frequency-consistency trend is a common thread, different SAE architectures clearly lead to varying aggregate levels of feature consistency and distinct feature activation profiles. The detailed multi-panel visualizations offer a more granular understanding of consistency than a single global PW-MCC score, by illustrating how stability is distributed across features with different activation characteristics within each architectural type. These findings highlight the interplay of architectural choice and feature activation statistics in determining the reliability and interpretability of features learned by SAEs from real-world LLM activations.

\subsection{Analysis of SAEs Trained on Gemma-2-2B Layer 12 Activations}
\label{app:sec-gemma}

This section presents results from training SAEs on activations from layer 12 of the Gemma-2-2B model \citep{team2024gemma} and analyzes the mean pairwise MCC averaged across two training seeds. Due to computational constraints, we report only the final aggregated pairwise MCC values. Computing the MCC requires solving an assignment problem with complexity \(\mathcal{O}\!\left(N^{3}\right)\), where \(N\) represents the dictionary size. This becomes particularly challenging for Gemma-2-2B, which employs a dictionary four times larger than Pythia-160M (\(65\,536\) versus \(16\,384\)).

Figure \ref{fig:gemma-2-2b} shows the final training PW-MCC of SAEs trained on 500M tokens from \texttt{monology/pile-uncopyrighted}. TopK SAE achieves the highest pairwise MCC on this larger model, corroborating our findings on Pythia-160M and supporting our theoretical analysis. Most other SAE variants exhibit performance patterns similar to those observed in Figure \ref{fig:pythia_160M}. One notable exception is JumpReLU, which previously showed relatively poor consistency on Pythia-160M but achieves substantially improved feature consistency on Gemma-2-2B. Despite this improvement, JumpReLU still does not match the consistency levels attained by TopK SAE.
JumpReLU is commonly used in extremely large-scale SAEs within the community \citep{ameisen2025circuit,lindsey2025biology}.

\begin{figure}[htbp]
    \centering
    \includegraphics[width=0.8\linewidth]{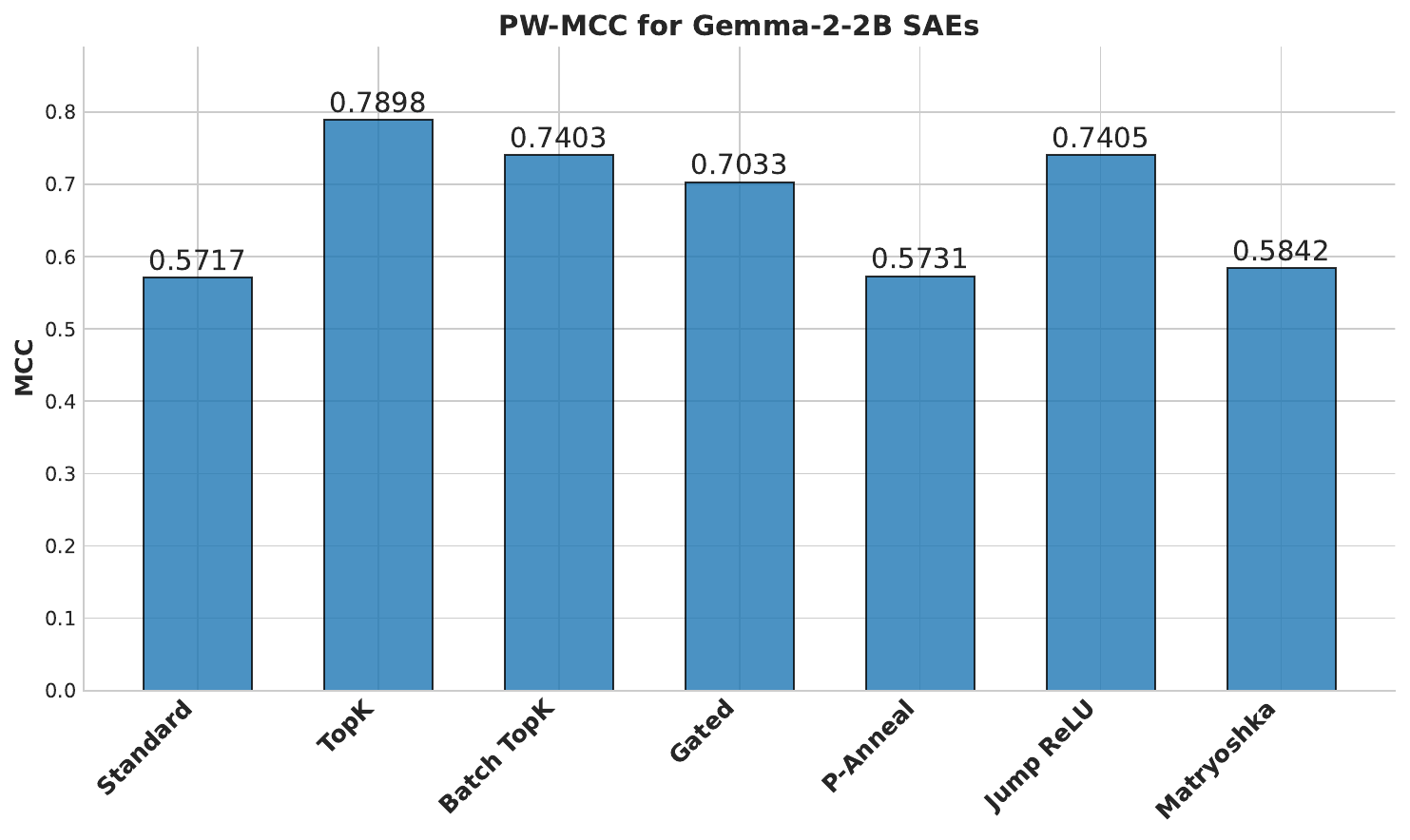}
    \caption{Final PW-MCC for BatchTopK, Gated, P-Anneal, JumpReLU, Standard, TopK, and Matryoshka BatchTopK SAEs on Gemma-2-2B activations. Higher PW-MCC indicates greater run-to-run feature consistency.}
    \label{fig:gemma-2-2b}
\end{figure}

\subsection{Qualitative Analysis: Example Feature Pairs across Similarity Buckets}
\label{app:qualitative}

To complement the quantitative assessments of feature consistency, this section provides a qualitative examination of example feature pairs extracted from 2 SAEs trained on Pythia-160M. We leverage an automated interpretation pipeline~\citep{paulo2024automatically}, to generate natural language explanations for individual SAE features and to measure the functional similarity between pairs of features from independently trained models.

The process for generating an explanation for a single SAE feature is as follows: first, activations for the feature are recorded across a substantial corpus (100,000 diverse text examples from \texttt{monology/pile-uncopyrighted} in our setup). The 10 text segments eliciting the strongest (highest magnitude) activations for that feature are then selected. These top-activating examples are formatted, the top activating tokens are emphasized (surrounded by << >>), the activation strength of the tokens is shown after each example and presented to an explainer LLM, specifically \texttt{gpt-4.1-2025-04-14}. The LLM is prompted to identify common patterns or semantic concepts within these examples and produce a concise natural language interpretation of the feature's apparent function.

For the analysis presented in \autoref{tab:feature_pairs_examples}, we first matched features between two independently trained SAEs (SAE1 and SAE2). This matching was achieved by computing the pairwise cosine similarity of all their dictionary vectors and subsequently applying the Hungarian algorithm to find an optimal one-to-one correspondence. These matched feature pairs were then categorized into five buckets based on their dictionary vector cosine similarity scores (Bucket 1: low vector similarity; Bucket 5: high vector similarity).

To measure the functional similarity of each matched pair, their independently generated explanations were provided to the same \texttt{gpt-4.1-2025-04-14} model. This model was then prompted to evaluate the semantic resemblance between the two explanations and assign a functional similarity score on a 1-10 scale. This score is presented as the \textbf{GPT-Score} score in \autoref{tab:feature_pairs_examples}, alongside the feature indices and their generated explanations.

\autoref{tab:feature_pairs_examples} shows a clear concordance between the quantitative cosine similarity of the feature dictionary vectors and the functional similarity of their roles, as determined by the automated interpretability pipeline. Feature pairs with low vector similarity (e.g., Buckets 1 and 2) typically receive divergent functional explanations and low similarity scores from GPT-4.1. For instance, one feature might be interpreted as activating on \LaTeX{} math symbols, while its low-vector-similarity counterpart from another SAE is interpreted as activating on Go/Rust code structures. Conversely, feature pairs exhibiting high vector similarity (e.g., Buckets 4 and 5) frequently obtain near-identical functional explanations and high similarity scores, such as both features being interpreted as responding to phrases indicating future events or specific Wikipedia category tags like births.

This analysis reinforces the utility of dictionary vector cosine similarity as a meaningful measure of feature consistency. The strong correlation observed suggests that high vector similarity between features learned across different runs reliably indicates the stable learning of semantically coherent and functionally equivalent interpretable units.

\begin{table}[hb]
\scriptsize
\caption{\small Representative feature pairs from two independently trained SAEs (SAE1 and SAE2) organized by similarity buckets, with corresponding GPT-evaluated functional similarity scores. Buckets represent increasing levels of feature similarity, demonstrating a strong correlation between computed similarity measures and functional equivalence. Lower-similarity pairs (Buckets 1-2) exhibit largely unrelated behavioral patterns, while higher-similarity pairs (Buckets 4-5) demonstrate nearly identical semantic functions across both SAEs.}
\label{tab:feature_pairs_examples}
\begin{tabular}{crrp{0.35\textwidth}p{0.35\textwidth}}
\toprule
\textbf{Bucket} & \textbf{SAE1, SAE2} & \textbf{GPT-Score} & \textbf{SAE1 Feature Explanation} & \textbf{SAE2 Feature Explanation} \\
\midrule
\rowcolor{bucket1}
 & 10742, 13528 & 3/10 & Activates on punctuation marks and transition words marking syntactic or discourse boundaries. & Activates on section separator "Background \{Sec1\}" in scientific writing. \\
\rowcolor{bucket1}
 & 37, 6034 & 2/10 & Activates on LaTeX/math environments and symbols within mathematical markup. & Activates at the start of code block bodies after opening braces in Go and Rust. \\
\rowcolor{bucket1}
1 & 11993, 9627 & 3/10 & Activates on bibliographic reference tokens and citation markers in academic writing. & Activates on closing parenthesis and angle bracket sequence at the end of figure captions. \\
\rowcolor{bucket1}
 & 16044, 13563 & 2/10 & Activates on common 2-4 letter substrings within larger tokens or variable names. & Activates on log message prefixes like "W/" and "E/" in Android logcat output. \\
\rowcolor{bucket1}
 & 5177, 15030 & 3/10 & Activates on LaTeX math tokens beginning or ending with backslash or angle brackets. & Activates on file paths or URLs containing delimiter sequences between directory or resource names. \\
\midrule
\rowcolor{bucket2}
 & 15430, 4143 & 4/10 & Activates on ordinal terms and comparative adjectives marking ordered elements in a sequence. & Activates on words following punctuation or in enumerated lists, especially suffixes or grammatical constructs. \\
\rowcolor{bucket2}
 & 13156, 12246 & 6/10 & Activates on markup-like sequences of repeated or paired symbols used as section separators in code. & Activates on closing angle brackets that terminate LaTeX-style or math expression delimiters. \\
\rowcolor{bucket2}
2 & 2292, 15215 & 4/10 & Activates on forms of the verb "to be" used as auxiliary verbs or main verbs. & Activates on the auxiliary verb "have" used for forming present perfect constructs. \\
\rowcolor{bucket2}
 & 4789, 3718 & 4/10 & Activates on phrases like "sounds like", "feels like" expressing resemblance or subjective impressions. & Activates on verbs expressing mental or emotional impact in reactions or realizations. \\
\rowcolor{bucket2}
 & 15326, 4995 & 2/10 & Activates on the token "int" in various contexts, both as a suffix and as a programming token. & Activates on "Image" and variants in variable names, software names, and UI elements. \\
\midrule
\rowcolor{bucket3}
 & 10739, 10630 & 5/10 & Activates on paired angle brackets used as delimiters or markers in code and technical writing. & Activates on code keywords, method names, and identifiers in programming contexts. \\
\rowcolor{bucket3}
 & 1203, 2543 & 3/10 & Activates on tokens containing the sequence "red", "whit", or "Reddit" within longer words. & Activates on scientific nouns denoting materials used in laboratory or industrial contexts. \\
\rowcolor{bucket3}
3 & 9668, 13698 & 6/10 & Activates on past-tense passive forms of verbs indicating completion or accomplishment. & Activates on tokens signifying the achievement or fulfillment of requirements or conditions. \\
\rowcolor{bucket3}
 & 14177, 3256 & 4/10 & Activates on delimiters for copyright and license statements in code and documentation. & Activates on punctuation elements and conjunctions serving as delimiters in complex sentences. \\
\rowcolor{bucket3}
 & 7237, 7841 & 4/10 & Activates on wordpieces containing "wor" and similar substrings within longer words. & Activates on "War" in proper nouns and as a standalone word in relevant contexts. \\
\midrule
\rowcolor{bucket4}
 & 8557, 8334 & 9/10 & Activates on mentions of the color "yellow" when describing objects or attributes. & Activates on the token "yellow" as a standalone word or within color-related phrases. \\
\rowcolor{bucket4}
 & 3796, 6569 & 4/10 & Activates on ".>" token in numeric, scientific notation or code fragments. & Activates on the period character when used as a decimal point in numerical values. \\
\rowcolor{bucket4}
4 & 3161, 11659 & 6/10 & Activates on multi-word phrases with verbs plus prepositions introducing perspectives. & Activates on "in" within common prepositional phrases introducing abstract relationships. \\
\rowcolor{bucket4}
 & 2045, 14698 & 5/10 & Activates on mentions of diseases, particularly cancer and related medical conditions. & Activates on scientific terms denoting important issues, processes, or domains in research. \\
\rowcolor{bucket4}
 & 10453, 9653 & 7/10 & Activates on tokens related to biometric identification, especially fingerprints and forensics. & Activates on technical nouns denoting identifiers, labels, codes, or tags in specialized domains. \\
\midrule
\rowcolor{bucket5}
 & 13974, 1289 & 10/10 & Activates on phrases indicating future events or developments with specific timeframes. & Activates on future-related phrases indicating when something is expected or planned. \\
\rowcolor{bucket5}
 & 9040, 9704 & 10/10 & Activates on "births" in Wikipedia-style category tags denoting birth years. & Activates on "births" within Wikipedia category tags indicating year of birth. \\
\rowcolor{bucket5}
5 & 13514, 4930 & 9/10 & Activates on multi-token prepositions and conjunctions, especially with "of", "by", "to". & Activates on phrases containing prepositions that indicate causes, relationships, or compositions. \\
\rowcolor{bucket5}
 & 3430, 6144 & 9/10 & Activates on function definitions in code following parameter lists before function bodies. & Activates on opening curly braces that begin function or method bodies in code. \\
\rowcolor{bucket5}
 & 3215, 7110 & 9/10 & Activates on "exclude", "excluded", or "exclusion" in procedural or scientific contexts. & Activates on "exclude" and related forms in the context of setting boundaries or omission criteria. \\
\bottomrule
\end{tabular}

\end{table}

\end{document}